\documentclass{article}

\PassOptionsToPackage{round}{natbib}

\usepackage[final]{neurips_2025}

\usepackage[utf8]{inputenc} 
\usepackage[T1]{fontenc}    
\usepackage{hyperref}       
\usepackage{url}            
\usepackage{booktabs}       
\usepackage{amsfonts}       
\usepackage{amsmath}
\usepackage{amsthm}
\usepackage{nicefrac}       
\usepackage{microtype}      
\usepackage[dvipsnames,svgnames]{xcolor}         
\usepackage{comment}
\usepackage{dsfont}

\usepackage{graphicx}
\usepackage{subcaption}
\usepackage{mathrsfs}

\newtheorem{theorem}{Theorem}[section]
\newtheorem{lemma}[theorem]{Lemma}
\newtheorem{proposition}[theorem]{Proposition}

\newtheorem{remark}[theorem]{Remark}

\usepackage{wrapfig}
\usepackage{url}

\definecolor{ForestGreen}{cmyk}{0.91,0,0.88,0.12}
\colorlet{pierrem}{black}
\newcommand{\pierrecomment}[1]{ \textcolor{pierrem!60}{(P: \textit{#1})}}
\newcommand{\pierremodif}[1]{{\color{pierrem}#1\color{black}}}
\colorlet{clrcolor}{black}
\newcommand{\clrcomment}[1]{ \textcolor{clrcolor!60}{(CB: \textit{#1})}}
\newcommand{\clrmodif}[1]{{\color{clrcolor}#1\color{black}}}

\newcommand{\R}{\mathbb{R}}

\DeclareMathOperator*{\erf}{erf}
\DeclareMathOperator*{\argmin}{argmin}
\newcommand{\eqdef}{\stackrel{\text{\rm\tiny def}}{=}}

\title{Attention-based clustering}
\newcommand*\samethanks[1][\value{footnote}]{\footnotemark[#1]}

\author{ 
  Rodrigo Maulen-Soto\thanks{Correspondence: \texttt{maulen@lpsm.paris} and \texttt{claire.boyer@universite-paris-saclay.fr}.} \\
  Sorbonne Université, LPSM  \\
  Paris, France \\
  \And
  Pierre Marion \\
  EPFL, Institute of Mathematics \\
  Lausanne, Switzerland \\
  \And
  Claire Boyer\samethanks \\
  Laboratoire de Mathématiques d’Orsay\\ Université Paris Saclay \\
  Institut Universitaire de France \\
}

\begin{document}

\maketitle

\begin{abstract}
Transformers have emerged as a powerful neural network architecture capable of tackling a wide range of learning tasks. In this work, we provide a theoretical analysis of their ability to automatically extract structure from data in an unsupervised setting. In particular, we demonstrate their suitability for clustering when the input data is generated from a Gaussian mixture model. To this end, we study a simplified two-head attention layer and define a population risk whose minimization with unlabeled data drives the head parameters to align with the true mixture centroids. 

This phenomenon highlights the ability of attention-based layers to capture underlying distributional structure. We further examine an attention layer with key, query, and value matrices fixed to the identity, and show that, even without any trainable parameters, it can perform in-context quantization, revealing the surprising capacity of transformer-based methods to adapt dynamically to input-specific distributions.
\end{abstract}

\section{Introduction}
Attention-based models \citep{bahdanau2015}, in particular Transformers \citep{vaswani}, have achieved state-of-the-art performance across a wide range of learning tasks. These include applications in natural language processing \citep{devlin, bubeck, luong, bahdanau2016} and computer vision \citep{dosovitskiy, liu, stand}. The success of the attention mechanism has been linked to its ability to capture long-range relationships in input sequences \citep{bahdanau2015,vaswani}. They do this by computing pairwise dependencies between tokens based on learned projections, without regard to the tokens' positions in the sequence.

On the theoretical side, a full understanding of attention-based mechanisms has not yet been developed. This is due to the complexity of the architectures and the diversity of relevant tasks they manage to achieve. 
A promising research direction to bridge this gap involves identifying essential features from real-world problems and constructing minimal yet representative tasks that retain the essential difficulty—paired with provable models that solve them using attention-based mechanisms. Notable recent efforts in this vein include \citet{ahn,oswald, yang, zhang2024trained, li2024one, li2023}. However, the existing literature mainly focuses on supervised learning aspects, and in particular in-context learning \citep{oswald,zhang2024trained,garg, li2023, furuya}. The goal of in-context learning is to predict the output corresponding to a new query, given a prompt consisting of input/output pairs.

Beyond the standard supervised setting, Transformer models are often (pre)trained in practice with semi-supervised objectives such as masked language modeling \citep{phuong2022formal}. This raises important questions about their statistical behavior and training dynamics in unsupervised regimes. 
In this work, we examine Transformers through the lens of clustering, thereby revealing the inherent capacity of attention mechanisms to perform unsupervised representation learning.
To the best of our knowledge, the only prior theoretical work that explicitly explores clustering with Transformers is \citet{he}, who demonstrate that attention layers can mimic the EM algorithm \citep{lloyd}, albeit assuming known cluster labels during training.
In contrast, our analysis focuses on the fully unsupervised setting and further provides insight into the functional roles of individual attention heads in the context of model-based clustering.

\paragraph{Contributions.}
In this paper, we investigate the behavior of attention layers in an unsupervised learning setting, where input data is drawn from mixture distributions. We focus on a two-component Gaussian mixture model. Within this classical clustering framework, we introduce a two-headed linear attention layer designed to capture cluster membership through attention scores, while remaining analytically and computationally tractable.
To assess the quality of the embeddings produced by the attention mechanism, we define a theoretical risk analogous to the classical quantization error in unsupervised learning. We analyze the training dynamics of the proposed predictor under projected gradient descent and prove that, with appropriate initialization, the algorithm can learn the true latent centroids of the mixture components, despite the non-convexity of the loss landscape and without access to cluster labels. 
To relax the initialization requirements in practice, we further propose a regularization scheme that promotes disentanglement between attention heads. Our theoretical findings are supported by numerical experiments under varying conditions, including different initialization regimes, mixture separability levels, and problem dimensionalities.
Overall, we show that attention-based predictors can successfully adapt to mixture models by learning the underlying centroids through training. 
\clrmodif{We also study their quantization properties, that is, how the layer summarizes the input distribution. In the oracle regime, where the attention parameters have converged to the true centroids, the problem can be viewed as an approximate denoising task, in the sense commonly used in the statistical physics literature, where each input can be interpreted as a noisy observation of an underlying centroid, and the transformed token corresponds, to some extent, to its denoised estimate.}
Finally, we focus on a particular attention layer in which the key, query, and value parameters are fixed to the identity matrix. Surprisingly, we show that this type of layer, despite having no trainable parameters, can still perform in-context quantization, meaning it adapts to the case where the distribution of each input sequence comes with its own mixing parameters. This further demonstrates the remarkable ability of transformer-based methods to adapt on the fly to the underlying data distribution, even when no attention parameters are trained.


\paragraph{Organization.}  Section \ref{sec:starter} introduces the problem and outlines the proposed approach. 
In Section \ref{sec:gmm}, we address the general problem with linear attention applied to a two-component Gaussian mixture model. 
In Section \ref{sec:statistical_properties}, we discuss the quantization properties of attention-based predictor with oracle parameters.
In Section \ref{sec:incontext}, we explore an in-context clustering framework and examine the quantization capabilities of a simple attention-based layer with no learned parameters. 
The proofs of all the theoretical results can be found in the appendices.

\section{A starter on attention-based layers and clustering}\label{sec:starter}

\paragraph{Data distribution in model-based clustering.} In attention-based learning, the key idea is to map a set of input tokens to a transformed set of output tokens. With this in mind, we consider an
input sequence $\mathbb{X}\in\R^{L\times d}$  composed of $L$ tokens $(X_1,\ldots, X_L)$, each token being a vector of $\R^d$, i.e., $\mathbb{X}= (X_1,\ldots,X_L)^\top$.
We assume that the tokens are i.i.d.\ drawn from a simple mixture model: for $1\leq \ell\leq L$, 
\begin{align}
\label{def:gaussian_mixture_model}
\tag{$\mathrm{P}_{\sigma}$}
X_\ell &\sim \frac{1}{2}\mathcal{N}(\mu_0^\star,\sigma^2I_d)+\frac{1}{2}\mathcal{N}(\mu_1^\star,\sigma^2I_d), 
\end{align}
with balanced components and where the centroids $(\mu_0^\star,\mu_1^\star)\in (\mathbb{S}^{d-1})^2$ (i.e., $\|\mu_0^\star\|_2=\|\mu_1^\star\|_2=1$) are assumed to be orthogonal, i.e., such that $\langle \mu_0^\star,\mu_1^\star \rangle=0$. Therefore, for each token, there exists an associated latent variable, denoted by $Z_\ell$, corresponding to a Bernoulli random variable of parameter $1/2$ and encoding its corresponding cluster.

\paragraph{Attention-based predictors.} 
An attention head made of a self-attention layer can be written as 
 $$
 H^{\mathrm{soft_\lambda}}(\mathbb{X}) = \mathrm{softmax}_\lambda \left(  \mathbb{X}QK^\top \mathbb{X}^\top \right) \mathbb{X} V
 $$
where the softmax of temperature $\lambda >0$ is applied row-wise, no skip connection is considered and the matrices $K,Q,V \in \mathbb{R}^{d\times p}$ are usually referred to as keys, queries and values. 
We adopt the convention that the values are identity matrices; thus, the attention head simply outputs combinations of the initial tokens weighted by the attention scores. While this simplification is certainly convenient for facilitating the mathematical analysis that follows, it is also supported by experimental studies showing comparable performance when the value matrices are removed \citep{simplifying}.
Furthermore, assume that the key and query matrices are equal to the same column matrix $\mu\in \mathbb{R}^{d\times 1}$, we obtain 
 \begin{align}
 \label{eq:softmax_head}
 H^{\mathrm{soft_{\lambda}},\mu}(\mathbb{X}) = \mathrm{softmax}_\lambda \left(  \mathbb{X}\mu \mu^\top \mathbb{X}^\top \right) \mathbb{X} .
 \end{align}
With such an architecture, the $\ell$-th output vector is therefore given by
\begin{equation}
H^{\mathrm{soft_{\lambda}},\mu}(\mathbb{X})_\ell = \sum_{k=1}^L \mathrm{softmax}_\lambda \left(  X_\ell^\top  \mu \mu^\top \mathbb{X}^\top \right)_k X_k ,
\end{equation}
which corresponds to aggregating the $X_k$'s when $X_k$ and $X_\ell$ are simultaneously aligned with $\mu$. This suggests that attention heads may act as effective learners in a clustering framework.

The softmax nonlinearity used in the attention head \eqref{eq:softmax_head} introduces a coupling between tokens, which undoubtedly complicates the mathematical analysis. To address this difficulty, we propose to look at a simplified linear attention head, still parameterized by 
$\mu\in\mathbb{R}^d$, and defined for $1\leq \ell \leq L$, as
\begin{equation}
    H^{\mathrm{lin},\mu}(\mathbb{X})_\ell=\frac{2}{L}\sum_{k=1}^L (\lambda X_\ell^\top\mu\mu^\top X_k)X_k.
\end{equation}
This head uses a linear activation function instead of the traditional softmax found in practical architectures, and has already received interest in several mathematical studies \citep[see][]{zhang2024trained, oswald, han, rnn}. \pierremodif{This model is a particular case of Gaussian sequence multi-index models \citep{cui2025high}, which were already used to study attention models \citep{cui2024phase,arnaboldi2025asymptotics,troiani2025fundamental}, albeit not in a clustering context.}

Note that when 
$\mu$ is chosen to be $\mu_0^\star$, then for tokens 
$X_\ell$ and 
$X_k$ whose corresponding latent variables $Z_\ell$ and $Z_k$ are both equal to 0 (i.e., the samples belong to the same cluster centered at $\mu_0^\star$), the vectors 
$X_\ell$ and 
$X_k$ are likely to be aligned with $\mu_0^\star$. In this case, we have $(X_\ell^\top\mu\mu^\top X_k)X_k \simeq X_k$.
Conversely, if $X_\ell$ and 
    $X_k$ are associated with different latent variables (e.g., $Z_\ell=0$ and  $Z_k=1$), then $(X_\ell^\top\mu\mu^\top X_k)X_k \simeq 0$. This behavior suggests that when setting $\mu=\mu_0^\star$, and if $X_\ell$ belongs to the cluster centered at $\mu_0^\star$, the sum $\sum_{k=1}^L (\lambda X_\ell^\top\mu\mu^\top X_k)X_k$ effectively aggregates the $X_k$'s from the same cluster, whose expected number is $L/2$, motivating the renormalizing factor of $2/L$. 
    Overall, $H^{\mathrm{lin},\mu_0^\star}(\mathbb{X})_\ell$ can be seen as producing an empirical mean of the tokens belonging to the same cluster, serving as an estimator of the corresponding centroid.

    Therefore, assuming that the number of clusters in the data is known, it is natural to consider an attention-based predictor composed of two attention heads, parameterized by $\mu_0$ and $\mu_1\in\mathbb{R}^d$,
    \begin{equation}
    T^{{\rm lin}, \mu_0, \mu_1}(\mathbb{X}) = H^{\mathrm{lin},\mu_0}(\mathbb{X}) + H^{\mathrm{lin},\mu_1}(\mathbb{X}).     
    \end{equation}

\paragraph{Metric loss.} 
As no label is available, we focus on minimizing the following theoretical loss: 
\begin{equation} \label{eq:risk}
{\mathcal{L}}(T)\eqdef \frac{1}{L}\sum_{\ell=1}^L\mathbb{E}\left[\Vert X_\ell-T(\mathbb{X})_\ell\Vert_2^2\right], \,     
\end{equation} 

where $T$ is an arbitrary attention-based predictor.
The distinctive feature of this risk lies in the fact that, if the predictor were able to return, for each token $X_\ell$, its associated centroid $\mu^\star_{Z_\ell}$, the risk would exactly correspond to a quantization error, characteristic of a standard clustering task. Note that, due to the independence of the tokens, we have 
$
\mathcal{L}(T) = \mathbb{E}\|X_1 - T(\mathbb{X})_1\|_2^2
$, so we can confine the following theoretical analysis on the minimization of the predictive error for the first token only.

\paragraph{PGD iterates.} 
In this paper, we focus on the training dynamics of Transformer-based predictors when minimizing the theoretical risk $\mathcal{L}$. While we acknowledge that, in practice, an empirical version of this risk is typically used, analyzing the optimization of the theoretical risk is already a non-trivial task, which offers valuable insights into the behavior observed in practice.

For a given predictor $T^{\mathrm{lin}, \mu_0, \mu_1}$ made of two linear attention heads parameterized respectively by $\mu_0$ and $\mu_1$,  one can reinterpret the objective $\mathcal{L}$ as a function $\mathcal{R}$ of the parameters $(\mu_0, \mu_1)$, defined by
\begin{equation}
\mathcal{R}(\mu_0, \mu_1) = \mathcal{L}(T^{\mathrm{lin}, \mu_0, \mu_1}).    
\end{equation}

Note that the computation of the risk also depends on the choice of the underlying data distribution. However, as is common practice in the literature, we do not explicitly indicate the dependence of the risk on the data distribution. This should not hinder understanding, as the distributional assumptions and context will always be made clear.
As we rely on the theoretical analysis on an expression of this risk restricted to the sphere, we consider as a gradient strategy, the Projected (Riemaniann) Gradient Descent (PGD) algorithm  \citep{boumal}. Given an initialization $(\mu_0^0, \mu_1^0) \in (\mathbb{S}^{d-1})^2$ and a step size $\gamma > 0$, the PGD iterates $(\mu_0^k, \mu_1^k) \in (\mathbb{S}^{d-1})^2$ are recursively defined by:
\begin{equation}\label{eq:PGD_iterates}\tag{$\mathrm{PGD}$}
    \begin{split}
        \mu_0^{k+1}&=\frac{\mu_0^k-\gamma(I_d-\mu_0^k(\mu_0^k)^\top)\nabla_{\mu_0}\mathcal{R}(\mu_0^k,\mu_1^k)}{\Vert \mu_0^k-\gamma(I_d-\mu_0^k(\mu_0^k)^\top)\nabla_{\mu_0}\mathcal{R}(\mu_0^k,\mu_1^k)\Vert_2},\\
        \mu_1^{k+1}&=\frac{\mu_1^k-\gamma(I_d-\mu_1^k(\mu_1^k)^\top)\nabla_{\mu_1}\mathcal{R}(\mu_0^k,\mu_1^k)}{\Vert \mu_1^k-\gamma(I_d-\mu_1^k(\mu_1^k)^\top)\nabla_{\mu_1}\mathcal{R}(\mu_0^k,\mu_1^k)\Vert_2}.
        \end{split}
    \end{equation}

In what follows, we analyze the convergence of these iterates to the oracle centroids, both theoretically and numerically.

\section{Training dynamics: The centroids are learned as attention parameters}
\label{sec:gmm}
\clrmodif{In this section, we analyze the training dynamics of an attention layer optimized by minimizing the risk $\mathcal{R}$ through the PGD iterations defined in \eqref{eq:PGD_iterates}.
We focus on the case of clustering data generated from a Gaussian mixture model, as specified in \eqref{def:gaussian_mixture_model}, where the input tokens are i.i.d. samples
$$
X_\ell\sim \frac{1}{2}\mathcal{N}(\mu_0^\star,\sigma^2 I_d)+\frac{1}{2}\mathcal{N}(\mu_1^\star,\sigma^2 I_d),
$$
for $1\leq \ell\leq L$, 
where $(\mu_0^\star, \mu_1^\star) \in (\mathbb{S}^{d-1})^2$ are orthogonal unit vectors. 
As a first step toward a mathematical understanding of Transformer-based layers in clustering settings, we have also considered the degenerate case $\sigma^2 = 0$, in which the data are drawn from a mixture of Dirac masses. We perform a detailed analysis of the training dynamics under this simplified setting, which provides valuable intuition and serves as a foundation for the subsequent study of the non-degenerate mixture model. For completeness, this preliminary analysis is presented in Appendix~\ref{sec:dirac}.
}


\subsection{Theoretical analysis}

\paragraph{Preliminary computations.}  We start by introducing the following quantities:
\begin{align}\label{notation}
\kappa_0&\eqdef\langle\mu_0^\star,\mu_0\rangle, \quad \kappa_1\eqdef\langle\mu_1^\star,\mu_1\rangle, \quad 
\eta_0\eqdef\langle \mu_1,\mu_0^\star\rangle,
\quad\eta_1\eqdef\langle \mu_0,\mu_1^\star\rangle,
\quad\xi \eqdef \langle\mu_0,\mu_1\rangle,
\end{align}
and derive a closed-form expression for the risk w.r.t.\ this reparameterization. Although the full expression is somewhat complex (see Appendix \ref{app:risk_gmm}), the following proposition highlights its key structural properties, as being a polynomial in these five variables. 
\begin{proposition}\label{prop:structure_of_R1}
    Under the Gaussian mixture model $\eqref{def:gaussian_mixture_model}$, consider the attention-based predictor $T^{{\rm lin}, \mu_0,\mu_1}$ composed of two linear heads parameterized by $(\mu_0,\mu_1)\in(\mathbb{S}^{d-1})^2$.  Then, there exists a function $\mathcal{R}^<:[-1,1]^5\mapsto \R$ such that $\mathcal{R}(\mu_0,\mu_1)=\mathcal{R}^<(\kappa_0,\kappa_1,\eta_0,\eta_1,\xi)$ and 
\begin{align}
\begin{split}
\mathcal{R}^<(\kappa_0,\kappa_1,\eta_0,\eta_1,\xi) \in 
\mathrm{span}\big(\{&\kappa_0^4,\eta_0^4,\kappa_1^4,\eta_1^4,
\kappa_0^2\eta_0^2,\kappa_1^2\eta_1^2,\kappa_0^2\eta_1^2,\kappa_1^2\eta_0^2,
\kappa_0\eta_0\kappa_1\eta_1,\\
&\kappa_0^2,\eta_0^2,\kappa_1^2,\eta_1^2,\kappa_0\eta_0\xi,\kappa_1\eta_1\xi,\xi^2,1\}\big).
\end{split}
\end{align}
\end{proposition} 
We remark that when $\eta_0, \eta_1$, and $\xi$ are fixed to $0$, most of the monomials vanish, yielding a fully explicit formula for the risk (see Lemma \ref{lem:riskmanifold_gmm} in the appendices). Far from being a mere rewrite, this step provides the algebraic foundation for all the exact calculations and insights that follow.


\paragraph{Optimality conditions.}
Given the complexity of the theoretical analysis, we focus on a simplified setting by restricting our study to parameters $(\mu_0,\mu_1)$ lying on a specific manifold\footnote{up to relabeling the head parameters, since a priori, $\mu_0$ (resp. $\mu_1$) does not have to be automatically related to $\mu_0^\star$ (resp. $\mu_1^\star$).}:
\begin{equation}
\mathcal{M}=\{(\mu_0,\mu_1)\in(\mathbb{S}^{d-1})^2: \langle\mu_1^\star,\mu_0\rangle=0, \langle\mu_0^\star,\mu_1\rangle=0, \langle\mu_0,\mu_1\rangle=0\}.  
\end{equation}

In terms of notation, it is equivalent to assume that $\eta_0=0, \eta_1=0, \xi=0$. Therefore on this manifold, we adopt the shorthand notation $\mathcal{R}^<(\kappa_0,\kappa_1)\eqdef \mathcal{R}^<(\kappa_0,\kappa_1,0,0,0)$.

\begin{lemma}\label{lem:riskmanifoldshort_gmm}
Under the Gaussian mixture model $\eqref{def:gaussian_mixture_model}$, the risk $\mathcal{R}^<$ restricted to $\mathcal{M}$ has the form \begin{align} \mathcal{R}^<(\kappa_0,\kappa_1)&=A(\kappa_0^4+\kappa_1^4)+B(\kappa_0^2+\kappa_1^2)+C\kappa_0^2\kappa_1^2+D,
 \end{align}
 for $A,B,C,D$ non-negatives constants, dependent on $\sigma$ and $L$, made explicit in Lemma \ref{lem:riskmanifold_gmm}.
\end{lemma}

\begin{proposition}\label{prop:optimallambdashort_gmm}
    Consider $\mathcal{R}^<(\kappa_0,\kappa_1)$, there exists $\lambda^\star(\sigma,L)>0$ such that the points $(\pm 1,\pm 1)$ are global minima of $\mathcal{R}^<(\kappa_0,\kappa_1)$.    
\end{proposition}
This result demonstrates that, under a suitable condition on the temperature parameter—specifically, when $\lambda = \lambda^\star(\sigma, L)$—the points $\pm \mu_0^\star$ and $\pm \mu_1^\star$ are global minimizers of the risk. The explicit form of $\lambda^\star(\sigma, L)$ is provided in Proposition \ref{optimallambda} in the appendices.
Moreover, it is worth noting that as the variance $\sigma^2$ of the Gaussian components tends to zero, $\lambda^\star(\sigma, L)$ approaches the value $\lambda_0^\star = \frac{L+1}{L+3}$,which coincides with the value previously identified in the degenerate case \eqref{def:dirac_mixture_model}. 
On the other hand, when $\sigma$ is fixed and $L$ grows large, $\lambda^\star(\sigma, L)$ tends toward  $\lambda_\infty^\star = \frac{1 + 4\sigma^2}{1 + 5\sigma^2 + 6\sigma^4}$, which will guide us to properly choose $\lambda$ in our numerical experiments.  

\paragraph{Convergence analysis.} In what follows, we show that the \eqref{eq:PGD_iterates} iterates 
can indeed converge to global minimizers, provided they are suitably initialized on the manifold $\mathcal{M}$.
\begin{theorem}\label{thm:main_gmm}
Under the Gaussian mixture model $\eqref{def:gaussian_mixture_model}$, consider the attention-based predictor $T^{{\rm lin}, \mu_0,\mu_1}$ composed of two linear heads. 
    Take $\lambda\in ]0,\lambda^\star(\sigma,L)]$, with $\lambda^\star(\sigma,L)$ defined as in Proposition \ref{prop:optimallambdashort_gmm}. Then there exists $\bar{\gamma}>0$ such that for any stepsize $0<\gamma<\bar{\gamma}$, and for a generic initialization $(\mu_0^0,\mu_1^0)\in \mathcal{M}$, the iterates $(\mu_0^k,\mu_1^k)$ generated by \eqref{eq:PGD_iterates} converge to the centroids (up to a sign), i.e. 
    $$
    (\mu_0^k,\mu_1^k)\xrightarrow[k\rightarrow\infty]{}(\pm \mu_0^\star,\pm \mu_1^\star).
    $$
\end{theorem}

Theorem \ref{thm:main_gmm} underlines the capabilities of linear attention-based predictors in a clustering framework. With appropriate initialization, the attention heads align with the true underlying centroids even when trained without access to labels. This result shows that attention layers can uncover and encode the latent structure of the input distribution in a fully unsupervised setting through their parameters.


\subsection{Experimental verification of the theoretical results}\label{xp_gmm}

\paragraph{Setting.} 
To better reflect practical algorithmic behavior, we implement Projected Stochastic Gradient Descent (PSGD; see Appendix \ref{def:psgdlinear}), which serves as an empirical counterpart to the \eqref{eq:PGD_iterates} iterates by relying on sample-based estimations.

In what follows, we use the metric referred to as \emph{distance to the centroids (up to a sign)}, given by 
\begin{align}
\min_{\pi\in S_2} \min_{s\in \{-1,1\}^2}\sqrt{\sum_{i=0}^1 \Vert \hat{\mu}_{\pi(i)}-s_i\mu_i^\star\Vert^2},
\label{def:xp_metric2}
\end{align}

where $S_2$ is the permutation group of two elements, $\mu_0^\star, \mu_1^\star$ denote the true centroids, respectively, while $\hat{\mu}_0, \hat{\mu}_1$ are the parameters returned by \eqref{PGDrho}. Note that this distance is invariant under relabeling and sign flips of the head parameters. More implementation details related to the following experiments can be found in Appendix \ref{xp_details_dirac}.

It is worth noting that the assumption of orthogonality of the centroids on the unit sphere always results in a constant distance between the centroids, namely $\|\mu_0^\star - \mu_1^\star\|_2=\sqrt{2}$.
In this setting, to characterize the separation between the two modes of the mixture, one can introduce a notion of \emph{interference} that depends solely on the variance level of each mode and which is defined as
$\mathrm{I}(\sigma)=\mathbb{P}(X_\sigma>\frac{\sqrt{2}}{2}),$ where $X_\sigma\sim \mathcal{N}(0,\sigma^2)$. Remark that this function is increasing with supremum $0.5$. 
This motivates the choice of two contrasting scenarios for the numerical experiments: a low-interference regime with $\sigma = 0.3$, where $\mathrm{I}(0.3) \approx 0.01$, and a high-interference regime with $\sigma = 1$, where $\mathrm{I}(1) \approx 0.24$.
More implementation details of the following numerical experiments can be found in Appendix \ref{xp_details_gmm}.

\paragraph{Results.}
\begin{wrapfigure}[13]{r}{0.44\textwidth}
\includegraphics[width=\linewidth]{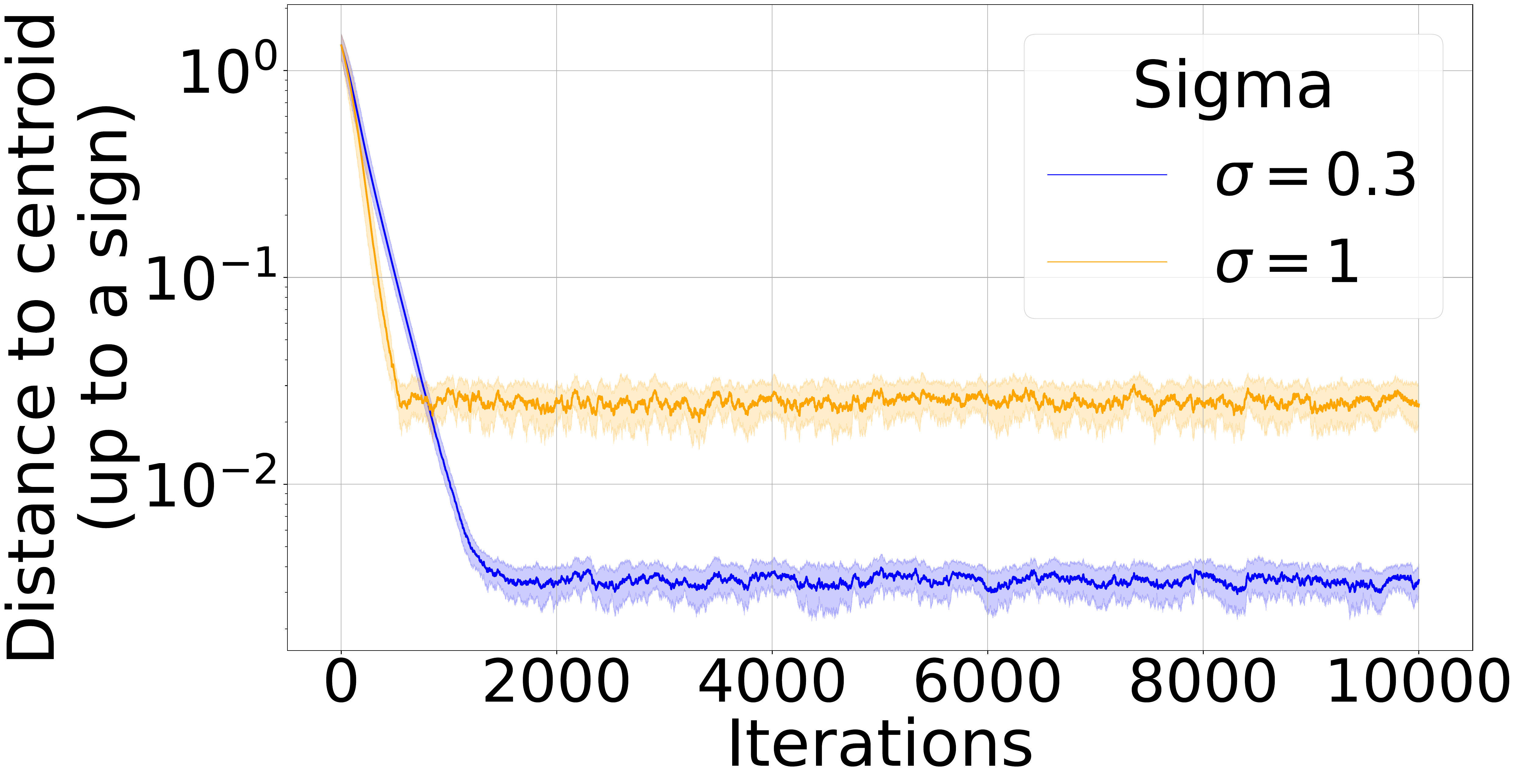}
    \caption{
    Distance to centroids vs \eqref{PGDrho} iterations for the minimization of $\mathcal{R}$ (therefore without regularization), with an initialization on the manifold $\mathcal{M}$. 10 runs, 95\% percentile intervals are plotted. }
         \label{fig:linear_manifold_iters}
\end{wrapfigure}
When initialization is done on the manifold, the training analysis depicted in Figure \ref{fig:linear_manifold_iters} demonstrates linear convergence of the head parameters toward the centroids (up to a sign) during the first $10^3$ iterations, which is in line with the obtained theoretical results. The error then plateaus at around $10^{-2}$ in the low-interference setting ($\sigma = 0.3$), and around $10^{-1}$ in the high-interference setting ($\sigma = 1$). This saturation phenomenon is attributed to the stochasticity introduced by \eqref{PGDrho} in place of \eqref{eq:PGD_iterates} in the simulations.

\subsection{Generalizations} 
We consider several generalizations of our theoretical setting, to give insight on the role of our assumptions. 

\paragraph{Random initialization on the unit sphere (outside of the manifold).}
When the initialization is performed outside the manifold, PGD iterates only partially align with the underlying centroids. A way to handle arbitrary initializations (suggested by our analysis in the degenerate case, see Appendix~\ref{sec:dirac}), is to introduce a regularized risk minimization problem:
\begin{align}\label{penalizationreal}\tag{$\mathcal{P}_{\rho}$}
    \min_{\mu_0,\mu_1\in\mathbb{S}^{d-1}} & \mathcal{R}^{\rho} (\mu_0, \mu_1), \quad \text{with} \quad \mathcal{R}^{\rho} (\mu_0, \mu_1) \eqdef\mathcal{R}(\mu_0,\mu_1)+\rho r(\mu_0,\mu_1),
\end{align}
for $\rho>0$ and the regularization term defined by $r(\mu_0,\mu_1)=\mathbb{E}[\langle \mu_0,X_1\rangle^2\langle \mu_1,X_1\rangle^2]$. The role of this term is to encourage the orthogonality conditions on $\mu_0, \mu_1$, thereby compensating for initializations that may fall outside the manifold $\mathcal{M}$. Numerical results show that the centroids can be recovered with an appropriate level of regularization (see Figure~\ref{linear_noise_reg}). Note that, as the strength of the regularization increases, it gradually overrides the original objective and impairs the alignment of the head parameters with the true centroids ---an effect that becomes more pronounced at higher noise levels. In Figure \ref{linear_noise_iters}, we fix the regularization strength and observe linear convergence towards the centroids over the course of $4\times 10^3$ iterations. The error eventually plateaus near $10^{-3}$ for $\sigma = 0.3$ and near $10^{-1}$ for $\sigma = 1$. This shows that the regularization strategy inspired by the analysis of the simplified Dirac mixture case remains effective in the more realistic setting of Gaussian mixtures. In this context as well, it enhances the interpretability of the attention parameters by encouraging their alignment with the unknown components of the underlying mixture.

\begin{figure}[ht]
    \centering
    \begin{subfigure}[b]{0.48\textwidth}
         \centering
         \includegraphics[width=\textwidth]{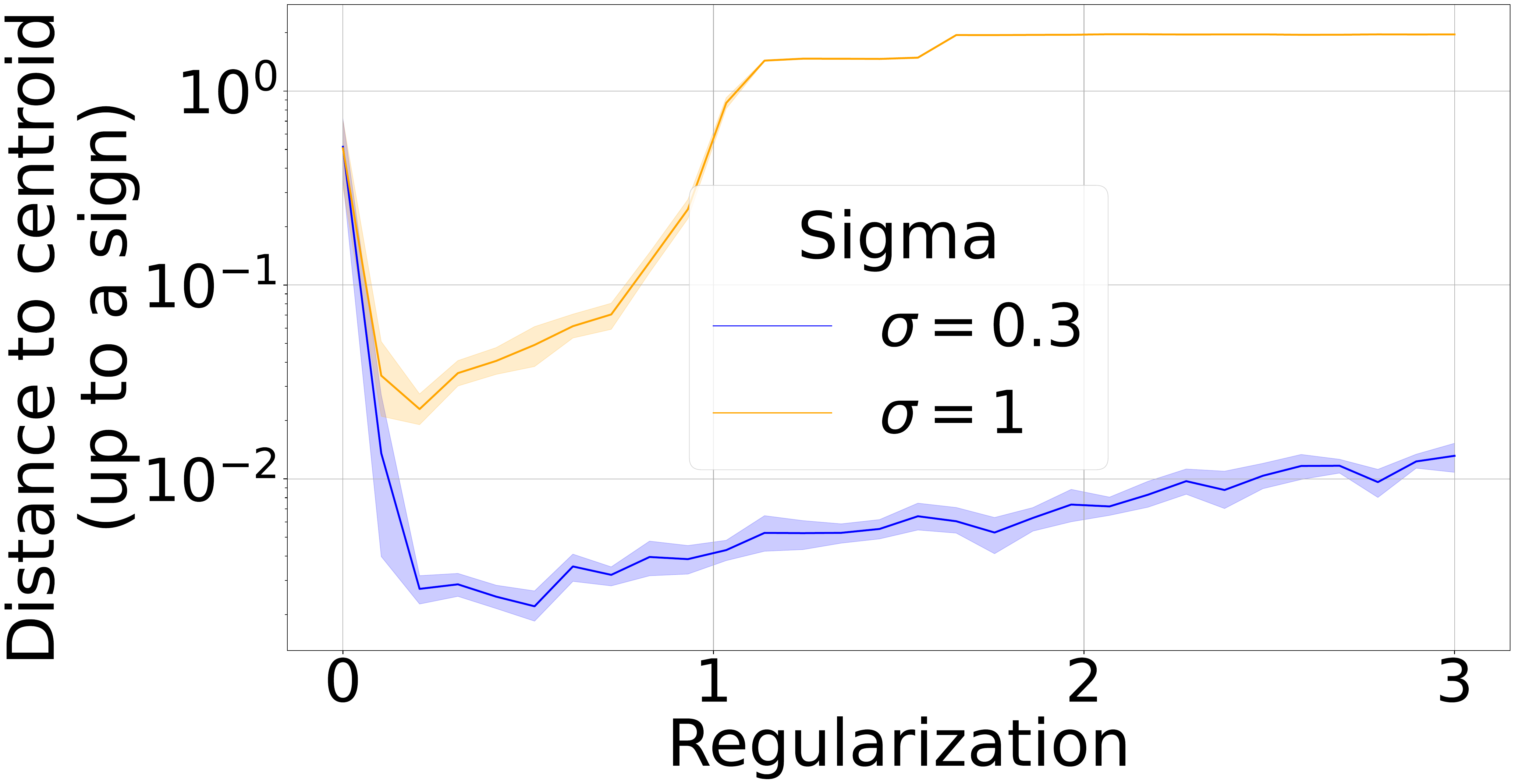}
         \caption{Distance to centroids after 5000 iterations vs regularization strength $\rho$ for the minimization of $\mathcal{R}^{\rho}$.}
    \label{linear_noise_reg}
     \end{subfigure}
     \begin{subfigure}[b]{0.48\textwidth}
         \centering
         \includegraphics[width=\textwidth]{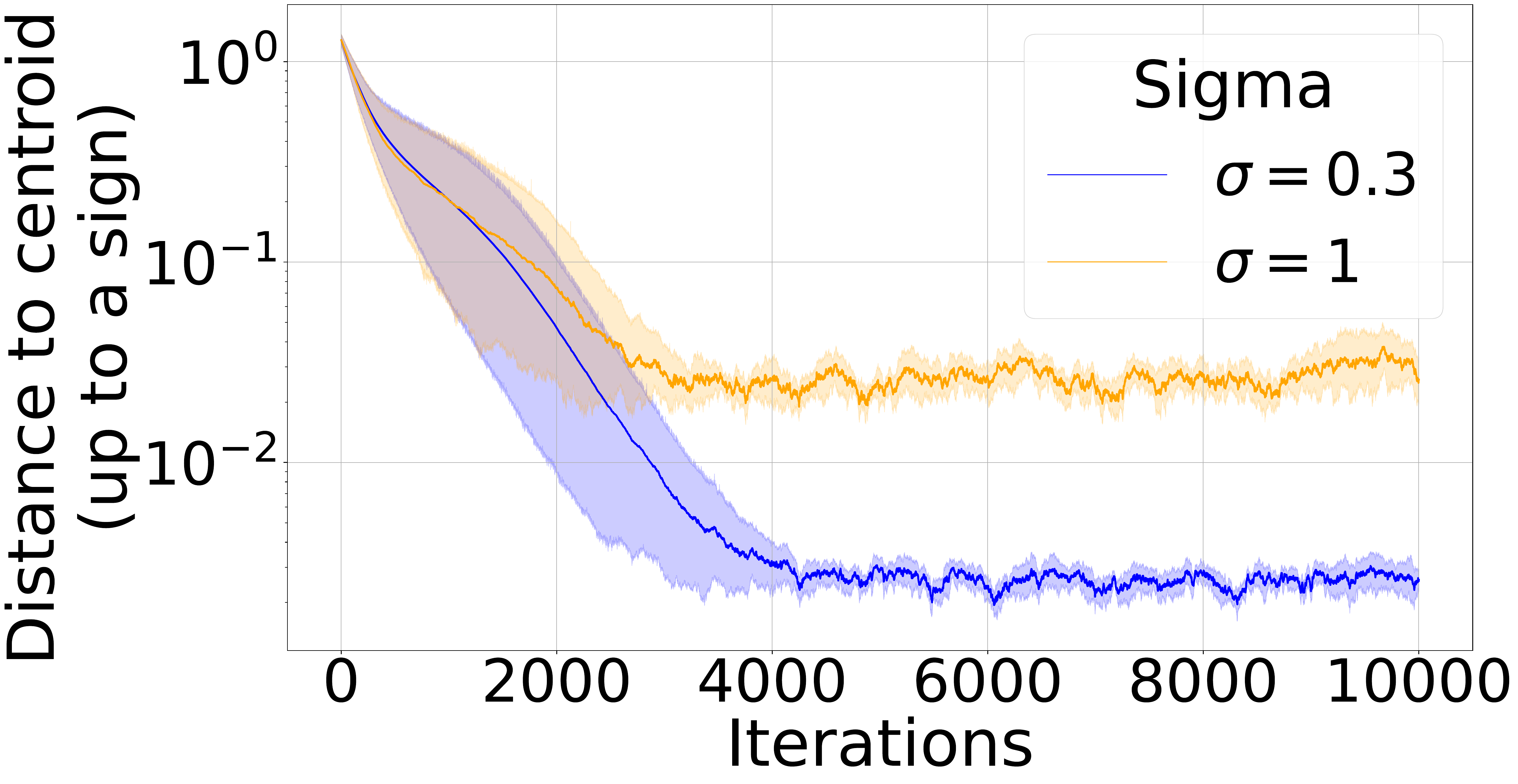}
         \caption{Distance to centroids vs \eqref{PGDrho} iterations for the minimization of $\mathcal{R}^{\rho}$, with regularization $\rho=0.2$. }
    \label{linear_noise_iters}
     \end{subfigure}
     \caption{Performance of \eqref{PGDrho}, when initializing on the unit sphere and minimizing the regularized risk \eqref{penalizationreal}. 10 runs, 95\% percentile intervals are plotted.}
\end{figure}
In Appendix~\ref{app:xp_dim}, we present additional experiments in higher-dimensional settings, highlighting the impact of dimensionality on the training dynamics.

\paragraph{Orthogonality of the clusters.}
The framework studied in this paper assumes that the centroids of the clusters are orthogonal. Relaxing this assumption to allow for potential overlap between centroid directions significantly complicates the theoretical analysis by introducing additional terms.
That said, the analysis in the orthogonal case remains highly informative for the non-orthogonal setting, particularly in high-dimensional regimes. Indeed, when the true centroids are randomly sampled on the sphere, they become nearly orthogonal as the dimension $d$ increases. In such cases, learning the attention heads via the regularized risk minimization \ref{penalizationreal} remains meaningful and yields compelling empirical results (see Appendix \ref{app:relaxing_orhtogonality_xp}).

\paragraph{Euclidean gradient instead of Riemannian gradient.}
The use of projected Riemannian gradient descent is both natural and mathematically justified in our setting, as we are able to compute the risk function only for parameters constrained to lie on the unit sphere. 
This constraint, however, is not a practical limitation: our numerical experiments indicate that standard projected SGD heuristics exhibit similar behavior to the theoretically grounded Riemannian updates, see Appendix \ref{projsgd}. We emphasize that the projection step is essential for the algorithm’s proper functioning. 

\paragraph{More clusters.} A natural extension is to consider the case of $K$ centroids and attention heads. We perform an experiment in the case $K=3$, which shows recovery of the centroids, see details in Appendix \ref{app:K=3}.
The main difference with the case $K=2$ is the regularization term which now writes
\[
r(\mu_0,\mu_1,\mu_2) = \sum_{0\leq i<j\leq 2}\langle\mu_i,X_1\rangle^2\langle\mu_j,X_1\rangle^2
\]
to promote pairwise orthogonality between the parameters. This approach should further generalize to the case of $K$ orthonormal centroids with $K<d$.

\color{black}
\section{Attention-based layers as approximate quantizers}
\label{sec:statistical_properties}

We have seen that attention-based predictors can adapt to mixture models by learning the underlying centroids through training. In this section, we investigate the quantization properties of an attention-based predictor whose parameters have converged to the true centroids. 
To guide our analysis, we introduce the optimal quantizer $T^\star$ as a statistical benchmark within a standard clustering framework. This oracle predictor returns the true centroid of each token, that is, for $1\leq \ell \leq L$, $T^\star(\mathbb{X})_\ell=\mu_{Z_\ell}^\star$ where $Z_\ell$ is encoding the latent cluster of the token $X_\ell$.  
One can immediately note  that the risk of the optimal quantizer is given by
\begin{equation*}
\mathcal{L}(T^\star) = \mathbb{E}\left[ \left\| X_1 -\mu_{Z_1}^\star \right\|_2^2 \right] = d\sigma^2.    
\end{equation*}
Returning to the attention-based predictor $T^{\mathrm{lin},\mu_0^\star,\mu_1^\star}$ with oracle parameters, the first key observation is that it closely resembles an optimal quantizer: its $\ell$-th output aligns, on average, with the centroid of the cluster to which the $\ell$-th token belongs, as shown by the next lemma.

\begin{proposition}
\label{prop:expectation_oracle_attention}
Under the Gaussian mixture model \eqref{def:gaussian_mixture_model}, it holds that
    $$
    \mathbb{E}[T^{\mathrm{lin},\mu_0^\star,\mu_1^\star}(\mathbb{X})_1|Z_1=c]=\mu_{c}^\star\frac{\lambda}{L}[(L+1)+2(L+3)\sigma^2],\quad \text{for} \quad c=\{0,1\}.
    $$
    Therefore, choosing $\lambda=\frac{L}{(L+1)+2(L+3)\sigma^2}$ leads to unbiased approximate quantization, i.e.,
    $$\mathbb{E}[T^{\mathrm{lin},\mu_0^\star,\mu_1^\star}(\mathbb{X})_1|Z_1=c]=\mu_{c}^\star.$$
\end{proposition}


One can next characterize the asymptotic risk and variance of the oracle attention-based predictor. 
\begin{proposition}
    \label{prop:risk_oracle_attention}
    Under the Gaussian mixture model \eqref{def:gaussian_mixture_model}, fix the temperature $\lambda = \frac{1+4\sigma^2+4\sigma^4}{1+6\sigma^2+12\sigma^4 +8\sigma^6}$.
 Then, the risk of the attention-based predictor $T^{\mathrm{lin}, \mu_0^\star, \mu_1^\star}$ with oracle parameters $\mu_0^\star$ and $\mu_1^\star$ satisfies
 \begin{align*}
     \lim_{L\rightarrow\infty}\mathcal{L}(T^{\mathrm{lin}, \mu_0^\star, \mu_1^\star}) & 
     = \sigma^2(d-2).
 \end{align*}
Moreover, for an arbitrary value of $\lambda$, when $L\rightarrow\infty$, we get 
    \begin{align*}
        \lim_{L\rightarrow\infty}\mathrm{Var}[ T^{\mathrm{lin},\mu_0^\star,\mu_1^\star}(\mathbb{X})_1|Z_1=c]=2\lambda^2\sigma^2(1+2\sigma^2)^2.
    \end{align*}
\end{proposition}
Strikingly, as the input sequence length $L$ increases, we find that
\begin{align*}
\lim_{L\rightarrow\infty}\frac{\mathcal{L}(T^{{\rm lin}, \mu_0^\star, \mu_1^\star})}{\mathcal{L}(T^\star)}
&= 1-\frac{2}{d}.
\end{align*}
This result shows that the attention-based predictor asymptotically achieves a lower risk than the optimal quantizer. This phenomenon can be partly explained by the fact that the comparison is not entirely fair:
the predictors $T^{{\rm lin},\mu_0^\star, \mu_1^\star}$ and $T^\star$ do not belong to the same class of functions. Indeed, the optimal quantizer $T^\star$ is only allowed to return two fixed vectors and relies on a single input token to predict the associated centroid (albeit with access to the latent label). 
\begin{wrapfigure}[20]{r}{0.46\textwidth}
\includegraphics[width=\linewidth]{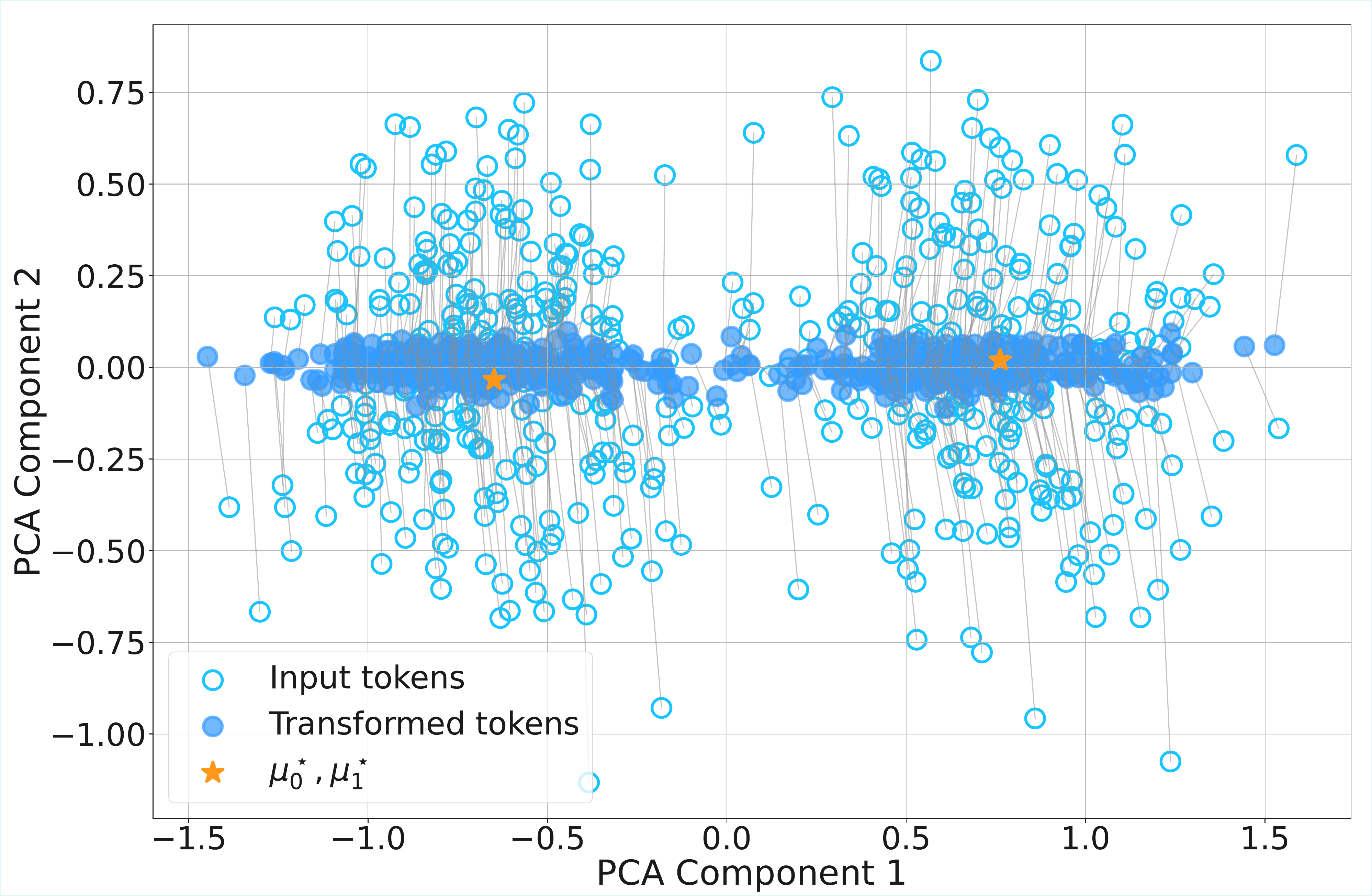}
    \caption{Comparison between inputs and embeddings ($d=10, \sigma=0.3$, $L=500$, $\lambda=\frac{1}{1+2\sigma^2}$, $\mu_0^\star=e_{10}, \mu_1^\star=-e_1$, $e_j$ is the $j$-th vector of the canonical basis of $\R^{10}$). PCA fitted on input tokens was used to project both input and transformed tokens to 2D.} 
    \label{fig:embedding}
\end{wrapfigure}
On the other hand, the image of attention-based encoder $T^{{\rm lin},\mu_0^\star, \mu_1^\star}$ is not discrete, and moreover this estimator aggregates a growing sequence of random variables, all drawn from the same mixture.
The aggregation of multiple inputs can be seen as a variance reduction mechanism, which reduces the risk (this is also evident in the proof of Proposition~\ref{prop:risk_oracle_attention}, where the risk is shown to decrease as $L$ increases). Note that the gap vanishes in a high-dimensional setting $d\to\infty$.

Another insightful comparison between the predictors $T^{{\rm lin},\mu_0^\star, \mu_1^\star}$ and $T^\star$ is through their conditional variances $\mathrm{Var}[ T^{{\rm lin},\mu_0^\star, \mu_1^\star}(\mathbb{X})_1|Z_1=c]$. While the conditional variance of the optimal quantizer is null by definition, it is positive for the linear attention layer, and asymptotically independent of $d$ as shown in Proposition \ref{prop:risk_oracle_attention}. 
This once again highlights the fact that these two quantifiers belong to function classes of different complexity.
These properties are illustrated in Figure \ref{fig:embedding}: we observe that the attention-based embeddings are approximate projections of the inputs on the line between the two centroids. In particular, the variance of the embedded point cloud is lower than the variance of the inputs, illustrating that (for $\lambda$ independent of $d$ and $d$ large enough),
\vspace{0.5cm}
\[
0 < \lim_{L\rightarrow\infty}\mathrm{Var}[ T^{\mathrm{lin},\mu_0^\star,\mu_1^\star}(\mathbb{X})_1|Z_1=c]=2\lambda^2\sigma^2(1+2\sigma^2)^2 < d \sigma^2 = \mathrm{Var}[\mathbb{X}_1|Z_1=c].
\]

\section{In-context clustering}\label{sec:incontext}
\subsection{Setting} 
So far, we have considered the traditional setting for model-based clustering, with a mixture of Gaussian made of two components: each token was assumed to be distributed as follows
$$
X_\ell \sim \frac{1}{2}\mathcal{N}(\mu_0^\star, \sigma^2 I_d)  + \frac{1}{2}\mathcal{N}(\mu_1^\star, \sigma^2 I_d) 
$$
with fixed centroids $\mu_0^\star$ and $\mu_1^\star$ of unit-norm and orthogonal.
We have shown that despite the non-convexity of the problem, attention-based predictors including two heads could perform approximate quantization and discover the underlying centroids encoded in their parameters. Note that for an input sequence of tokens $\mathbb{X}=(X_1 | \hdots |X_L)^\top \in\mathbb{R}^{L\times d}$, the first output of the type of attention-based predictor considered in this paper, parameterized by $\mu_0$ and $\mu_1$, can be rewritten as 
\begin{align*}
T^{\mathrm{lin},\mu_0,\mu_1}(\mathbb{X})_1 &=  \frac{2}{L} \sum_{\ell=1}^L \lambda (X_1^\top \mu_0 \mu_0^\top X_\ell) X_\ell + \lambda(X_1^\top \mu_1 \mu_1^\top X_\ell) X_\ell \\
&= \frac{2}{L} \sum_{\ell=1}^L \lambda X_1^\top (\mu_0\mu_0^\top + \mu_1 \mu_1^\top) X_\ell X_\ell.
\end{align*}
Therefore, when we consider two linear heads parameterized by row vectors for the queries and keys, and constrain them to be orthogonal, the setup can be interpreted as a single attention head with a query/key matrix of rank 2.
Interestingly, we are able to effectively train this rank-2 query/key head by leveraging the non-convex optimization of two simple, row-structured heads.

Let us now challenge the clustering setting, by assuming that each input sequence is still drawn from a 2-component mixture, but with its \textit{own} centroids, i.e., for each input sequence $\mathbb{X}_i = (X_{i1}, \ldots, X_{iL})^\top\in \mathbb{R}^{L\times d}$, $i=1,\hdots, n$, we assume the tokens $(X_{i\ell})_{\ell}$ to be i.i.d., such that the $\ell$-th token is distributed as
$$
X_{i\ell} | \mu_{i0}^\star, \mu_{i1}^\star \sim 
 \frac{1}{2}\mathcal{N}(\mu_{i0}^\star, \sigma^2 I_d)  + \frac{1}{2}\mathcal{N}(\mu_{i1}^\star, \sigma^2 I_d),
$$
for some random orthogonal centroids $\mu^\star_{i0}$ and $\mu^\star_{i1}$ of unit-norm.

If the prior distribution over the centroids is concentrated along preferred directions, say $\mu_0^{\star\star}$ and $\mu_1^{\star\star}$, then the predictor $T^{\mathrm{lin}, \mu_0, \mu_1}$ will likely perform well, as the Transformer's parameters will tend to align with these directions after training. Conversely, if the centroids are distributed in an isotropic way, $T^{\mathrm{lin}, \mu_0, \mu_1}$ will struggle with in-context clustering due to its limited flexibility, as only two parameters $\mu_0$ and $\mu_1$ govern the embedding, whereas the centroids vary significantly from one input sequence to another.
To address this issue, a natural idea is to give more degrees of freedom to the predictor by increasing the number of linear attention heads. 
Specifically, if we consider 
$d$ attention heads whose parameters are constrained to be row vectors $(\mu_c)_{c=1,\hdots, d}\in (\mathbb{R}^d)^d$, each of unit norm and mutually orthogonal, we obtain the following attention layer: for an input sequence $\mathbb{X}=(X_{1},\ldots , X_{L})^\top\in \mathbb{R}^{L\times d}$, and for $1\leq \ell \leq L$,
\begin{align*}
     T^{\mathrm{ctx}}(\mathbb{X})_\ell &=\sum_{c=1}^d H^{\mathrm{lin},\mu_c}(\mathbb{X})_{\ell} = \frac{2\lambda}{L} \sum_{c=1}^d \sum_{k=1}^L (X_{\ell}^\top \mu_c \mu_c^\top X_{k}) X_{k} = \frac{2\lambda}{L}  \sum_{k=1}^L \Big(X_{\ell}^\top \underbrace{\Big(\sum_{c=1}^d \mu_c \mu_c^\top\Big)}_{\mathrm{Id}} X_{k}\Big) X_{k}, 
\end{align*}
so finally,
\begin{align}
\label{def:attention_layer_in_context}
    T^{\mathrm{ctx}}(\mathbb{X})_\ell &=  \frac{2\lambda}{L}  \sum_{k=1}^L X_{\ell}^\top X_{k} X_{k}.
\end{align}
Somewhat surprisingly, using $d$ simplified linear heads in parallel, while enforcing orthogonality among their parameters, ultimately amounts to employing an attention layer with no trainable parameters. In what follows, we discuss the properties of $T^{\rm ctx}$ in an in-context clustering framework.

More formally, we refer to as the in-context clustering as a setting in which the input consists of a generic sequence of $L$ tokens $X_1,\hdots , X_L$, sampled from a Gaussian mixture model whose component means (centroids) are randomly drawn on the unit sphere:
\[
\left\{
\begin{array}{l}
\mu_0^\star \sim \mathcal{U}(\mathbb{S}^{d-1}) \qquad \text{and} \qquad 
\mu_{1}^\star \,|\,\mu_{0}^\star \quad \text{ 
arbitrarily distributed on } \mathbb{S}^{d-1} \cap (\mu_{0}^\star)^\perp \\
\\
X_1, \hdots , X_L \vert  \mu_0^\star, \mu_{1}^\star  \sim 
\frac{1}{2}\mathcal{N}(\mu_{0}^\star, \sigma^2 I_d)  + \frac{1}{2}\mathcal{N}(\mu_{1}^\star, \sigma^2 I_d)
\end{array}
\right.
\]
Associated with each token $X_\ell$, we still consider a latent variable $Z_\ell$, corresponding to a Bernoulli random variable of parameter $1/2$ and encoding its corresponding cluster, so that
$$
X_\ell | \mu_0^\star, \mu_1^\star , Z_\ell \sim \mathcal{N}(\mu_{Z_\ell}^\star , \sigma^2 I_d ).
$$

\subsection{Linear attention layers can perform in-context approximate quantization}

We first observe, that similarly to the fixed centroids setting of Section \ref{sec:statistical_properties}, the attention-based encoder $T^{\mathrm{ctx}}$  performs an approximate in-context quantization of the input distribution by aggregating sequences sampled from a Gaussian mixture. Specifically, as formalized below, this simple architecture effectively aligns its output with the correct centroid.
\begin{proposition}\label{def:attention_layer_in_contextsetting}
For $c\in\{0,1\}$, one has     
\begin{align*}
    \mathbb{E}\left[T^{\mathrm{ctx}}(\mathbb{X})_1| \mu_1^\star, \mu_0^\star, Z_1=c\right]&=\frac{2\lambda}{L}\left[(1+(d+2)\sigma^2)+(L-1)\left(\frac{1}{2}+\sigma^2\right)\right]\mu_c^\star.
\end{align*}
Choosing $\lambda=\frac{L}{2}\frac{1}{1+(d+2)\sigma^2+(L-1)\left(\frac{1}{2}+\sigma^2\right)}$ yields an unbiased embedding, i.e., $$\mathbb{E}\left[T^{\mathrm{ctx}}(\mathbb{X})_1| \mu_1^\star, \mu_0^\star, Z_1=c\right]=\mu_c^\star.$$
\end{proposition}

In what follows, we characterize the risk and variance of the attention-based embedding $T^{\mathrm{ctx}}$ defined in \eqref{def:attention_layer_in_context}, when the input sequence contains an infinite number of tokens.
\begin{proposition}\label{riskrandomsetting}
In the asymptotic regime where $L\rightarrow\infty$, one has that
\begin{align*}
   \lim_{L\rightarrow\infty}\mathcal{L}(T^{\mathrm{ctx}}) &=   (1+\sigma^2d)-2\lambda(1+4\sigma^2+2d\sigma^4)+4\lambda^2\left(2\left(\sigma^2+\frac{1}{2}\right)^3+(d-2)\sigma^6\right).
\end{align*}
Choosing the temperature $\lambda=\frac{1+4\sigma^2+2d\sigma^4}{4\left(2\left(\sigma^2+\frac{1}{2}\right)^3+(d-2)\sigma^6\right)}$ gives \begin{align*}
     \lim_{L\to \infty} \mathcal{L}(T^{\mathrm{ctx}})
    &=\sigma^2(d-2)\frac{1+2\sigma^2}{1+6\sigma^2+12\sigma^4+4d\sigma^6} \leq \sigma^2(d-2).
\end{align*}
Moreover, the conditional variance of the embedding satisfies when $L \to \infty$
\begin{align*}
\lim_{L\rightarrow\infty}\mathrm{Var}\left[T^{\mathrm{ctx}}(\mathbb{X})_1| \mu_1^\star,\mu_0^\star,Z_1=c\right]=2\lambda^2\sigma^2(1+4\sigma^2+2d\sigma^4).
\end{align*}
Choosing $\lambda=\frac{1}{1+2\sigma^2}$, we obtain an unbiased embedding with a  conditional variance of $$2\sigma^2\frac{1+4\sigma^2+2d\sigma^4}{(1+2\sigma^2)^2}.$$

\end{proposition}

As in the fixed centroids setting (Proposition~\ref{prop:risk_oracle_attention}),  we retrieve that for a suitable choice of $\lambda$, the loss satisfies 
\[\lim_{L\rightarrow\infty}\frac{\mathcal{L}(T^{\rm ctx})}{\mathcal{L}(T^\star)}
= \left(1-\frac{2}{d}\right)\frac{1+2\sigma^2}{1+6\sigma^2+12\sigma^4+4d\sigma^6}\leq \left(1-\frac{2}{d}\right).\]
\begin{wrapfigure}[12]{r}{0.46\textwidth}
\vspace{-0.6cm}
\includegraphics[width=\linewidth]{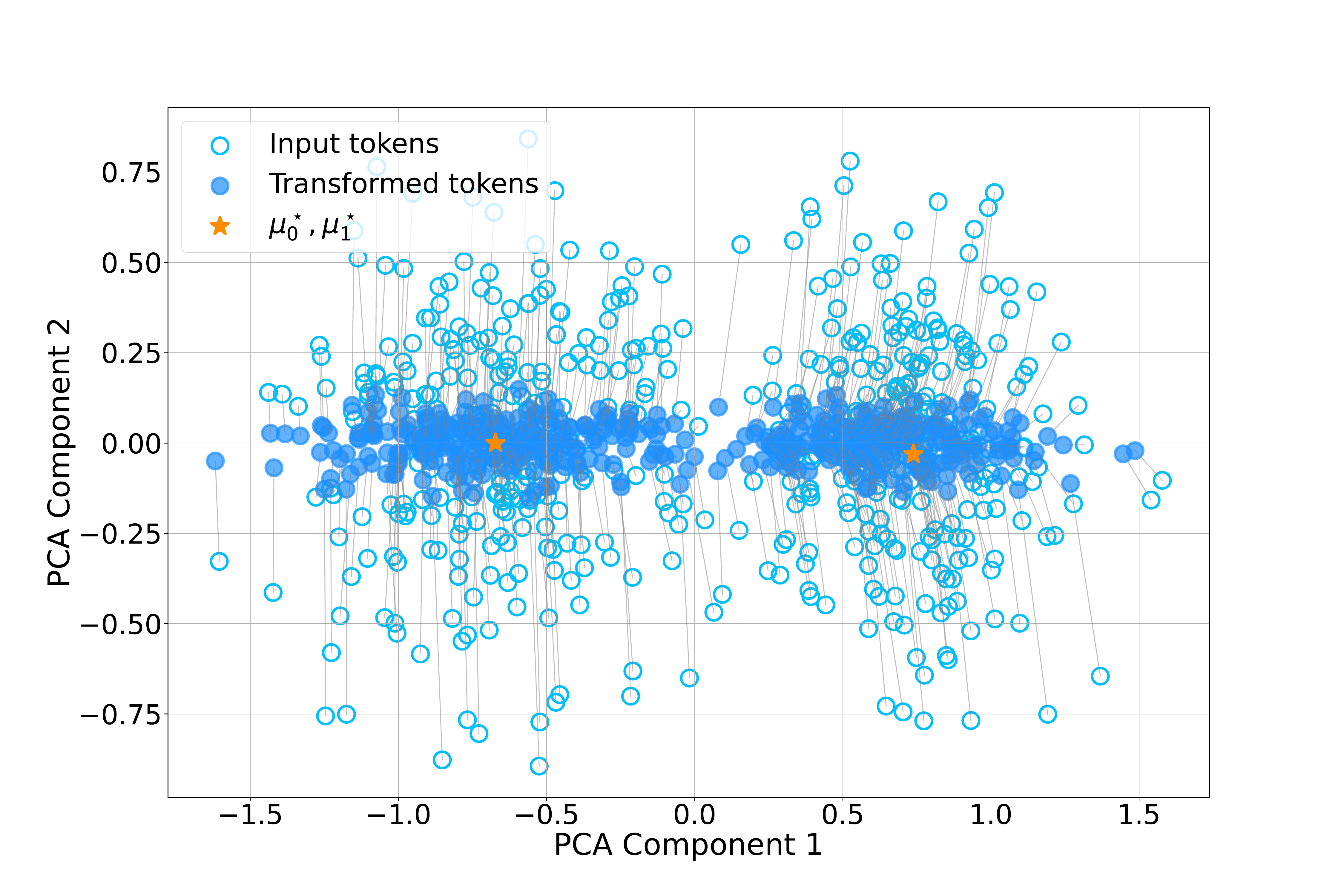}
    \caption{Comparison between inputs and embeddings ($d=10, \sigma=0.3$, $L=500$, $\lambda=\frac{1}{1+2\sigma^2}$, $\mu_0^\star=e_{10}, \mu_1^\star=-e_1$). PCA was fitted on the input tokens and was used to project both input and transformed tokens to 2D.
    }
    \label{fig:embedding2}
\end{wrapfigure}
This result is all the more surprising given that it emerges in a more complex setting, yet with a simpler mechanism. Unlike the attention-based predictor $T^{\mathrm{lin}, \mu_0,\mu_1}$, which benefits from access to the true centroids and a trainable architecture, the in-context encoder $T^{\mathrm{ctx}}$ achieves a smaller asymptotic risk without any learned parameters. The ability of $T^{\mathrm{ctx}}$ to extract meaningful representations illustrates the relevance of the key-query-value structure in the attention mechanism, even stripped of any learnable parameter and of the softmax nonlinearity.


It is worth noting that for $d\geq 2$, taking \( \lambda = \frac{1}{1 + 2\sigma^2} \),\vspace{0.2cm} 
\resizebox{.9\linewidth}{!}{\begin{minipage}{\linewidth}
    \begin{align*}
        &\lim_{L\rightarrow\infty}\mathrm{Var}\left[T^{\mathrm{ctx}}(\mathbb{X})_1| \mu_1^\star,\mu_0^\star,Z_1=c\right]\\
        &=2\sigma^2\frac{1+4\sigma^2+2d\sigma^4}{(1+2\sigma^2)^2}\\
        &\leq \sigma^2d = \mathrm{Var}\left[\mathbb{X}_1| \mu_1^\star,\mu_0^\star,Z_1=c\right].
    \end{align*}
\end{minipage}} 

Therefore, in the regime of infinite input sequences ($L \to \infty$) and for the appropriate value of $\lambda$, the embedding becomes unbiased and exhibits a variance reduction effect compared to the input data points, as illustrated in Figure~\ref{fig:embedding2} and previously discussed in Section~\ref{sec:statistical_properties}.

\color{black}
\section{Conclusion}\label{sec:conclusion}

This work offers a mathematically grounded, principled perspective on the unsupervised learning capabilities of attention mechanisms within mixture models. By combining a classical clustering framework with simplified yet non-trivial attention architectures, we present theoretical and empirical evidence showing that, when properly trained, attention layers can effectively recover latent structure in data.
Our analysis provides insight into the training dynamics, quantization behavior, and design choices, such as attention head regularization.
 
We further investigate an in-context setting, where attention-based models still perform efficient approximate quantization, and achieve lower risk than the optimal quantizer.
Future directions include exploring richer attention architectures, closer to that used in practice, which may further challenge the theoretical analysis.


\section*{Acknowledgments}

The authors would like to thank Gérard Biau and Lenka Zdeborová for insightful discussions.
This work was supported by the Projet Émergence(s) of the City of Paris. P.M.~is supported by a Google PhD Fellowship.
Part of this work was done while P.M.~was visiting the Simons Institute for the Theory of Computing.

\section*{Code availability}
Our code is available at \\ \url{https://github.com/rodrigomaulen/Attention-based-clustering} 



\bibliographystyle{plainnat}
\bibliography{citas}

\begin{thebibliography}{36}
\providecommand{\natexlab}[1]{#1}
\providecommand{\url}[1]{\texttt{#1}}
\expandafter\ifx\csname urlstyle\endcsname\relax
  \providecommand{\doi}[1]{doi: #1}\else
  \providecommand{\doi}{doi: \begingroup \urlstyle{rm}\Url}\fi

\bibitem[Ahn et~al.(2023)Ahn, Cheng, Daneshmand, and Sra]{ahn}
Kwangjun Ahn, Xiang Cheng, Hadi Daneshmand, and Suvrit Sra.
\newblock Transformers learn to implement preconditioned gradient descent for in-context learning.
\newblock In \emph{Thirty-seventh Conference on Neural Information Processing Systems}, 2023.
\newblock URL \url{https://openreview.net/forum?id=LziniAXEI9}.

\bibitem[Arnaboldi et~al.(2025)Arnaboldi, Loureiro, Stephan, Krzakala, and Zdeborova]{arnaboldi2025asymptotics}
Luca Arnaboldi, Bruno Loureiro, Ludovic Stephan, Florent Krzakala, and Lenka Zdeborova.
\newblock Asymptotics of sgd in sequence-single index models and single-layer attention networks.
\newblock \emph{arXiv preprint arXiv:2506.02651}, 2025.

\bibitem[Attouch(1996)]{attouch96}
Hedy Attouch.
\newblock Viscosity solutions of minimization problems.
\newblock \emph{SIAM Journal of Optimization}, 3:\penalty0 769--806, 1996.

\bibitem[Bahdanau et~al.(2015)Bahdanau, Cho, and Bengio]{bahdanau2015}
Dzmitry Bahdanau, Kyunghyun Cho, and Yoshua Bengio.
\newblock Neural machine translation by jointly learning to align and translate.
\newblock \emph{3rd International Conference on Learning Representations}, 2015.

\bibitem[Bahdanau et~al.(2016)Bahdanau, Chorowski, and Serdyuk]{bahdanau2016}
Dzmitry Bahdanau, Jan Chorowski, and Dmitriy Serdyuk.
\newblock Neural machine translation by jointly learning to align and translate.
\newblock \emph{3rd International Conference on Acoustics, Speech and Signal Processing (ICASSP)}, pages 4945--4949, 2016.

\bibitem[Boumal(2023)]{boumal}
Nicolas Boumal.
\newblock \emph{An Introduction to Optimization on Smooth Manifolds}.
\newblock Cambridge University, 2023.

\bibitem[Bradbury et~al.(2018)Bradbury, Frostig, Hawkins, Johnson, Leary, Maclaurin, Necula, Paszke, VanderPlas, Wanderman-Milne, and Zhang]{jax}
James Bradbury, Roy Frostig, Peter Hawkins, Matthew~James Johnson, Chris Leary, Dougal Maclaurin, George Necula, Adam Paszke, Jake VanderPlas, Skye Wanderman-Milne, and Qiao Zhang.
\newblock {JAX}: Composable transformations of {P}ython+{N}um{P}y programs.
\newblock \url{https://github.com/google/jax}, 2018.

\bibitem[Bubeck et~al.(2023)Bubeck, Chandrasekaran, Eldan, Gehrke, Horvitz, Kamar, Peter, Tat~Lee, Li, Lundberg, and et~al.]{bubeck}
Sébastien Bubeck, Varun Chandrasekaran, Ronen Eldan, Johannes Gehrke, Eric Horvitz, Ece Kamar, Lee Peter, Yin Tat~Lee, Yuanzhi Li, Scott Lundberg, and et~al.
\newblock Sparks of artificial general intelligence: Early experiments with gpt-4.
\newblock \emph{arXiv:2303.12712}, 2023.

\bibitem[Cui(2025)]{cui2025high}
Hugo Cui.
\newblock High-dimensional learning of narrow neural networks.
\newblock \emph{Journal of Statistical Mechanics: Theory and Experiment}, 2025\penalty0 (2):\penalty0 023402, 2025.

\bibitem[Cui et~al.(2024)Cui, Behrens, Krzakala, and Zdeborov\'{a}]{cui2024phase}
Hugo Cui, Freya Behrens, Florent Krzakala, and Lenka Zdeborov\'{a}.
\newblock A phase transition between positional and semantic learning in a solvable model of dot-product attention.
\newblock In \emph{Advances in Neural Information Processing Systems}, volume~37. Curran Associates, Inc., 2024.

\bibitem[Devlin et~al.(2018)Devlin, Chang, Lee, and Toutanova]{devlin}
Jacob Devlin, Ming-Wei Chang, Kenton Lee, and Kristina Toutanova.
\newblock Bert: Pre-training of deep bidirectional transformers for language understanding.
\newblock \emph{arXiv:1810.04805}, 2018.

\bibitem[Dosovitskiy et~al.(2020)Dosovitskiy, Beyer, Kolesnikov, and Weissenborn]{dosovitskiy}
Alexey Dosovitskiy, Lucas Beyer, Alexander Kolesnikov, and Dirk Weissenborn.
\newblock An image is worth 16x16 words: Transformers for image recognition at scale.
\newblock \emph{International Conference on Learning Representations}, 2020.

\bibitem[Furuya et~al.(2024)Furuya, de~Hoop, and Peyré]{furuya}
Takashi Furuya, Maarten~V. de~Hoop, and Gabriel Peyré.
\newblock Transformers are universal in-context learners, 2024.
\newblock URL \url{https://arxiv.org/abs/2408.01367}.

\bibitem[Garg et~al.(2023)Garg, Tsipras, Liang, and Valiant]{garg}
Shivam Garg, Dimitris Tsipras, Percy Liang, and Gregory Valiant.
\newblock What can transformers learn in-context? a case study of simple function classes, 2023.
\newblock URL \url{https://arxiv.org/abs/2208.01066}.

\bibitem[Han et~al.(2023)Han, Pan, Han, Song, and Huang]{han}
Dongchen Han, Xuran Pan, Yizeng Han, Shiji Song, and Gao Huang.
\newblock Flatten transformer: Vision transformer using focused linear attention, 2023.
\newblock URL \url{https://arxiv.org/abs/2308.00442}.

\bibitem[He and Hofmann(2024)]{simplifying}
Bobby He and Thomas Hofmann.
\newblock Simplifying transformer blocks.
\newblock In \emph{The Twelfth International Conference on Learning Representations}, 2024.
\newblock URL \url{https://openreview.net/forum?id=RtDok9eS3s}.

\bibitem[He et~al.(2025)He, Chen, Cao, Fan, and Liu]{he}
Yihan He, Hong-Yu Chen, Yuan Cao, Jianqing Fan, and Han Liu.
\newblock Transformers versus the em algorithm in multi-class clustering.
\newblock \emph{arXiv:2502.06007v1}, 2025.

\bibitem[Isserlis(1918)]{isserlis}
L.~Isserlis.
\newblock On a formula for the product-moment coefficient of any order of a normal frequency distribution.
\newblock \emph{Biometrika}, 12\penalty0 (1):\penalty0 134--139, 1918.

\bibitem[Katharopoulos et~al.(2020)Katharopoulos, Vyas, Pappas, and Fleuret]{rnn}
Angelos Katharopoulos, Apoorv Vyas, Nikolaos Pappas, and François Fleuret.
\newblock Transformers are rnns: Fast autoregressive transformers with linear attention, 2020.
\newblock URL \url{https://arxiv.org/abs/2006.16236}.

\bibitem[Lange(2013)]{lange}
Kenneth Lange.
\newblock \emph{Optimization}, volume 2 edition.
\newblock Springer, New York, 2013.

\bibitem[Li et~al.(2023)Li, Ildiz, Papailiopoulos, and Oymak]{li2023}
Yingcong Li, M.~Emrullah Ildiz, Dimitris Papailiopoulos, and Samet Oymak.
\newblock Transformers as algorithms: Generalization and stability in in-context learning, 2023.
\newblock URL \url{https://arxiv.org/abs/2301.07067}.

\bibitem[Li et~al.(2024)Li, Cao, Gao, He, Liu, Jason, Fan, and Wang]{li2024one}
Zihao Li, Yuan Cao, Cheng Gao, Yihan He, Han Liu, Klusowkski Jason, Jianqing Fan, and Mengdi Wang.
\newblock One-layer transformer provably learns one-nearest neighbor in context.
\newblock \emph{Advances in Neural Information Processing Systems}, 2024.

\bibitem[Liu et~al.(2021)Liu, Lin, Cao, Hu, Wei, Zhang, Lin, and Guo]{liu}
Ze~Liu, Yutong Lin, Yue Cao, Han Hu, Yixuan Wei, Zheng Zhang, Stephen Lin, and Baining Guo.
\newblock Swin transformer: Hierarchical vision transformer using shifted windows.
\newblock \emph{Proceedings of the IEEE/CVF international conference on computer vision}, pages 10012--10022, 2021.

\bibitem[Lloyd(1982)]{lloyd}
Stuart Lloyd.
\newblock Least squares quantization in {PCM}.
\newblock \emph{IEEE transactions on information theory}, 28(2):\penalty0 129--137, 1982.

\bibitem[Luong et~al.(2015)Luong, Pham, and Manning]{luong}
Thang Luong, Hieu Pham, and Christopher~D. Manning.
\newblock Effective approaches to attention-based neural machine translation.
\newblock \emph{Proceedings of the 2015 Conference on Empirical Methods in Natural Language Processing}, pages 1412--1421, 2015.

\bibitem[Marion and Berthier(2023)]{marion2023}
Pierre Marion and Raphael Berthier.
\newblock Leveraging the two timescale regime to demonstrate convergence of neural networks.
\newblock In \emph{Advances in Neural Information Processing Systems}, volume~36, 2023.

\bibitem[Marion et~al.(2024)Marion, Berthier, Biau, and Boyer]{marion2024attention}
Pierre Marion, Raphaël Berthier, Gérard Biau, and Claire Boyer.
\newblock Attention layers provably solve single-location regression, 2024.
\newblock URL \url{https://arxiv.org/abs/2410.01537}.

\bibitem[Noci et~al.(2023)Noci, Li, Li, He, Hofmann, Maddison, and Roy]{noci2023shaped}
Lorenzo Noci, Chuning Li, Mufan~Bill Li, Bobby He, Thomas Hofmann, Chris~J. Maddison, and Daniel~M. Roy.
\newblock The shaped transformer: Attention models in the infinite depth-and-width limit.
\newblock In \emph{Thirty-seventh Conference on Neural Information Processing Systems}, 2023.
\newblock URL \url{https://openreview.net/forum?id=PqfPjS9JRX}.

\bibitem[Phuong and Hutter(2022)]{phuong2022formal}
Mary Phuong and Marcus Hutter.
\newblock Formal algorithms for transformers.
\newblock \emph{arXiv:2207.09238}, 2022.

\bibitem[Ramachandran et~al.(2019)Ramachandran, Parmar, Vaswani, Bello, Levskaya, and Shlens]{stand}
Prajit Ramachandran, Niki Parmar, Ashish Vaswani, Irwan Bello, Anselm Levskaya, and Jonathon Shlens.
\newblock Stand-alone self-attention in vision models.
\newblock In \emph{33rd Conference on Neural Information Processing Systems (NeurIPS 2019), Vancouver, Canada}, 2019.

\bibitem[Shub(1987)]{shub}
Michael Shub.
\newblock \emph{Global Stability of Dynamical Systems}.
\newblock Springer, New York, 1987.

\bibitem[Troiani et~al.(2025)Troiani, Cui, Dandi, Krzakala, and Zdeborova]{troiani2025fundamental}
Emanuele Troiani, Hugo Cui, Yatin Dandi, Florent Krzakala, and Lenka Zdeborova.
\newblock Fundamental limits of learning in sequence multi-index models and deep attention networks: high-dimensional asymptotics and sharp thresholds.
\newblock In \emph{Forty-second International Conference on Machine Learning}, 2025.

\bibitem[Vaswani et~al.(2017)Vaswani, Shazeer, Parmar, Uszkoreit, Jones, Gomez, Kaiser, and Polosukhin]{vaswani}
Ashish Vaswani, Noam Shazeer, Niki Parmar, Jakob Uszkoreit, Llion Jones, Aidan~N. Gomez, Lukasz Kaiser, and Illia Polosukhin.
\newblock Attention is all you need.
\newblock \emph{Advances in Neural Information Processing Systems}, 30:\penalty0 6000--6010, 2017.

\bibitem[von Oswald et~al.(2023)von Oswald, Niklasson, Randazzo, Sacramento, Mordvintsev, Zhmoginov, and Vladymyrov]{oswald}
Johannes von Oswald, Eyvind Niklasson, Ettore Randazzo, João Sacramento, Alexander Mordvintsev, Andrey Zhmoginov, and Max Vladymyrov.
\newblock Transformers learn in-context by gradient descent.
\newblock In \emph{International Conference on Machine Learning}, pages 35151--35174, 2023.
\newblock URL \url{https://arxiv.org/abs/2212.07677}.

\bibitem[Yang et~al.(2025)Yang, Wang, Lee, and Liang]{yang}
Hongru Yang, Zhangyang Wang, Jason~D. Lee, and Yingbin Liang.
\newblock Transformers provably learn two-mixture of linear classification via gradient flow.
\newblock In \emph{The Thirteenth International Conference on Learning Representations}, 2025.
\newblock URL \url{https://openreview.net/forum?id=AuAj4vRPkv}.

\bibitem[Zhang et~al.(2024)Zhang, Frei, and Bartlett]{zhang2024trained}
Ruiqi Zhang, Spencer Frei, and Peter~L Bartlett.
\newblock Trained transformers learn linear models in-context.
\newblock \emph{Journal of Machine Learning Research}, 25\penalty0 (49):\penalty0 1--55, 2024.

\end{thebibliography}

\appendix

\section{Training dynamics in the degenerate case}
\label{sec:dirac}

In this section, we discuss the training behavior of the Transformer-based predictor $T^{\mathrm{lin}, \mu_0, \mu_1}$ in the context of clustering, assuming the data are drawn from the degenerate mixture model, where
for $1\leq \ell\leq L$, 
\begin{equation}
\label{def:dirac_mixture_model}
\tag{$\mathrm{P}_{0}$}
X_\ell\sim \frac{1}{2}\delta_{\mu_0^\star}+\frac{1}{2}\delta_{\mu_1^\star},
\end{equation}
(the orthogonal centroids $\mu_0^\star$ and $\mu_1^\star$ still lie on the unit sphere). We emphasize that despite its apparent simplicity, this study framework is already sufficient to reveal some of the complexity inherent in the clustering task carried out by a self-attention layer. 

\subsection{Theoretical analysis}

Since training is performed by minimizing the risk $\mathcal{R}$, the first steps of our analysis focus on studying the critical points and extrema of this risk. 

\paragraph{Critical points and minimizers.} First, we reparameterize the problem using the quantities
\begin{equation}\label{notation0}
\kappa_0 \eqdef\langle\mu_0^\star,\mu_0\rangle, \quad \kappa_1\eqdef\langle\mu_1^\star,\mu_1\rangle, 
\quad 
\eta_0\eqdef\langle \mu_1,\mu_0^\star\rangle,
\quad \eta_1\eqdef\langle \mu_0,\mu_1^\star\rangle.
\end{equation}
The scalar products $\kappa_0$ and $\kappa_1$  measure the alignment of the parameters $\mu_0$ and $\mu_1$ with the true centroids $\mu_0^\star$ and $\mu_1^\star$, respectively, while the scalar products $\eta_0$ and $\eta_1$  capture their orthogonality with the inverted centroids $\mu_1^\star$ and $\mu_0^\star$. The theoretical risk w.r.t.\ $\kappa_0, \kappa_1,\eta_0$ and $\eta_1$ reads as follows. \pierremodif{The proof of this result and the following are given in Appendix \ref{sec:app_degenerate}}.

\begin{proposition}  \label{risknoiseless} 
Under the degenerate mixture model $\eqref{def:dirac_mixture_model}$, considering the attention-based predictor $T^{{\rm lin}, \mu_0,\mu_1}$ composed of two linear heads parameterized by $\mu_0$ and $\mu_1$, the theoretical risk $\mathcal{R}$ can be re-expressed as a function $\mathcal{R}^<:\R^4\rightarrow\R$ such that $\mathcal{R}(\mu_0,\mu_1)=\mathcal{R}^<(\kappa_0,\kappa_1,\eta_0,\eta_1)$, where
\begin{align}
\begin{split}
    \mathcal{R}^<(\kappa_0,\kappa_1,\eta_0,\eta_1)&=1-\lambda\frac{L+1}{L}(\kappa_0^2+\kappa_1^2+\eta_0^2+\eta_1^2)+\lambda^2\frac{L+3}{2L}([\kappa_0^2+\eta_0^2]^2+[\kappa_1^2+\eta_1^2]^2)\\
    &\qquad +\lambda^2\frac{L-1}{L}(\kappa_0\eta_1+\kappa_1\eta_0)^2.
\end{split}
\end{align}
In addition, if $(\mu_0,\mu_1)\in(\mathbb{S}^{d-1})^2$ are prescribed to the unit sphere, then $\mathrm{dom}(\mathcal{R}^<)=[-1,1]^4$.
\end{proposition}

\begin{remark}
After a direct computation, we note that the critical points of the risk $\mathcal{R}$ correspond to those of its reparameterized version, $\mathcal{R}^<$, i.e., 
$$(\mu_0,\mu_1)\in\mathrm{crit}(\mathcal{R})\iff (\kappa_0,\kappa_1,\eta_0,\eta_1) \in\mathrm{crit}(\mathcal{R}^<).$$    
\end{remark}

\begin{proposition}[Characterization of global minima]
\label{prop:optimality_condition_dirac_model}
Consider $\mathcal{R}^<:\R^4\rightarrow\R$ defined as in Proposition \ref{risknoiseless} with $\lambda_0^\star=\frac{L+1}{L+3}$, then a point $(\kappa_0,\kappa_1,\eta_0,\eta_1)$ belongs to $\mathrm{argmin}(\mathcal{R}^<)$ if and only if
\begin{equation}\label{eqminima}
\kappa_0^2+\eta_0^2=1, \quad\kappa_1^2+\eta_1^2=1, \quad \text{and} \quad 
\kappa_0\eta_1+\kappa_1\eta_0=0.
\end{equation}
\end{proposition}

While the characterization can be made for any value of $\lambda$, choosing $\lambda_0^\star = \frac{L+1}{L+3}$ simplifies the system by setting the first two conditions equal to 1. Moreover, this specific value provides a critical upper bound on the temperature parameter that guarantees recovery of the underlying centroids via risk minimization, as highlighted in the theorem below, and discussed in Remark \ref{lambdachoice} in the appendices.

\pierremodif{The proof of this result in Appendix \ref{sec:app_degenerate} also characterizes all critical points, beyond global minima.}


\paragraph{Convergence analysis.} 
From the characterization of the minima of $\mathcal{R}$ given in Proposition \ref{prop:optimality_condition_dirac_model}, we observe that the points $(\kappa_0, \kappa_1, \eta_0, \eta_1)$ that saturate the first two equations (i.e., satisfy $\kappa_0^2 = \kappa_1^2 = 1$ and $\eta_0^2 = \eta_1^2 = 0$)  correspond to global minimizers of the risk that also recover the centroids (up to a sign). However, in general, other global minimizers may exist that do not exhibit this saturation behavior and are therefore disconnected from centroid recovery. In the next result, we show that under appropriate initialization conditions, the \eqref{eq:PGD_iterates} algorithm converges to the desired global minimum, which aligns with the clustering objective. To this end, we introduce the following manifold
\begin{align}
    \tilde{\mathcal{M}}=\{(\mu_0,\mu_1)\in(\mathbb{S}^{d-1})^2: \langle\mu_1^\star,\mu_0\rangle=0, \langle\mu_0^\star,\mu_1\rangle=0\}.
\end{align}
\begin{theorem}\label{main0}
    Under the Dirac mixture model $\eqref{def:dirac_mixture_model}$, consider the attention-based predictor $T^{{\rm lin}, \mu_0,\mu_1}$ composed of two linear heads. 
    Take $\lambda\in ]0, \frac{L+1}{L+3}]$. Then there exists $\bar{\gamma}>0$ such that for any stepsize $0<\gamma<\bar{\gamma}$, and for a generic initialization $(\mu_0^0,\mu_1^0)\in \tilde{\mathcal{M}}$, the sequence of iterates $(\mu_0^k, \mu_1^k)$, generated by \eqref{eq:PGD_iterates}, converges to the centroids (up to a sign), i.e., 
    $$
    (\mu_0^k,\mu_1^k)\xrightarrow[k\rightarrow\infty]{}(\pm \mu_0^\star,\pm \mu_1^\star).
    $$

\end{theorem}
\begin{proof}
    This is a direct consequence of Theorem \ref{thm:main_gmm} with $\sigma=0$.
\end{proof}
This result demonstrates that, despite the non-convexity of the objective function, the key and query row matrices of a linear attention layer trained via \eqref{eq:PGD_iterates} align with the centroids of the underlying Dirac mixture. Although the setting is simplified, it already highlights the representational role of key and query matrices in attention-based learning, and serves as a foundation for addressing the more general case of Gaussian mixtures. 

Note that the convergence is up to a sign, a consequence of the symmetry inherent in $H^{\mathrm{lin},\mu}$. Nonetheless, this sign ambiguity does not affect the output of the attention layer. To resolve this ambiguity and identify the true centroids, one could compare likelihoods or perform a hard assignment step, selecting the centroid that minimizes the total distance to all points within its assigned cluster.
Besides, by generic initialization, we mean that the set of initializations $(\mu_0^0,\mu_1^0)\in \tilde{\mathcal{M}}$ for which \eqref{eq:PGD_iterates} fails to recover the centroids is of Lebesgue measure zero with respect to $\tilde{\mathcal{M}}$.  

Initialization on the manifold \(\tilde{\mathcal{M}}\) relies on prior knowledge of the centroids, which may be impractical. While a theoretical analysis under generic initialization on the unit sphere would be of interest, it remains analytically intractable due to the complexity of the resulting dynamical system derived from \eqref{eq:PGD_iterates}. In the following, however, we present numerical experiments incorporating a regularization term that proves effective in solving the problem without initialization constraints.


\subsection{Numerical experiments}\label{xp_dirac}

In this section, we study the empirical convergence of the \eqref{eq:PGD_iterates} iterates when the data follows the degenerate mixture model \eqref{def:dirac_mixture_model}.

\paragraph{Results.}
Figure \ref{fig:ig:risk_linear_noiseless_manifold} clearly illustrates that when initialized on the manifold $\tilde{\mathcal{M}}$, the \eqref{PGDrho} iterates, over the objective function $\mathcal{R}$, converge to the centroids, as established in Theorem \ref{main0}. 
\begin{figure}
    \begin{subfigure}[b]{0.48\textwidth}
        \centering
        \includegraphics[width=\textwidth]{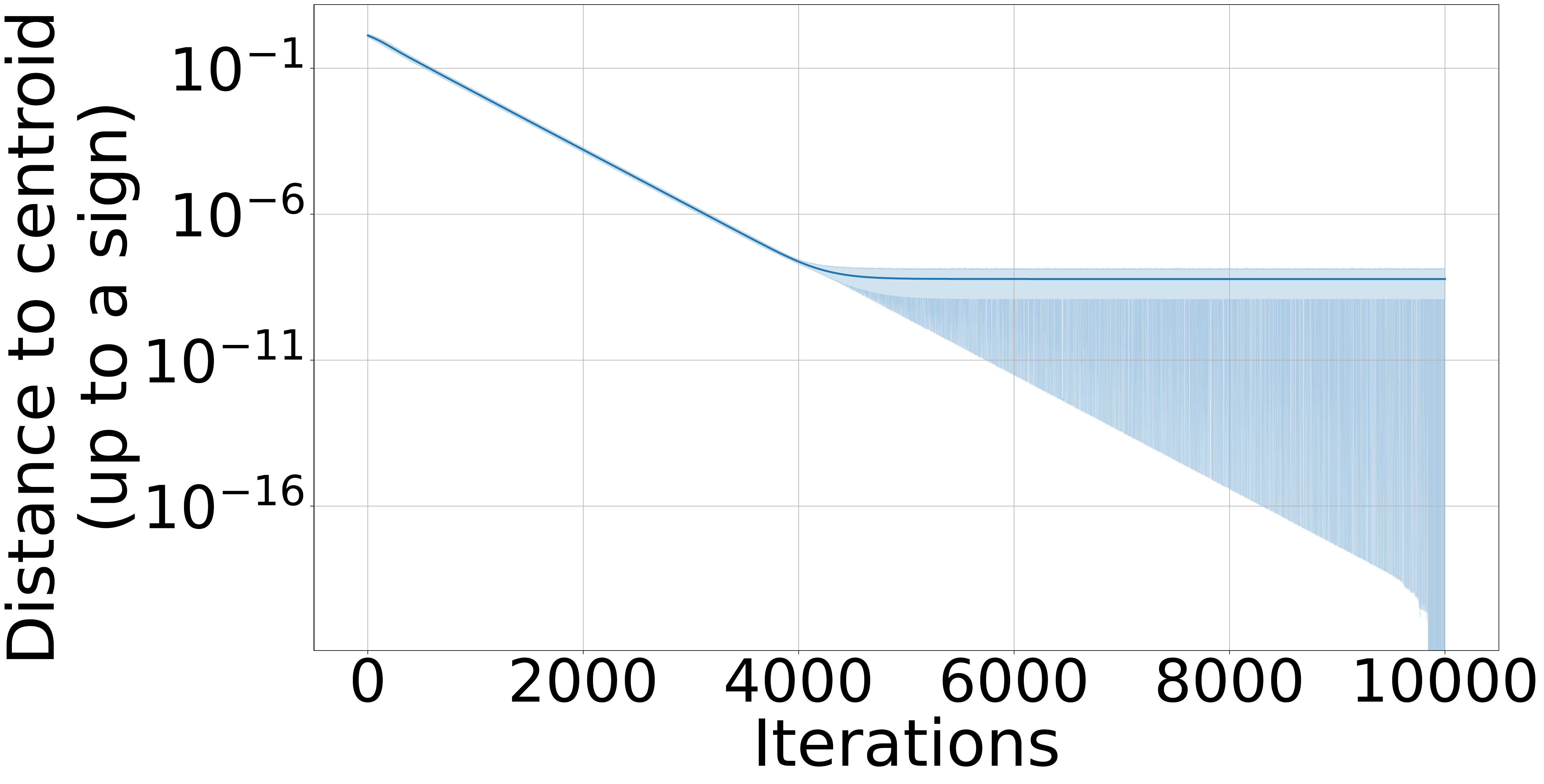}
         \caption{Initialization on the manifold $\tilde{\mathcal{M}}$}
         \label{fig:ig:risk_linear_noiseless_manifold}
     \end{subfigure}
        \begin{subfigure}[b]{0.48\textwidth}
         \centering
         \includegraphics[width=\textwidth]{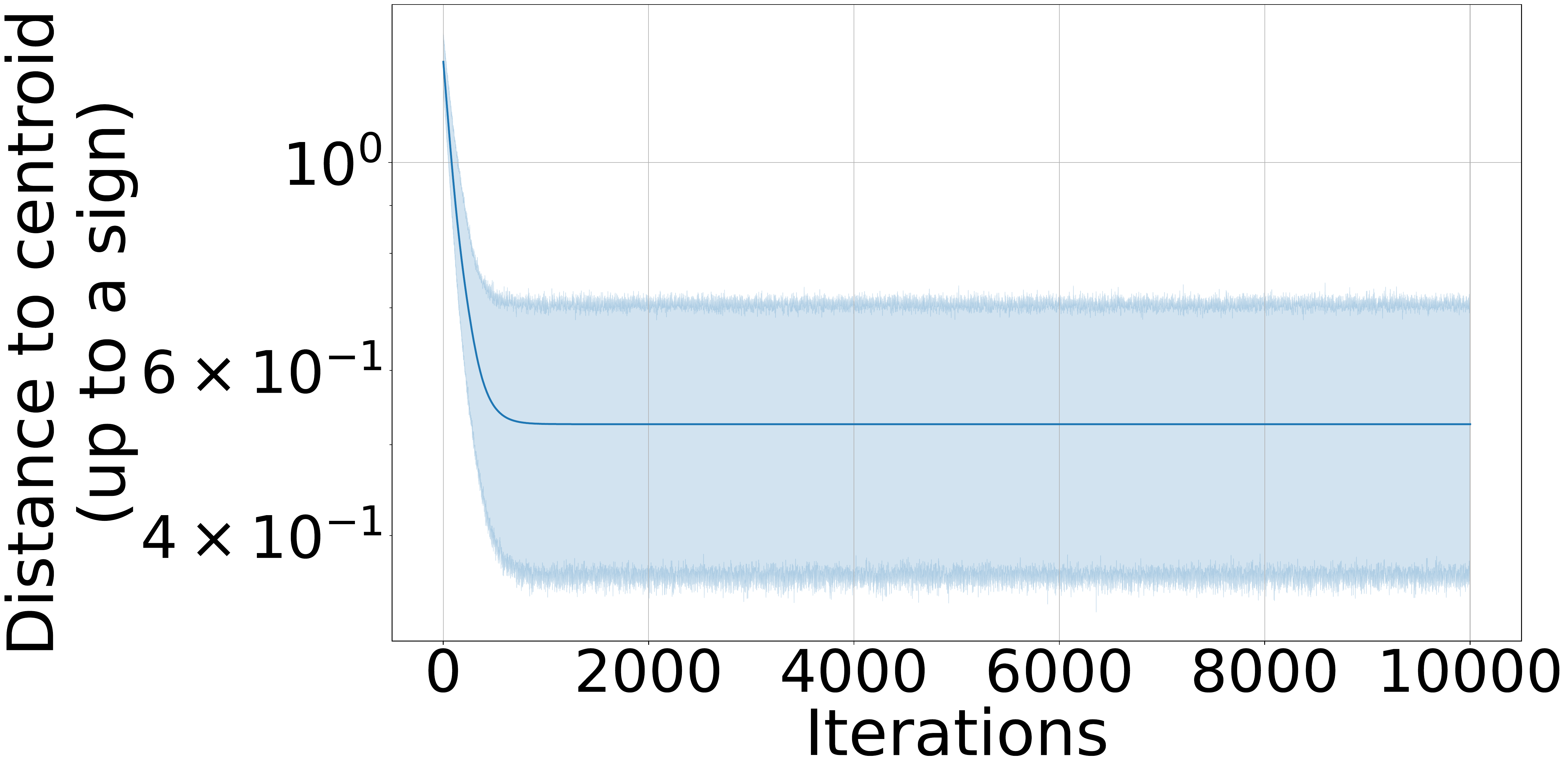}
         \caption{Initialization on the sphere $(\mathbb{S}^{d-1})^2$}
         \label{fig:risk_linear_noiseless_sphere}
     \end{subfigure}
    \caption{Distance to centroids vs \eqref{PGDrho} iterations for the minimization of $\mathcal{R}$, with data drawn from the degenerate case $\eqref{def:dirac_mixture_model}$. 10 runs, 95\% percentile intervals are plotted.}
    \label{fig:risk_linear_noiseless}
\end{figure}


The situation differs outside the manifold, where numerical evidence shows that the Transformer parameters only partially align with the true centroids as shown in Figure \ref{fig:risk_linear_noiseless_sphere}. In fact, we observe empirically that each parameter learns a mixture of both centroids.  This indicates that the \eqref{PGDrho} iterates may converge to optima that do not coincide with the underlying centroids. To mitigate this and better guide the learning process, we propose using a specific form of regularization:
\begin{equation}
r(\mu_0,\mu_1)\eqdef\mathbb{E}[\langle \mu_0,X_1\rangle^2\langle \mu_1,X_1\rangle^2].    
\end{equation}
Therefore, we train the attention-based predictor $H^{{\rm lin}, \mu_0, \mu_1}$ now by minimizing the regularized risk
\begin{equation}\label{penalization}\tag{$\tilde{\mathcal{P}}_{\rho
}$}
    \min_{\mu_0,\mu_1\in\mathbb{S}^{d-1}} \mathcal{R}^{\rho}(\mu_0,\mu_1) \qquad \text{with} \qquad \mathcal{R}^{\rho}(\mu_0,\mu_1) \eqdef\mathcal{R}(\mu_0,\mu_1)+\rho r(\mu_0,\mu_1),
\end{equation}
where $\rho>0$ denotes the strength of the regularization. 
It can be rigorously shown that as $\rho$ approaches 0, the minimizers of $\mathcal{R}^\rho$ converge to those of $\mathcal{R}$, exhibiting the saturation phenomenon, desirable to bolster the interpretability of the attention heads.
We refer to Appendix \ref{discreg} for more details.


\begin{figure}[ht]
    \centering
    \begin{subfigure}[b]{0.48\textwidth}
         \centering
         \includegraphics[width=\textwidth]{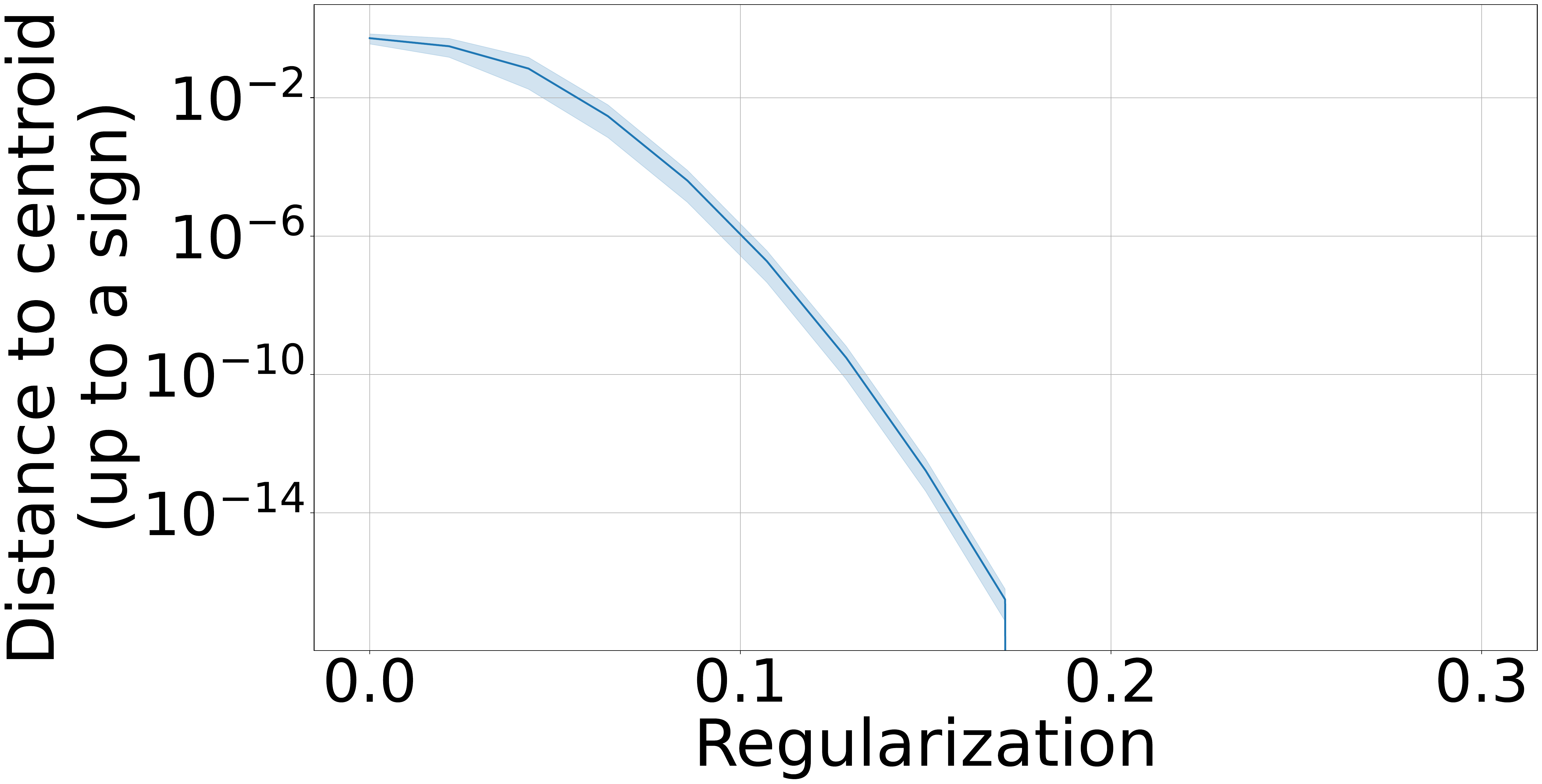}
         \caption{Distance to centroids vs. regularization strength}
         \label{linear_noiseless_reg}
     \end{subfigure}
    \begin{subfigure}[b]{0.48\textwidth}
         \centering
         \includegraphics[width=\textwidth]{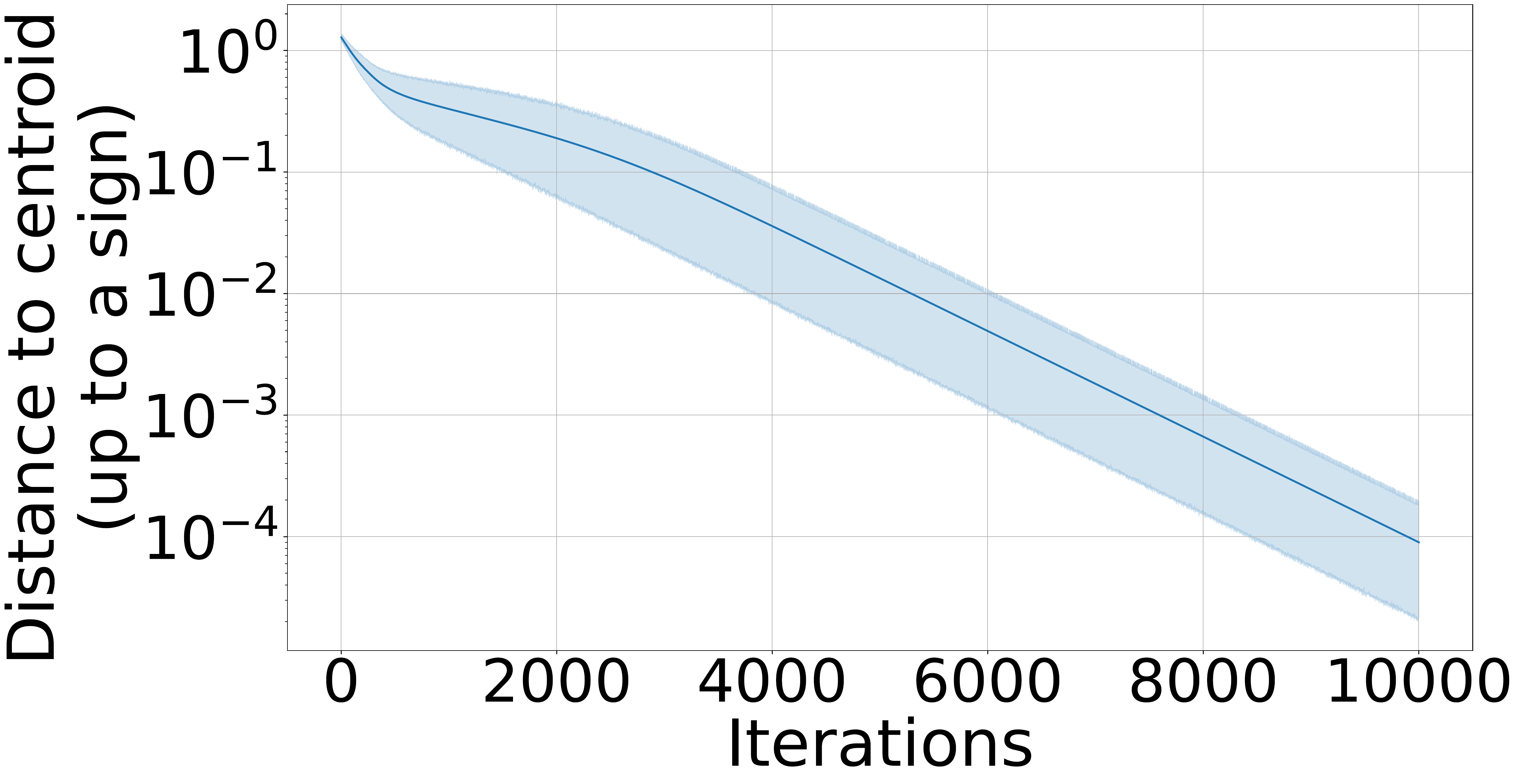}
         \caption{Distance to centroids vs iterations ($\rho=0.1$).}
         \label{linear_noiseless_iters}
     \end{subfigure}
    \caption{Convergence analysis of the \eqref{PGDrho} iterates for the minimization of the regularized risk $\mathcal{R}^\rho$, under the degenerate mixture case \eqref{def:dirac_mixture_model}. 10 runs, 95\% percentile intervals are plotted.}
\end{figure}

In Figure \ref{linear_noiseless_reg}, we observe that a relatively small regularization parameter (of the order of $10^{-1}$) is sufficient to achieve centroid alignment, with numerical error below $10^{-14}$. In Figure \ref{linear_noiseless_iters}, we fix the regularization parameter and observe that over the course of $10^4$ iterations, the attention head parameters exhibit linear convergence towards the true centroids. This numerical experiment highlights the effectiveness of this form of regularization in enhancing the interpretability of attention heads—by promoting their disentanglement—in the context of mixture models.

\section{Proofs of Section \ref{sec:dirac} (degenerate case)}   

\label{sec:app_degenerate}
In this section, we present the postponed proofs that support and elaborate on the arguments developed in the main text. We begin by characterizing the critical points of the Dirac mixture risk, then proceed to a discussion on the effects of a regularization term in the Dirac setting. Finally, we outline the proofs of Theorems~\ref{main0} and~\ref{thm:main_gmm}, which constitute the main theoretical results of our work. 

\subsection{Proof of Proposition \ref{risknoiseless} (expression of the risk in the degenerate case)}

To facilitate the analysis that follows, we introduce the notation  $e_k(\mu) \eqdef \lambda X_1^\top\mu\mu^\top X_k$, for $1\leq k \leq L$, which allows us to write
$$
\mathcal{R}(\mu_0,\mu_1)=\mathbb{E}\left[\Big\Vert X_1-\frac{2}{L}\sum_{k=1}^L (e_k(\mu_0)+e_k(\mu_1))X_k\Big\Vert_2^2\right].
$$

\allowdisplaybreaks

In what follows, we give an expression of the risk of an attention-based predictor, in the case where the data is distributed according to the Dirac mixture model \eqref{def:dirac_mixture_model}. Then, the risk of an attention-based predictor $T^{\mathrm{lin}, \mu_0, \mu_1}$ can be written as 
    \begin{align*}
    \mathcal{R}(\mu_0,\mu_1)&=\mathbb{E}\left[\Big\Vert X_1-\frac{2}{L}\sum_{k=1}^L (e_k(\mu_0)+e_k(\mu_1))X_k\Big\Vert_2^2\right]\\
    &=\mathbb{E}\left[\Vert X_1\Vert^2\right]-\frac{4}{L}\sum_{k=1}^L\mathbb{E}\big[\langle X_1, (e_k(\mu_0)+e_k(\mu_1))X_k\rangle\big]\\
    &+\frac{4}{L^2}\mathbb{E}\left[\Big\Vert \sum_{k=1}^L (e_k(\mu_0)+e_k(\mu_1))X_k\Big\Vert^2\right]\\
     &=1-\frac{4}{L}\sum_{k=1}^L\mathbb{E}\big[\langle X_1, (e_k(\mu_0)+e_k(\mu_1))X_k\rangle\big]+\frac{4}{L^2} \sum_{k=1}^L \mathbb{E}[\Vert(e_k(\mu_0)+e_k(\mu_1))X_k\Vert^2]\\
     &+\frac{8}{L^2}\sum_{1\leq k<j\leq L }\mathbb{E}\big[  (e_k(\mu_0)+e_k(\mu_1))(e_j(\mu_0)+e_j(\mu_1))\langle X_k, X_j\rangle\big]\\
    &=1-\frac{4}{L}\underbrace{\mathbb{E}[(e_1(\mu_0)+e_1(\mu_1))\Vert X_1\Vert^2]}_{\eqdef (I_0)}+\frac{4}{L^2}\underbrace{\mathbb{E}\left[\Vert  (e_1(\mu_0)+e_1(\mu_1))X_1\Vert^2\right]}_{\eqdef (II_0)} \\
    &+\frac{8}{L^2}\underbrace{\sum_{k=2}^L\mathbb{E}\big[  (e_1(\mu_0)+e_1(\mu_1))(e_k(\mu_0)+e_k(\mu_1))\langle X_1, X_k\rangle\big]}_{\eqdef (III_0)}\\
    &-\frac{4}{L}\underbrace{\sum_{k=2}^L\mathbb{E}\big[\langle X_1, (e_k(\mu_0)+e_k(\mu_1))X_k\rangle\big]}_{\eqdef(I)}+\frac{4}{L^2}\underbrace{\sum_{k=2}^L\mathbb{E}\left[\Vert  (e_k(\mu_0)+e_k(\mu_1))X_k\Vert^2\right]}_{\eqdef (II)}\\
    &+\frac{8}{L^2}\underbrace{\sum_{1< k<j\leq L}\mathbb{E}\big[  (e_k(\mu_0)+e_k(\mu_1))(e_j(\mu_0)+e_j(\mu_1))\langle X_k, X_j\rangle\big]}_{\eqdef (III)}.
\end{align*}
We compute $(I)$ by conditioning on $Z_1,Z_k$, \begin{align*}
    \mathbb{E}\big[(e_k(\mu_0)+e_k(\mu_1))\langle X_1, X_k\rangle\big]&=\mathbb{E}\big[\mathbb{E}[(e_k(\mu_0)+e_k(\mu_1))\langle X_1, X_k\rangle|Z_1,Z_k]\big]\\
    & =\lambda\mathbb{E}[(\langle \mu_{Z_1}^\star,\mu_0\rangle\langle \mu_{Z_k}^\star,\mu_0\rangle+\langle \mu_{Z_1}^\star,\mu_1\rangle\langle \mu_{Z_k}^\star,\mu_1\rangle)\langle \mu_{Z_1}^\star,\mu_{Z_k}^\star\rangle]\\
    &=\lambda\frac{\kappa_0^2+\kappa_1^2+\eta_0^2+\eta_1^2}{4}.
\end{align*}
This leads to $(I)=\lambda\frac{L-1}{L}(\kappa_0^2+\kappa_1^2+\eta_0^2+\eta_1^2)$.\\

Similarly for $(II)$, conditioning on $Z_1,Z_k$, 
\begin{align*}
    \mathbb{E}\left[\Vert  (e_k(\mu_0)+e_k(\mu_1))X_k\Vert^2\right]&=\mathbb{E}\left[\mathbb{E}[\Vert  (e_k(\mu_0)+e_k(\mu_1))X_k\Vert^2|Z_1,Z_k]\right]\\
    &=\lambda^2\mathbb{E}[\Vert(\langle \mu_{Z_1}^\star,\mu_0\rangle\langle \mu_{Z_k}^\star,\mu_0\rangle+\langle \mu_{Z_1}^\star,\mu_1\rangle\langle \mu_{Z_k}^\star,\mu_1\rangle)\mu_{Z_k}^\star\Vert^2]\\
&=\lambda^2\frac{[\kappa_0^2+\eta_0^2]^2+[\kappa_1^2+\eta_1^2]^2}{4}+\lambda^2\frac{(\kappa_0\eta_1+\kappa_1\eta_0)^2}{2}.
\end{align*}
Which gives $(II)=\lambda^2\frac{L-1}{L^2}([\kappa_0^2+\eta_0^2]^2+[\kappa_1^2+\eta_1^2]^2)+\lambda^2\frac{2(L-1)}{L^2}(\kappa_0\eta_1+\kappa_1\eta_0)^2$.\\

Finally, to compute $(III)$, note that
\begin{align*}
    \mathbb{E}&\big[  (e_k(\mu_0)+e_k(\mu_1))(e_j(\mu_0)+e_j(\mu_1))\langle X_k, X_j\rangle\big]\\
    &=\mathbb{E}\big[ \mathbb{E}[ (e_k(\mu_0)+e_k(\mu_1))(e_j(\mu_0)+e_j(\mu_1))\langle X_k, X_j\rangle|Z_1,Z_k,Z_j]\big]\\
    &=\lambda^2\mathbb{E}[(\langle \mu_{Z_1}^\star,\mu_0\rangle\langle \mu_{Z_k}^\star,\mu_0\rangle+\langle \mu_{Z_1}^\star,\mu_1\rangle\langle \mu_{Z_k}^\star,\mu_1\rangle)(\langle \mu_{Z_1}^\star,\mu_0\rangle\langle \mu_{Z_j}^\star,\mu_0\rangle+\langle \mu_{Z_1}^\star,\mu_1\rangle\langle \mu_{Z_j}^\star,\mu_1\rangle)\\
    &\quad\cdot\langle \mu_{Z_k}^\star,\mu_{Z_j}^\star\rangle]\\
    &=\lambda^2\frac{[\kappa_0^2+\eta_0^2]^2+[\kappa_1^2+\eta_1^2]^2}{8}+\lambda^2\frac{(\kappa_0\eta_1+\kappa_1\eta_0)^2}{4},
\end{align*}
leading to $(III)=\lambda^2\frac{(L-1)(L-2)}{2L^2}[[\kappa_0^2+\eta_0^2]^2+[\kappa_1^2+\eta_1^2]^2+2(\kappa_0\eta_1+\kappa_1\eta_0)^2]$.\\

In a similar fashion, we get that \begin{align*}
    (I_0)&=(\kappa_0^2+\kappa_1^2+\eta_0^2+\eta_1^2),\\
    (II_0)&=[\kappa_0^2+\eta_0^2]^2+[\kappa_1^2+\eta_1^2]^2,\\
    (III_0)&=[\kappa_0^2+\eta_0^2]^2+[\kappa_1^2+\eta_1^2]^2.
\end{align*}

Putting everything together, we obtain that the risk can be written in terms of $\kappa_0,\kappa_1,\eta_0,\eta_1$, i.e., $\mathcal{R}(\mu_0,\mu_1)=\mathcal{R}^<(\kappa_0,\kappa_1,\eta_0,\eta_1)$, where: 
\begin{align*}    \mathcal{R}^<&\eqdef 1-\lambda\left[\frac{2}{L}+\frac{L-1}{L}\right](\kappa_0^2+\kappa_1^2+\eta_0^2+\eta_1^2)\\
&+\lambda^2\left[\frac{2}{L^2}+\frac{2(L-1)}{L^2}+\frac{L-1}{L^2}+\frac{(L-1)(L-2)}{2L^2}\right]([\kappa_0^2+\eta_0^2]^2+[\kappa_1^2+\eta_1^2]^2)\\
&+\lambda^2\left[\frac{(L-1)(L-2)}{L^2}+\frac{2(L-1)}{L^2}\right](\kappa_0\eta_1+\kappa_1\eta_0)^2\\
    &=1-\lambda\frac{L+1}{L}(\kappa_0^2+\kappa_1^2+\eta_0^2+\eta_1^2)+\lambda^2\frac{L+3}{2L}([\kappa_0^2+\eta_0^2]^2+[\kappa_1^2+\eta_1^2]^2)\\
    &+\lambda^2\frac{L-1}{L}(\kappa_0\eta_1+\kappa_1\eta_0)^2.
\end{align*}

\subsection{Proof of Proposition \ref{prop:optimality_condition_dirac_model} (critical points of the risk in the degenerate case)}
\label{criticalnoiselesspoints}
\begin{proposition}\label{critical}
    Consider $\mathcal{R}^<:\R^4\rightarrow\R$ defined as in Proposition \ref{risknoiseless} with $\lambda=\lambda_0^\star=\frac{L+1}{L+3}$, then we characterize its critical points by
     \begin{enumerate}
        \item The point $(0,0,0,0)$ is a local maximum.
        \item The points $(\kappa_0,0,\eta_0,0)$, where $\kappa_0^2+\eta_0^2=1,$ and $(0,\kappa_1,0,\eta_1)$, where $\kappa_1^2+\eta_1^2=1$, are strict saddle points.
        \item The points $(\kappa_0,\kappa_1,\kappa_1,\kappa_0)$ and $(\kappa_0,\kappa_1,-\kappa_1,-\kappa_0)$, where $\kappa_0^2+\kappa_1^2=\frac{L+3}{2(L+1)}$, are strict saddle points. 
         \item $(\kappa_0,\kappa_1,\eta_0,\eta_1)$ belongs to $\mathrm{argmin}(\mathcal{R}^<)$ if and only if:
\begin{equation}
\left\{
\begin{array}{rl}
\kappa_0^2+\eta_0^2&=1,\\
\kappa_1^2+\eta_1^2&=1,\\
\kappa_0\eta_1+\kappa_1\eta_0&=0.\\
\end{array}
\right.
\end{equation}
\end{enumerate}
\end{proposition}
\begin{proof}
    Let us define $\zeta_0=(\kappa_0,\eta_0), \zeta_1=(\eta_1,\kappa_1)$, then there exists a function $\mathcal{R}^{<<}:\R^2\times\R^2\rightarrow\R,$ such that $\mathcal{R}^<(\kappa_0,\kappa_1,\eta_0,\eta_1)=\mathcal{R}^{<<}(\zeta_0,\zeta_1)$, in fact, let us define $A=\frac{(L+1)^2}{L(L+3)}, B=\frac{(L+1)^2}{2L(L+3)}, C=\frac{(L+1)^2(L-1)}{L(L+3)^2},$ then \pierremodif{with the value of $\lambda$ defined in the proposition, we obtain} $$\mathcal{R}^{<<}(\zeta_0,\zeta_1)=1-A(\Vert \zeta_0\Vert ^2+\Vert \zeta_1\Vert ^2)+B(\Vert \zeta_0\Vert ^4+\Vert \zeta_1\Vert ^4)+C\langle \zeta_0,\zeta_1\rangle^2.$$
    To analyze its critical points, we take the partial derivatives, \begin{align*}
        \nabla_{\zeta_0}\mathcal{R}^{<<}(\zeta_0,\zeta_1)&=-2A \zeta_0+4B\Vert \zeta_0\Vert^2\zeta_0+2C\langle \zeta_0,\zeta_1\rangle \zeta_1,\\
        \nabla_{\zeta_1}\mathcal{R}^{<<}(\zeta_0,\zeta_1)&=-2A \zeta_1+4B\Vert \zeta_1\Vert^2\zeta_1+2C\langle \zeta_0,\zeta_1\rangle \zeta_0.\\
    \end{align*}
    And also, we compute its Hessian, we define \begin{align*}
        \nabla_{\zeta_0,\zeta_0}^2 \mathcal{R}^{<<}(\zeta_0,\zeta_1)&=-2AI_2+4B(2\zeta_0\zeta_0^\top+\Vert \zeta_0\Vert^2I_2)+2C\zeta_1\zeta_1^\top, \\
        \nabla_{\zeta_0,\zeta_1}^2 \mathcal{R}^{<<}(\zeta_0,\zeta_1)&=2C(\zeta_0\zeta_1^\top+\zeta_0^\top \zeta_1 I_2),  \\
        \nabla_{\zeta_1,\zeta_0}^2\mathcal{R}^{<<}(\zeta_0,\zeta_1)&=2C(\zeta_1\zeta_0^\top+\zeta_0^\top \zeta_1I_2),\\
        \nabla_{\zeta_1,\zeta_1}^2 \mathcal{R}^{<<}(\zeta_0,\zeta_1)&= -2AI_2+4B(2\zeta_1\zeta_1^\top+\Vert \zeta_1\Vert^2I_2)+2C\zeta_0\zeta_0^\top 
    \end{align*}
    Then the Hessian will be defined by
    \begin{equation}
        \nabla^2 \mathcal{R}^{<<}(\zeta_0,\zeta_1)=\begin{pmatrix}
             \nabla_{\zeta_0,\zeta_0}^2 \mathcal{R}^{<<}(\zeta_0,\zeta_1) & \nabla_{\zeta_0,\zeta_1}^2 \mathcal{R}^{<<}(\zeta_0,\zeta_1)\\
            \nabla_{\zeta_1,\zeta_0}^2 \mathcal{R}^{<<}(\zeta_1,\zeta_0) & \nabla_{\zeta_1,\zeta_1}^2 \mathcal{R}^{<<}(\zeta_0,\zeta_1)
        \end{pmatrix}
    \end{equation}

    To find the critical points, we solve the following system of equations: 
    
    \begin{equation}\label{system}
    \begin{split}
        -2A \zeta_0+4B\Vert \zeta_0\Vert^2\zeta_0+2C\langle \zeta_0,\zeta_1\rangle \zeta_1&=0,\\
        -2A \zeta_1+4B\Vert \zeta_1\Vert^2\zeta_1+2C\langle \zeta_0,\zeta_1\rangle \zeta_0&=0. 
    \end{split}
    \end{equation}

    \paragraph{$(0,0)$ is a local maximum.} We see that a trivial solution to this system is $(\zeta_0,\zeta_1)=(0,0)$, and replacing into the Hessian matrix, we see directly that this point is a local maximum.\\

    \paragraph{$(0,\zeta_1), (\zeta_0, 0)$ are strict saddle points.} We check the case when $\zeta_0=0, \zeta_1\neq 0$, then we need to solve
    \begin{align*}
        -2A\zeta_1+4B\Vert \zeta_1\Vert^2 \zeta_1=0.
    \end{align*}
    Since $\zeta_1\neq 0$, this forces \pierremodif{$-2A+4B\Vert \zeta_1\Vert^2=0$. }
    Replacing $(0,\zeta_1)$ into the Hessian matrix gives us $$\nabla^2\mathcal{R}^{<<}(0,\zeta_1)=\begin{pmatrix}
        -2AI_2+2C\zeta_1\zeta_1^\top & 0\\
        0 & \pierremodif{8B\zeta_1\zeta_1^\top}
    \end{pmatrix}.$$
    And $\mathrm{eig}(\nabla^2\mathcal{R}^{<<}(0,\zeta_1))=\mathrm{eig}(-2AI_2+2C\zeta_1\zeta_1^\top)\cup\mathrm{eig}(\pierremodif{8B\zeta_1\zeta_1^\top}),$
    where $\mathrm{eig}$ is the set of eigenvalues of a matrix. We have that \begin{align*}
        \mathrm{eig}(-2AI_2+2C\zeta_1\zeta_1^\top)&=\{-2A,2(C-A)\},\\ 
        \mathrm{eig}(\pierremodif{8B\zeta_1\zeta_1^\top})&=\{\pierremodif{0,8B\|\zeta_1\|^2}\}
    \end{align*} where we have used that $8B=4A$. We also note that $C-A<0$, then we conclude there are $2$ negative eigenvalues and $1$ positive eigenvalue, concluding that these points are strict saddle points, due to symmetry we conclude the same for the points of the form $(\zeta_0,0)$ for $\Vert \zeta_0\Vert^2=1$.

    \paragraph{Non-trivial critical points.} We will first show that the critical points that are of the form $(\zeta_0,\zeta_1)$ for $\zeta_0\neq 0,\zeta_1\neq 0$ necessarily satisfy $\Vert \zeta_0\Vert=\Vert \zeta_1\Vert\neq 0$, multiplying the first equation of \eqref{system} by $\zeta_0$ and the second by $\zeta_1$, then subtracting both resulting expressions \pierremodif{we obtain $4B(\|\zeta_1\|^4 - \|\zeta_0\|^4) = 2A(\|\zeta_1\|^2 - \|\zeta_0\|^2)$, and then}
    $$(\Vert \zeta_0\Vert^2-\Vert \zeta_1\Vert^2)(-1+\Vert \zeta_0\Vert^2+\Vert \zeta_1\Vert^2)=0,$$
    thus either $\Vert \zeta_0\Vert=\Vert \zeta_1\Vert$ and we get the first claim, or $\Vert \zeta_0\Vert^2+\Vert \zeta_1\Vert^2=1$, in this second case we divide in two subcases:
    \begin{itemize}
        \item Let us assume that $\langle \zeta_0,\zeta_1\rangle=0$, then multiplying the first equation of \eqref{system} by $\zeta_0$ and the second equation by $\zeta_1$, we get that $\Vert \zeta_0\Vert^2=\Vert \zeta_1\Vert^2=1$, which is a contradiction since we are in the case where $\Vert \zeta_0\Vert^2+\Vert \zeta_1\Vert^2=1$.
        \item Therefore $\langle \zeta_0,\zeta_1\rangle\neq 0$, we multiply the first equation of \eqref{system} by $\zeta_1$ and second by $\zeta_0$, after dividing by $2\langle \zeta_0,\zeta_1\rangle$ we get that \begin{align*}
            -A+A\Vert \zeta_0\Vert^2+C\langle \zeta_0,\zeta_1\rangle&=0\\
            -A+A\Vert \zeta_1\Vert^2+C\langle \zeta_0,\zeta_1\rangle&=0.
        \end{align*}
        Substracting both equations we get that $\Vert \zeta_0\Vert=\Vert \zeta_1\Vert$.
        
    \end{itemize}
    So we get that necessarily $\Vert \zeta_0\Vert=\Vert \zeta_1\Vert=r>0$. Then the equation \eqref{system} becomes \begin{equation}\label{system2}
        \begin{split}
            A(r^2-1)\zeta_0+C\langle \zeta_0,\zeta_1\rangle \zeta_1&=0,\\
            A(r^2-1)\zeta_1+C\langle \zeta_0,\zeta_1\rangle \zeta_0&=0.
        \end{split}
    \end{equation}
     \paragraph{$(\zeta_0,\pm \zeta_0)$ are strict saddle points.} In the case where $\zeta_0=\pm \zeta_1$, by \eqref{system2} we get that when $\zeta_0=\zeta_1$, then $$A(r_0^2-1)\zeta_0+ Cr_0^2\zeta_0=0,$$ and $\Vert\zeta_0\Vert=r_0$, with $r_0^2=\frac{A}{A+C}=\frac{L+3}{2(L+1)}$. Besides when $\zeta_0=-\zeta_1$, then $$A(r_1^2-1)\zeta_0+ Cr_1^2\zeta_0=0,$$ and $\Vert\zeta_0\Vert=r_1$, with $r_1^2=\frac{A}{A-C}=\frac{L+3}{4}$. Replacing this point on the Hessian matrix $\nabla^2\mathcal{R}^{<<}(\zeta_0,\pm \zeta_0)$,
     \begin{align*}
     \nabla^2\mathcal{R}^{<<}(\zeta_0, \zeta_0)&=\begin{pmatrix}
         2A(r_0^2-1)I_2+2(2A+C)\zeta_0\zeta_0^\top & 2C(r_0^2I_2+\zeta_0\zeta_0^\top)\\
         2C(r_0^2I_2+\zeta_0\zeta_0^\top) & 2A(r_0^2-1)I_2+2(2A+C)\zeta_0\zeta_0^\top
     \end{pmatrix},\\
     \nabla^2\mathcal{R}^{<<}(\zeta_0, -\zeta_0)&=\begin{pmatrix}
         2A(r_1^2-1)I_2+2(2A+C)\zeta_0\zeta_0^\top & -2C(r_1^2I_2+\zeta_0\zeta_0^\top)\\
         -2C(r_1^2I_2+\zeta_0\zeta_0^\top) & 2A(r_1^2-1)I_2+2(2A+C)\zeta_0\zeta_0^\top
     \end{pmatrix}.
     \end{align*}
    We can write the Hessians in block form with $2\times 2$ diagonal blocks $$M_0 = 2A(r^2-1)I_2 + 2(2A+C)\,\zeta_0\zeta_0^\top$$ and off-diagonal blocks $$M_1 = 2C(r^2 I_2 + \zeta_0\zeta_0^\top).$$
Considering vectors of the form $(x,\pm x), x\in\R^2$, this reduces the eigenvalue problem to the $2\times 2$ matrices $M_0 \pm M_1$. Substituting $r_0^2 = A/(A+C)$ or $r_1^2 = A/(A-C)$ shows that each matrix has one positive eigenvalue along $\zeta_0$ (equal to $4A$) and one negative eigenvalue orthogonal to $\zeta_0$ (equal to $-\frac{4AC}{A\pm C}$), so both Hessians are indefinite.
     
     \paragraph{Characterization of global minima.} If $\zeta_0\neq \pm \zeta_1$ and $\Vert \zeta_0\Vert=\Vert \zeta_1\Vert=r>0$, then both vectors are linearly independent, thus the first equation of \eqref{system2} is only possible when $r^2=1$ and $\langle \zeta_0,\zeta_1\rangle=0$, in which case we have to analyze the points $(\zeta_0,\zeta_1)$ such that $\langle \zeta_0,\zeta_1\rangle=0$ and $\Vert \zeta_0\Vert^2=\Vert \zeta_1\Vert^2=1$, we replace these points on the Hessian matrix and this gives us $$\nabla^2\mathcal{R}^{<<}(\zeta_0,\zeta_1)=\begin{pmatrix}
         4A\zeta_0\zeta_0^\top+2C\zeta_1\zeta_1^\top & 2C\zeta_0\zeta_1^\top\\ 2C\zeta_1\zeta_0^\top & 4A\zeta_1\zeta_1^\top+2C\zeta_0\zeta_0^\top
     \end{pmatrix}.$$
     A direct computation of the eigenvalues \pierremodif{with eigenvectors $(\zeta_1, \zeta_0)$ and $(\zeta_1, -\zeta_0)$} gives us that all the eigenvalues are positive in this case, since $\mathcal{R}^{<<}$ is coercive, these points are in fact global minima.
\end{proof}

\begin{remark}\label{lambdachoice}
    Note that the specific characterization of the global minima of $\mathcal{R}^<$ was valid only for $\lambda = \lambda_0^\star = \frac{L+1}{L+3}$. However, when restricting the analysis to the manifold $\tilde{\mathcal{M}}$ and considering $\lambda \in ]0, \frac{L+1}{L+3}[$, the global minima lie outside the domain $[-1,1]^2$. As a result, due to the structure of the update rule in \eqref{eq:PGD_iterates}, the extreme points ${(\pm1, \pm1)}$ of $[-1,1]^2$ become fixed points of the algorithm and serve as global minimizers of $\mathcal{R}^<$.
\end{remark} 

\subsection{Discussion on regularization}\label{discreg}
In order to solve the clustering problem in the degenerate case, we train the attention-based predictor $H^{{\rm lin}, \mu_0, \mu_1}$ now by minimizing the regularized risk
\begin{equation}\label{penalizationappendix}\tag{$\tilde{\mathcal{P}}_{\rho
}$}
    \min_{\mu_0,\mu_1\in\mathbb{S}^{d-1}} \mathcal{R}^{\rho}(\mu_0,\mu_1) \qquad \text{with} \qquad \mathcal{R}^{\rho}(\mu_0,\mu_1) \eqdef\mathcal{R}(\mu_0,\mu_1)+\rho r(\mu_0,\mu_1),
\end{equation}
where $r(\mu_0,\mu_1)=\mathbb{E}[\langle \mu_0,X_1\rangle^2\langle \mu_1,X_1\rangle^2]$, and  $\rho>0$ denotes the strength of the regularization. 

It is direct to check that there exists $r^<:\R^4\rightarrow\R$, such that $r(\mu_0,\mu_1)=r^<(\kappa_0,\kappa_1,\eta_0,\eta_1)$ according to the notation defined in \eqref{notation0}, and $r^<(\kappa_0,\kappa_1,\eta_0,\eta_1)=\frac{1}{2}(\kappa_0^2\eta_0^2+\kappa_1^2\eta_1^2)$. We define the following optimization problem \begin{equation}\label{penalization1}\tag{$\tilde{\mathcal{P}}_{\rho}^<$}
    \min_{\kappa_0,\kappa_1,\eta_0,\eta_1\in [-1,1]} \mathcal{R}^<(\kappa_0,\kappa_1,\eta_0,\eta_1)+\rho r^<(\kappa_0,\kappa_1,\eta_0,\eta_1),
\end{equation}
where $\mathcal{R}^<$ is defined in Proposition \ref{risknoiseless}. Since $\mathcal{R}^<$ and $r^<$ are coercive, we apply \citet[Theorem 2.1]{attouch96} to conclude that if $u_{\rho}\in[-1,1]^4$ is a solution of \eqref{penalization1}, then every limit point $\hat{u}$ of $u_{\rho}$, when $\rho\rightarrow0$, satisfies that: 
\begin{equation*}
\begin{cases}
    r^<(\hat{u})\leq r^<(v),\quad \text{for every $v\in \argmin \mathcal{R}^<$},\\
    \hat{u}\in\argmin \mathcal{R}^<.
\end{cases}
\end{equation*}
Due to the geometry of $r^<$ and the characterization of $\argmin \mathcal{R}^<$ we got in Proposition \ref{eqminima}, we obtain that if $\hat{u}=(\hat{\kappa_0},\hat{\kappa_1},\hat{\eta_0},\hat{\eta_1})$, then 
\begin{equation} \label{eq:discussion-reg}
    \begin{cases}
        \hat{\kappa}_0^2=1, \hat{\kappa}_1^2=1, \hat{\eta}_0^2=0, \hat{\eta}_1^2=0, \quad\text{or}\\
        \hat{\eta}_0^2=1, \hat{\eta}_1^2=1, \hat{\kappa}_0^2=0, \hat{\kappa}_1^2=0.
    \end{cases}
\end{equation}
 Then the optimal solution for the regularized problem when $\rho\to 0$ achieves a saturation effect, corresponding to global minimizers that recover the centroids. However, due to the non-convex nature of the problem, it is not guaranteed \pierremodif{a priori that PGD on} \eqref{penalization1} will converge to the desired solution. 
 \pierremodif{A possible direction of analysis is to study the dynamics of PGD in the limit where $\rho \to 0$. We know from Proposition \ref{critical} that the only global minimizers of the unregularized problem lie on a manifold, so we expect that PGD converges to this manifold, before evolving on the manifold due to the regularization term, to converge to the minimizers given by \eqref{eq:discussion-reg}. Technically, this dynamics could be studied by using two-timescale tools, e.g.~similar to \citet[][]{marion2023} and references therein. We leave this analysis for future work.}

\section{Proofs of Section \ref{sec:gmm} (Gaussian mixture model)}
\label{sec:app_gmm}

\subsection{Proof of Proposition \ref{prop:structure_of_R1}
(expression of the risk in the non-degenerate case).} 
\label{app:risk_gmm}

Recall the notation  $e_k(\mu) \eqdef \lambda X_1^\top\mu\mu^\top X_k$, for $1\leq k \leq L$, which allows us to write
$$
\mathcal{R}(\mu_0,\mu_1)=\mathbb{E}\left[\Big\Vert X_1-\frac{2}{L}\sum_{k=1}^L (e_k(\mu_0)+e_k(\mu_1))X_k\Big\Vert_2^2\right].
$$
Under the Gaussian mixture model, we are going to show that the risk $$\mathcal{R}(\mu_0,\mu_1)=\mathbb{E}\left[\Vert X_1-(H^{\mu_0}+H^{\mu_1})(\mathbb{X})_1\Vert_2^2\right]$$ admits a closed-form representation in terms of elementary functions. It holds that 

\allowdisplaybreaks
\begin{align*}
    \mathcal{R}(\mu_0,\mu_1)&=\mathbb{E}\left[\Big\Vert X_1-\frac{2}{L}\sum_{k=1}^L (e_k(\mu_0)+e_k(\mu_1))X_k\Big\Vert_2^2\right]\\
    &=\mathbb{E}\left[\Vert X_1\Vert^2\right]-\frac{4}{L}\sum_{k=1}^L\mathbb{E}\big[\langle X_1, (e_k(\mu_0)+e_k(\mu_1))X_k\rangle\big]\\
    &\quad+\frac{4}{L^2}\mathbb{E}\left[\Big\Vert \sum_{k=1}^L (e_k(\mu_0)+e_k(\mu_1))X_k\Big\Vert^2\right].\\
     &=(1+d\sigma^2)-\frac{4}{L}\sum_{k=1}^L\mathbb{E}\big[\langle X_1, (e_k(\mu_0)+e_k(\mu_1))X_k\rangle\big]\\
     &\quad+\frac{4}{L^2} \sum_{k=1}^L \mathbb{E}[\Vert(e_k(\mu_0)+e_k(\mu_1))X_k\Vert^2]\\
     &\quad+\frac{8}{L^2}\sum_{1\leq k<j\leq L }\mathbb{E}\big[  (e_k(\mu_0)+e_k(\mu_1))(e_j(\mu_0)+e_j(\mu_1))\langle X_k, X_j\rangle\big]\\
    &=(1+d\sigma^2)-\underbrace{\frac{4}{L}\mathbb{E}[(e_1(\mu_0)+e_1(\mu_1))\Vert X_1\Vert^2]}_{\eqdef(I_0)}+\underbrace{\frac{4}{L^2}\mathbb{E}\left[\Vert  (e_1(\mu_0)+e_1(\mu_1))X_1\Vert^2\right]}_{\eqdef(II_0)} \\
    &+\underbrace{\frac{8}{L^2}\sum_{k=2}^L\mathbb{E}\big[  (e_1(\mu_0)+e_1(\mu_1))(e_k(\mu_0)+e_k(\mu_1))\langle X_1, X_k\rangle\big]}_{\eqdef (III_0)}\\
    &-\underbrace{\frac{4}{L}\sum_{k=2}^L\mathbb{E}\big[\langle X_1, (e_k(\mu_0)+e_k(\mu_1))X_k\rangle\big]}_{\eqdef(I)}+\underbrace{\frac{4}{L^2}\sum_{k=2}^L\mathbb{E}\left[\Vert  (e_k(\mu_0)+e_k(\mu_1))X_k\Vert^2\right]}_{\eqdef (II)}\\
    &+\underbrace{\frac{8}{L^2}\sum_{1< k<j\leq L}\mathbb{E}\big[  (e_k(\mu_0)+e_k(\mu_1))(e_j(\mu_0)+e_j(\mu_1))\langle X_k, X_j\rangle\big]}_{\eqdef (III)}\\
    &=(1+d\sigma^2)-(I_0)+(II_0)+(III_0)- (I)+(II)+(III).
\end{align*}

We now proceed to compute each of the six terms. To compute $(I_0)$, we can use Lemma \ref{I0}, since \begin{align*}
     &\mathbb{E}[(\langle X_1,\mu_0\rangle^2+\langle X_1,\mu_1\rangle^2)\Vert X_1\Vert^2]\\
     &=\frac{1}{2}\left(\mathbb{E}[\langle X_1,\mu_0\rangle^2\Vert X_1\Vert^2|Z_1=0]+\mathbb{E}[(\langle X_1,\mu_0\rangle^2\Vert X_1\Vert^2|Z_1=1]\right)\\
     &+\frac{1}{2}\left(\mathbb{E}[\langle X_1,\mu_1\rangle^2\Vert X_1\Vert^2|Z_1=0]+\mathbb{E}[\langle X_1,\mu_1\rangle^2\Vert X_1\Vert^2|Z_1=1]\right)\\
     &=\frac{1}{2}\left[(\kappa_0^2+\eta_0^2+\kappa_1^2+\eta_1^2)(1+\sigma^2(d+4))+2\sigma^2(1+\sigma^2(d+2))(\Vert\mu_0\Vert^2+\Vert\mu_1\Vert^2)\right]   
 \end{align*}
 Then, $(I_0)=\frac{2\lambda}{L}\left[(\kappa_0^2+\eta_0^2+\kappa_1^2+\eta_1^2)(1+\sigma^2(d+4))+2\sigma^2(1+\sigma^2(d+2))(\Vert\mu_0\Vert^2+\Vert\mu_1\Vert^2)\right]$.\\

 To compute $(II_0)$, by defining $p_0(\mu_0,\mu_1,\mu^\star)$ as in Lemma \ref{II0}, we get   
\begin{align*}
     &\mathbb{E}[(\langle X_1,\mu_0\rangle^2+\langle X_1,\mu_1\rangle^2)^2\Vert X_1\Vert^2]\\
     &=\frac{1}{2}\mathbb{E}[(\langle X_1,\mu_0\rangle^4+2\langle X_1,\mu_0\rangle^2\langle X_1,\mu_1\rangle^2+\langle X_1,\mu_1\rangle^4)\Vert X_1\Vert^2|Z_1=0]\\
     &+\frac{1}{2}\mathbb{E}[(\langle X_1,\mu_0\rangle^4+2\langle X_1,\mu_0\rangle^2\langle X_1,\mu_1\rangle^2+\langle X_1,\mu_1\rangle^4)\Vert X_1\Vert^2|Z_1=1]\\ 
     &=\frac{1}{2}(p_0(\mu_0,\mu_0,\mu_0^\star)+2p_0(\mu_0,\mu_1,\mu_0^\star)+p_0(\mu_1,\mu_1,\mu_0^\star))\\
     &+\frac{1}{2}(p_0(\mu_0,\mu_0,\mu_1^\star)+2p_0(\mu_0,\mu_1,\mu_1^\star)+p_0(\mu_1,\mu_1,\mu_1^\star)).
 \end{align*}
 Then,  
 \begin{align*}(II_0)&=\frac{4\lambda^2}{L^2}\mathbb{E}[(\langle X_1,\mu_0\rangle^2+\langle X_1,\mu_1\rangle^2)^2\Vert X_1\Vert^2]\\
 &=\frac{2\lambda^2}{L^2}\left(p_0(\mu_0,\mu_0,\mu_0^\star +2p_0(\mu_0,\mu_1,\mu_0^\star)+p_0(\mu_1,\mu_1,\mu_0^\star)\right)\\
     &+\frac{2\lambda^2}{L^2}\left(p_0(\mu_0,\mu_0,\mu_1^\star)+2p_0(\mu_0,\mu_1,\mu_1^\star)+p_0(\mu_1,\mu_1,\mu_1^\star)\right).\end{align*}

 To compute $(III_0)$, by defining $p_1(\mu_0,\mu_1,\mu_{Z_1}^\star,\mu_{Z_2}^\star)$ as in Lemma \ref{III0}, we get
 \begin{align*}
     &\mathbb{E}[(\langle X_1,\mu_0\rangle^2+\langle X_1,\mu_1\rangle^2)(\langle X_1,\mu_0\rangle\langle X_2,\mu_0\rangle+\langle X_1,\mu_1\rangle\langle X_2,\mu_1\rangle)\langle X_1,X_2\rangle]\\
     &=\frac{1}{4}\sum_{(z_1,z_2)\in \{0,1\}^2}\Upsilon_1(z_1,z_2),
     \end{align*}
     where \begin{align*}
         \Upsilon_1(z_1,z_2)=\mathbb{E}[&(\langle X_1,\mu_0\rangle^2+\langle X_1,\mu_1\rangle^2)(\langle X_1,\mu_0\rangle\langle X_2,\mu_0\rangle+\langle X_1,\mu_1\rangle\langle X_2,\mu_1\rangle)\\
         &\cdot\langle X_1,X_2\rangle|Z_1=z_1,Z_2=z_2].
     \end{align*}
     And then
     \begin{align*}
     &\frac{1}{4}\sum_{(z_1,z_2)\in \{0,1\}^2}\Upsilon_1(z_1,z_2)\\
     &=\frac{1}{4}(p_1(\mu_0,\mu_0,\mu_0^\star,\mu_0^\star)+p_1(\mu_0,\mu_1,\mu_0^\star,\mu_0^\star)+p_1(\mu_1,\mu_0,\mu_0^\star,\mu_0^\star)+p_1(\mu_1,\mu_1,\mu_0^\star,\mu_0^\star))\\
     &+\frac{1}{4}(p_1(\mu_0,\mu_0,\mu_1^\star,\mu_0^\star)+p_1(\mu_0,\mu_1,\mu_1^\star,\mu_0^\star)+p_1(\mu_1,\mu_0,\mu_1^\star,\mu_0^\star)+p_1(\mu_1,\mu_1,\mu_1^\star,\mu_0^\star))\\
     &+\frac{1}{4}(p_1(\mu_0,\mu_0,\mu_0^\star,\mu_1^\star)+p_1(\mu_0,\mu_1,\mu_0^\star,\mu_1^\star)+p_1(\mu_1,\mu_0,\mu_0^\star,\mu_1^\star)+p_1(\mu_1,\mu_1,\mu_0^\star,\mu_1^\star))\\
     &+\frac{1}{4}(p_1(\mu_0,\mu_0,\mu_1^\star,\mu_1^\star)+p_1(\mu_0,\mu_1,\mu_1^\star,\mu_1^\star)+p_1(\mu_1,\mu_0,\mu_1^\star,\mu_1^\star)+p_1(\mu_1,\mu_1,\mu_1^\star,\mu_1^\star)).
 \end{align*}
 Consequently, \begin{align*}
 (III_0)&=8\lambda^2\frac{(L-1)}{L^2}\mathbb{E}[(\langle X_1,\mu_0\rangle^2+\langle X_1,\mu_1\rangle^2)(\langle X_1,\mu_0\rangle\langle X_2,\mu_0\rangle+\langle X_1,\mu_1\rangle\langle X_2,\mu_1\rangle)\langle X_1,X_2\rangle]\\
 &=2\lambda^2\frac{(L-1)}{L^2}\sum_{(z_1,z_2)\in \{0,1\}^2}\Upsilon_1(z_1,z_2).
 \end{align*}

 To compute $(I)$, we can use Lemma \ref{I} to obtain: 
 \begin{align*}
        &\mathbb{E}[(\langle X_1,\mu_0\rangle\langle X_2,\mu_0\rangle+\langle X_1,\mu_1\rangle\langle X_2,\mu_1\rangle)\langle X_1, X_2\rangle]\\
        &=\frac{1}{4}\mathbb{E}[(\langle X_1,\mu_0\rangle\langle X_2,\mu_0\rangle+\langle X_1,\mu_1\rangle\langle X_2,\mu_1\rangle)\langle X_1, X_2\rangle|Z_1=0,Z_2=0]\\
        &+\frac{1}{4}\mathbb{E}[(\langle X_1,\mu_0\rangle\langle X_2,\mu_0\rangle+\langle X_1,\mu_1\rangle\langle X_2,\mu_1\rangle)\langle X_1, X_2\rangle|Z_1=0,Z_2=1]\\
        &+\frac{1}{4}\mathbb{E}[(\langle X_1,\mu_0\rangle\langle X_2,\mu_0\rangle+\langle X_1,\mu_1\rangle\langle X_2,\mu_1\rangle)\langle X_1, X_2\rangle|Z_1=1,Z_2=0]\\
        &+\frac{1}{4}\mathbb{E}[(\langle X_1,\mu_0\rangle\langle X_2,\mu_0\rangle+\langle X_1,\mu_1\rangle\langle X_2,\mu_1\rangle)\langle X_1, X_2\rangle|Z_1=1,Z_2=1]\\
        &=\frac{1+4\sigma^2}{4}(\kappa_0^2+\eta_0^2+\kappa_1^2+\eta_1^2)+\sigma^4(\Vert\mu_0\Vert^2+\Vert\mu_1\Vert^2).
    \end{align*}
Thus, $(I)=\lambda\frac{L-1}{L}[(\kappa_0^2+\eta_0^2+\kappa_1^2+\eta_1^2)(1+4\sigma^2)+4\sigma^4(\Vert\mu_0\Vert^2+\Vert\mu_1\Vert^2)]$ .\\

To compute $(II)$, by defining $p_2(\mu_0,\mu_1,\mu_{Z_1}^\star,\mu_{Z_2}^\star)$ as in Lemma \ref{II}, we obtain: \begin{align*}
 &\mathbb{E}[(\langle X_1,\mu_0\rangle\langle X_2,\mu_0\rangle+\langle X_1,\mu_1\rangle\langle X_2,\mu_1\rangle)^2\Vert X_2\Vert^2]\\
 &=\frac{1}{4}\sum_{(z_1,z_2)\in\{0,1\}^2}\Upsilon_2(z_1,z_2),
 \end{align*}
 where \begin{align*}
     \Upsilon_2(z_1,z_2)=\mathbb{E}[(\langle X_1,\mu_0\rangle\langle X_2,\mu_0\rangle+\langle X_1,\mu_1\rangle\langle X_2,\mu_1\rangle)^2\Vert X_2\Vert^2|Z_1=z_1,Z_2=z_2].
 \end{align*}
 And then
 \begin{align*}
 &\frac{1}{4}\sum_{(z_1,z_2)\in\{0,1\}^2}\Upsilon_2(z_1,z_2)
 \\
 &=\frac{1}{4}(p_2(\mu_0,\mu_0,\mu_0^\star,\mu_0^\star)+2p_2(\mu_0,\mu_1,\mu_0^\star,\mu_0^\star)+p_2(\mu_1,\mu_1,\mu_0^\star,\mu_0^\star))\\
 &+\frac{1}{4}(p_2(\mu_0,\mu_0,\mu_0^\star,\mu_1^\star)+2p_2(\mu_0,\mu_1,\mu_0^\star,\mu_1^\star)+p_2(\mu_1,\mu_1,\mu_0^\star,\mu_1^\star))\\
 &+\frac{1}{4}(p_2(\mu_0,\mu_0,\mu_1^\star,\mu_0^\star)+2p_2(\mu_0,\mu_1,\mu_1^\star,\mu_0^\star)+p_2(\mu_1,\mu_1,\mu_1^\star,\mu_0^\star))\\
 &+\frac{1}{4}(p_2(\mu_0,\mu_0,\mu_1^\star,\mu_1^\star)+2p_2(\mu_0,\mu_1,\mu_1^\star,\mu_1^\star)+p_2(\mu_1,\mu_1,\mu_1^\star,\mu_1^\star)).
\end{align*}
So we obtain, \begin{align*}
    (II)&=\frac{4\lambda^2(L-1)}{L^2}\mathbb{E}[(\langle X_1,\mu_0\rangle\langle X_2,\mu_0\rangle+\langle X_1,\mu_1\rangle\langle X_2,\mu_1\rangle)^2\Vert X_2\Vert^2]\\
    &=\frac{\lambda^2(L-1)}{L^2}\sum_{(z_1,z_2)\in\{0,1\}^2}\Upsilon_2(z_1,z_2).
\end{align*}

Finally, to compute $(III)$, by defining $p_3(\mu_0,\mu_1,\mu_{Z_1}^\star,\mu_{Z_2}^\star,\mu_{Z_3}^\star)$ as in Lemma \ref{III}, we get  \begin{align*}
    &\mathbb{E}[(\langle X_1,\mu_0\rangle\langle X_2,\mu_0\rangle+\langle X_1,\mu_1\rangle\langle X_2,\mu_1\rangle)(\langle X_1,\mu_0\rangle\langle X_3,\mu_0\rangle+\langle X_1,\mu_1\rangle\langle X_3,\mu_1\rangle)\langle X_2,X_3\rangle]\\
    &=\frac{1}{8}\sum_{(z_1,z_2,z_3)\in \{0,1\}^3}\Upsilon_3(z_1,z_2,z_3),\\
    \end{align*}
    where 
    \begin{align*}
\Upsilon_3(z_1,z_2,z_3) 
= \mathbb{E}\big[(&\langle X_1,\mu_0\rangle\langle X_2,\mu_0\rangle 
+ \langle X_1,\mu_1\rangle\langle X_2,\mu_1\rangle) \\
&\cdot (\langle X_1,\mu_0\rangle\langle X_3,\mu_0\rangle 
+ \langle X_1,\mu_1\rangle\langle X_3,\mu_1\rangle) \\
&\cdot \langle X_2,X_3\rangle 
\;\big|\, Z_1 = z_1, Z_2 = z_2, Z_3 = z_3 \big].
\end{align*}
\pierremodif{
Observing that 
    \begin{align*}
    &\sum_{(z_1,z_2,z_3)\in \{0,1\}^3}\Upsilon_3(z_1,z_2,z_3) = \sum_{(a, b, c, d, e) \in \{0, 1\}^5} p_3(\mu_a,\mu_b,\mu_c^\star,\mu_d^\star,\mu_e^\star) ,
\end{align*}
we get 
\begin{align*}
(III) & = \frac{8\lambda^2(L-1)(L-2)}{2L^2} \cdot \frac{1}{8}\sum_{(z_1,z_2,z_3)\in \{0,1\}^3}\Upsilon_3(z_1,z_2,z_3) \\
    & = \frac{\lambda^2(L-1)(L-2)}{2L^2} \sum_{(a, b, c, d, e) \in \{0, 1\}^5} p_3(\mu_a,\mu_b,\mu_c^\star,\mu_d^\star,\mu_e^\star).
\end{align*}
}

\pierremodif{Finally, putting everything together, recalling the notation introduced in \eqref{notation}, and inspecting the formulas given by Lemmas from \ref{I0} to \ref{III} allows to conclude.}

\subsection{Proof of Lemma \ref{lem:riskmanifoldshort_gmm} (expression of the risk on the manifold, non-degenerate case)} We provide in Lemma \ref{lem:riskmanifold_gmm} a more precise version of Lemma \ref{lem:riskmanifoldshort_gmm}, with explicit constants.
\begin{lemma}\label{lem:riskmanifold_gmm}
    Define $c_1(n)=1+n\sigma^2$ and $c_2(n)=1+\sigma^2(d+n)$, then the risk $\mathcal{R}^<(\kappa_0,\kappa_1)$ restricted to $\mathcal{M}$ has the form \begin{align*} \mathcal{R}^<(\kappa_0,\kappa_1)&=A(\kappa_0^4+\kappa_1^4)+B(\kappa_0^2+\kappa_1^2)+C\kappa_0^2\kappa_1^2+D,
 \end{align*}
 where  \begin{align*}
     A&=\frac{2\lambda^2}{L^2}c_2(8)+\frac{2\lambda^2(L-1)}{L^2}c_1(5)+\frac{\lambda^2 (L-1)}{L^2}c_2(4)+\frac{\lambda^2(L-1)(L-2)}{2L^2}c_1(4).\\
     B&=-\frac{2\lambda}{L}c_2(4)+\frac{16\lambda^2\sigma^2}{L^2}c_2(6)+\frac{8\lambda^2\sigma^2(L-1)}{L^2}c_1(6)-\frac{\lambda (L-1)}{L}c_1(4)\\
     &+\frac{4\lambda^2\sigma^2(L-1)}{L^2}c_2(3)+\frac{\lambda^2\sigma^2(L-1)(L-2)}{L^2}c_1(6).\\
     C&=\frac{4\lambda^2\sigma^2(L-1)}{L^2}.\\
     D&=c_1(d)-\frac{8\lambda\sigma^2}{L}c_2(2)+\frac{32\lambda^2\sigma^4}{L^2}c_2(4)+\frac{64\lambda^2\sigma^6(L-1)}{L^2}\\
     &-\frac{8\lambda\sigma^4 (L-1)}{L}+\frac{8\lambda^2\sigma^4(L-1)}{L^2}c_2(2)+\frac{8\lambda^2\sigma^6(L-1)(L-2)}{L^2}.
 \end{align*}
\end{lemma}

\begin{proof}[Proof of Lemma \ref{lem:riskmanifold_gmm}]
Using the decomposition obtained in the proof of Proposition \ref{prop:structure_of_R1}, after simple algebraic manipulation we get that on this manifold: 
\begin{itemize}
    \item $(I_0)=\frac{2\lambda}{L}[(\kappa_0^2+\kappa_1^2)(1+\sigma^2(d+4))+4\sigma^2(1+\sigma^2(d+2))]$.
    \item $(II_0)=\frac{2\lambda^2}{L^2}[(\kappa_0^4+\kappa_1^4)(1+\sigma^2(d+8))+8\sigma^2(\kappa_0^2+\kappa_1^2)(1+\sigma^2(d+6))+16\sigma^4(1+\sigma^2(d+4))]$.
    \item $(III_0)=2\lambda^2 \frac{(L-1)}{L^2}[(\kappa_0^4+\kappa_1^4)(1+5\sigma^2)+4\sigma^2(\kappa_0^2+\kappa_1^2)(1+6\sigma^2)+2\sigma^2\kappa_0^2\kappa_1^2+32\sigma^6].$
    \item $(I)=\lambda\frac{L-1}{L}[(\kappa_0^2+\kappa_1^2)(1+4\sigma^2)+8\sigma^4]$.
    \item $(II)=\lambda^2\frac{(L-1)}{L^2}[(\kappa_0^4+\kappa_1^4)(1+\sigma^2(d+4))+4\sigma^2(\kappa_0^2+\kappa_1^2)(1+\sigma^2(d+3))+8\sigma^4(1+\sigma^2(d+2))]$.
    \item $(III)=\lambda^2\frac{(L-1)(L-2)}{2L^2}[(\kappa_0^4+\kappa_1^4)(1+4\sigma^2)+2\sigma^2(\kappa_0^2+\kappa_1^2)(1+6\sigma^2)+16\sigma^6]$.
\end{itemize} 
We conclude by noting that the risk restricted to this manifold is $$\mathcal{R}^<(\kappa_0,\kappa_1)=(1+d\sigma^2)-(I_0)+(II_0)+(III_0)- (I)+(II)+(III),$$

and properly factorizing the terms.   
\end{proof}

\subsection{Proof of Proposition \ref{prop:optimallambdashort_gmm} (global minima of the risk, non-degenerate case).} In what follows we provide an extended version of Proposition \ref{prop:optimallambdashort_gmm} with explicit constant, together with its proof.

\begin{proposition}\label{optimallambda}
    Let us define \begin{align*}
        c_3(\sigma,L)&\eqdef 16\sigma^2c_2(6)+8\sigma^2(L-1)c_1(6)+4\sigma^2(L-1)c_2(3)+\sigma^2(L-1)(L-2)c_1(6)+4c_2(8)\\
    &+4(L-1)c_1(5)+2(L-1)c_2(4)+(L-1)(L-2)c_1(4)+4\sigma^2(L-1),
    \end{align*} and consider $\mathcal{R}^<(\kappa_0,\kappa_1)$ with the following $\lambda$: \begin{align*}
        \lambda^\star(\sigma,L)=\frac{2Lc_2(4)+L(L-1)c_1(4)}{c_3(\sigma,L)}.
    \end{align*}
Then the points $(\pm 1,\pm 1)$ are global minimum of $\mathcal{R}^<(\kappa_0,\kappa_1)$.    
\end{proposition}
\begin{remark}
    In the case where $\sigma=0$, we get $\lambda^\star(0,L)=\lambda_0^\star=\frac{L+1}{L+3}$.
\end{remark}
\begin{proof}
    Imposing first order conditions on $\mathcal{R}^<(\kappa_0,\kappa_1)$ from Lemma \ref{lem:riskmanifold_gmm}, we obtain an explicit form of its critical points. \pierremodif{From this expression,} we note that \pierremodif{the global minimum are the points $(\pm 1,\pm 1)$ if and only if $4A+2B+2C = 0$. The function $\lambda \mapsto 2A(\lambda)+B(\lambda)+C(\lambda)$ is a quadratic which is negative for $0 \leq \lambda < \lambda^\star(\sigma,L)$}, and vanishes at $\lambda=\lambda^\star(\sigma,L)$.
\end{proof}


\subsection{Proof of Theorem \texorpdfstring{\ref{thm:main_gmm}}{}}\label{outlinemainnoise}

The proof of this result is built upon a series of intermediate results that progressively lead to the desired conclusion.
\begin{lemma}\label{der0f}
    At a point $(\kappa_0,\kappa_1,\eta_0,\eta_1,\xi)$ such that $\eta_0=\eta_1=\xi=0$, we have $\partial_{\eta_0}\mathcal{R}^<=\partial_{\eta_1}\mathcal{R}^<=\partial_{\xi}\mathcal{R}^<=0.$
\end{lemma}
\begin{proof}
    According to Proposition \ref{prop:structure_of_R1}, we can directly obtain that $$\mathcal{R}^<(\kappa_0,\kappa_1,\eta_0,\eta_1,\xi)=\mathcal{R}^<(\kappa_0,\kappa_1,-\eta_0,-\eta_1,-\xi).$$
    Taking the partial derivative in $\eta_0$, we get $$\partial_{\eta_0}\mathcal{R}^<(\kappa_0,\kappa_1,\eta_0,\eta_1,\xi)=-\partial_{\eta_0}\mathcal{R}^<(\kappa_0,\kappa_1,-\eta_0,-\eta_1,-\xi).$$
    At a point such that $\eta_0=\eta_1=\xi=0$, this gives $\partial_{\eta_0}\mathcal{R}^<(\kappa_0,\kappa_1,0,0,0)=-\partial_{\eta_0}\mathcal{R}^<(\kappa_0,\kappa_1,0,0,0)$, therefore $\partial_{\eta_0}\mathcal{R}^<(\kappa_0,\kappa_1,0,0,0)=0$, the proof for $\partial_{\eta_1}\mathcal{R}^<, \partial_\xi\mathcal{R}^<$ is analogous.
\end{proof}

\begin{lemma}
    The manifold $\mathcal{M}$ is invariant under \eqref{eq:PGD_iterates} dynamics, this is if $(\mu_0^k,\mu_1^k)\in \mathcal{M}$, then $(\mu_0^{k+1},\mu_1^{k+1})\in \mathcal{M}$.
\end{lemma}
\begin{proof}
 We apply the chain rule and Lemma \ref{der0f} to get: \begin{equation}\label{chainrule0}
        \nabla_{\mu_0}\mathcal{R}=\partial_{\kappa_0}\mathcal{R}^< \mu_0^\star+\partial_{\eta_1}\mathcal{R}^< \mu_1^\star+\partial_{\xi}\mathcal{R}^< \mu_1=\partial_{\kappa_0}\mathcal{R}^< \mu_0^\star,
    \end{equation}
        \begin{equation}\label{chainrule1}
        \nabla_{\mu_1}\mathcal{R}=\partial_{\kappa_1}\mathcal{R}^< \mu_1^\star+\partial_{\eta_0}\mathcal{R}^< \mu_0^\star+\partial_{\xi}\mathcal{R}^< \mu_0=\partial_{\kappa_1}\mathcal{R}^< \mu_1^\star.
    \end{equation}
    We then follow the same ideas as in \citet[Lemma 4]{marion2024attention}, where our Lemma \ref{der0f} takes the role of \citet[Lemma 14]{marion2024attention}.
    
    More concretely, let us consider $c_0=\Vert \mu_0^k-\gamma(I_d-\mu_0^k(\mu_0^k)^\top)\nabla_{\mu_0}\mathcal{R}(\mu_0^k,\mu_1^k)\Vert_2,$ and \linebreak
    $c_1=\Vert \mu_1^k-\gamma(I_d-\mu_1^k(\mu_1^k)^\top)\nabla_{\mu_1}\mathcal{R}(\mu_0^k,\mu_1^k)\Vert_2$, then recalling \eqref{eq:PGD_iterates} updates, we have that if $(\mu_0^k,\mu_1^k)\in \mathcal{M}$, then \begin{align*}
       (\mu_1^\star)^\top\mu_0^{k+1}&=\frac{(\mu_1^\star)^\top\mu_0^{k}-\gamma(\mu_1^\star)^\top(I_d-\mu_0^k(\mu_0^k)^\top)\partial_{\kappa_0}\mathcal{R}^<(\kappa_0^k,\kappa_1^k) \mu_0^\star}{c_0} =0,\\
       (\mu_0^\star)^\top\mu_1^{k+1}&=\frac{(\mu_0^\star)^\top\mu_1^{k}-\gamma(\mu_0^\star)^\top(I_d-\mu_1^k(\mu_1^k)^\top)\partial_{\kappa_1}\mathcal{R}^<(\kappa_0^k,\kappa_1^k) \mu_1^\star}{c_1} =0,
       \end{align*}
       And
       \begin{align*}
       &(\mu_1^{k+1})^\top\mu_0^{k+1}\\
       &=\frac{(\mu_1^{k})^\top\mu_0^{k}}{c_0c_1}\\
       &-\frac{\gamma(\partial_{\kappa_1}\mathcal{R}^<(\kappa_0^k,\kappa_1^k))((I_d-\mu_1^k(\mu_1^k)^\top)\mu_1^\star)^\top\mu_0^{k}-\gamma(\partial_{\kappa_0}\mathcal{R}^<(\kappa_0^k,\kappa_1^k))((I_d-\mu_0^k(\mu_0^k)^\top)\mu_0^\star)^\top\mu_1^{k}}{c_0c_1}\\
    &+\frac{\gamma^2(\partial_{\kappa_0}\mathcal{R}^<(\kappa_0^k,\kappa_1^k))(\partial_{\kappa_1}\mathcal{R}^<(\kappa_0^k,\kappa_1^k))((I_d-\mu_0^k(\mu_0^k)^\top)\mu_0^\star)^\top(I_d-\mu_1^k(\mu_1^k)^\top)\mu_1^\star}{c_0c_1}=0,
    \end{align*}
    where the last term is zero since $$((I_d-\mu_0^k(\mu_0^k)^\top)\mu_0^\star)^\top(I_d-\mu_1^k(\mu_1^k)^\top)\mu_1^\star=0.$$
    Then $(\mu_0^{k+1},\mu_1^{k+1})\in\mathcal{M}$.
\end{proof}
\begin{lemma}
    When initialized on the manifold $\mathcal{M}$, the iterations generated by \eqref{eq:PGD_iterates} can be reformulated as follows:

    \begin{equation}\label{autonomousnoise}
(\kappa_0^{k+1},\kappa_1^{k+1})=\varphi(\kappa_0^{k},\kappa_1^k),
    \end{equation}
    where $\varphi:[-1,1]^2\rightarrow [-1,1]^2$ is given by $$\varphi(\kappa_0,\kappa_1)=\left( \frac{\kappa_0-\gamma(\partial_{\kappa_0} \mathcal{R}^<(\kappa_0,\kappa_1))(1-\kappa_0^2)}{\sqrt{1+\gamma^2 (\partial_{\kappa_0} \mathcal{R}^<(\kappa_0,\kappa_1))^2(1-\kappa_0^2)}},\frac{\kappa_1-\gamma(\partial_{\kappa_1} \mathcal{R}^<(\kappa_0,\kappa_1))(1-\kappa_1^2)}{\sqrt{1+\gamma^2 (\partial_{\kappa_1} \mathcal{R}^<(\kappa_0,\kappa_1))^2(1-\kappa_1^2)}}\right),$$
    and $\mathcal{R}^<(\kappa_0,\kappa_1)\eqdef \mathcal{R}^<(\kappa_0,\kappa_1,0,0,0)$ as in Lemma \ref{lem:riskmanifold_gmm}.
\end{lemma}
\begin{proof}
    By definition of the iterates and \eqref{chainrule0},\eqref{chainrule1}, we have \begin{align*}
        \kappa_0^{k+1}&=(\mu_0^{k+1})^\top\mu_0^\star=\frac{\kappa_0^k-\gamma\partial_{\kappa_0}\mathcal{R}^<(\kappa_0,\kappa_1)((\mu_0^\star)^\top(I_d-\mu_0^k(\mu_0^k)^\top)\mu_0^\star)}{\sqrt{1+\gamma^2(\partial_{\kappa_0}\mathcal{R}^<(\kappa_0,\kappa_1))^2\Vert (I_d-\mu_0^k(\mu_0^k)^\top)\mu_0^\star\Vert_2^2}}\\
        &=\frac{\kappa_0^k-\gamma\partial_{\kappa_0}\mathcal{R}^<(\kappa_0,\kappa_1)(1-(\kappa_0^k)^2)}{\sqrt{1+\gamma^2\partial_{\kappa_0}\mathcal{R}^<(\kappa_0,\kappa_1)(1-(\kappa_0^k)^2)}},\\
        \kappa_1^{k+1}&=(\mu_1^{k+1})^\top\mu_1^\star=\frac{\kappa_1^k-\gamma\partial_{\kappa_1}\mathcal{R}^<(\kappa_0,\kappa_1)((\mu_1^\star)^\top(I_d-\mu_1^k(\mu_1^k)^\top)\mu_1^\star)}{\sqrt{1+\gamma^2(\partial_{\kappa_1}\mathcal{R}^<(\kappa_0,\kappa_1))^2\Vert (I_d-\mu_1^k(\mu_1^k)^\top)\mu_1^\star\Vert_2^2}}\\
        &=\frac{\kappa_1^k-\gamma\partial_{\kappa_1}\mathcal{R}^<(\kappa_0,\kappa_1)(1-(\kappa_1^k)^2)}{\sqrt{1+\gamma^2\partial_{\kappa_1}\mathcal{R}^<(\kappa_0,\kappa_1)(1-(\kappa_1^k)^2)}}.
    \end{align*} 
\end{proof}

In the following propositions, we consider  $\mathcal{R}^<:[-1,1]^2\rightarrow\R_+$ defined as in Lemma \ref{lem:riskmanifold_gmm} with $\lambda\in ]0, \lambda^*(\sigma,L)]$, where $\lambda^*(\sigma,L)$ is defined in Proposition \ref{optimallambda}.

\begin{proposition}\label{condition_accumulation_noise}
     Let $(\mu_0^0,\mu_1^0)\in \mathcal{M}$, consider \eqref{autonomousnoise} with initial conditions $\kappa_0^0=\langle \mu_0^0,\mu_0^\star\rangle, \kappa_1^0=\langle \mu_1^0,\mu_1^\star\rangle$. Then there exists $\bar{\gamma}>0$ such that for every $0<\gamma<\bar{\gamma}$, the risk $\mathcal{R}^<$ is decreasing along the iterates of \eqref{autonomousnoise}. Besides, the distance between successive iterates tends to zero, and, if $(\kappa_0^\star,\kappa_1^\star)$ is an accumulation point of the sequence of iterates $(\kappa_0^k,\kappa_1^k)_{k\in\mathbb{N}}$, then \begin{equation}\label{criticalnoise}
        (1-(\kappa_0^{\star})^2)\partial_{\kappa_0}\mathcal{R}^<(\kappa_0^{\star},\kappa_1^{\star})=0, \quad (1-(\kappa_1^{\star})^2)\partial_{\kappa_1}\mathcal{R}^<(\kappa_0^{\star},\kappa_1^{\star})=0.
    \end{equation}
\end{proposition}
\begin{proof}
The proof is identical to that of \citet[Proposition 8]{marion2024attention} and is therefore omitted.
\end{proof}

\begin{proposition}\label{finitecritical}
    The points $(\kappa_0,\kappa_1)\in [-1,1]^2$ satisfying \eqref{criticalnoise} belong to the set $$\mathscr{C}\eqdef\{(\pm1,\pm1), (0,\pm 1), (\pm 1,0), (0,0)\}.$$
\end{proposition}
\begin{proof}
    We recall that by Lemma \ref{lem:riskmanifold_gmm}, the risk $\mathcal{R}^<$ restricted to the manifold $\mathcal{M}$, has the following form $$\mathcal{R}^<(\kappa_0,\kappa_1)=A(\kappa_0^4+\kappa_1^4)+B(\kappa_0^2+\kappa_1^2)+C\kappa_0^2\kappa_1^2+D.$$
    Then \begin{align*}
\partial_{\kappa_0}\mathcal{R}^<(\kappa_0,\kappa_1)&=4A\kappa_0^3+2B\kappa_0+2C\kappa_0\kappa_1^2,\\
\partial_{\kappa_1}\mathcal{R}^<(\kappa_0,\kappa_1)&=4A\kappa_1^3+2B\kappa_1+2C\kappa_1\kappa_0^2.
    \end{align*}
    And we can rewrite equations \eqref{criticalnoise} as
    \begin{align*}
        \kappa_0(1-\kappa_0^2)[2A\kappa_0^2+B+C\kappa_1^2]&=0,\\
        \kappa_1(1-\kappa_1^2)[2A\kappa_1^2+B+C\kappa_0^2]&=0,   
    \end{align*}
    Since each equation is a product of 3 terms, the general solution to this system occurs when at least in each equation is zero. By considering only the first two terms in each equation, we obtain the solution set $\{(\pm1,\pm1), (0,\pm 1), (\pm 1,0), (0,0)\}$. Now we consider the case when $2A\kappa_0^2+B+C\kappa_1^2=0$, this implies that:
    \begin{itemize}
        \item If $\kappa_1=0$, then $\kappa_0^2=-\frac{B}{2A}$.
        \item If $\kappa_1^2=1$, then $\kappa_0^2=-\frac{B+C}{2A}$.
        \item If $2A\kappa_1^2+B+C\kappa_0^2\pierremodif{=0}$, then \pierremodif{$2A\kappa_0^2 + C\kappa_1^2 = 2A\kappa_0^2+ C\kappa_1^2$, thus $(2A - C)(\kappa_0^2 - \kappa_1^2)=0$. By inspection, $C<2A$, hence we get} $\kappa_0^2=\kappa_1^2$ and $\kappa_0^2=-\frac{B}{2A+C}$.
        \end{itemize}
        We note the following relation $$-\frac{B+C}{2A} < -\frac{B}{2A+C} < -\frac{B}{2A}.$$
         Further remark \pierremodif{by inspecting the proof of Proposition \ref{optimallambda} that for $\lambda\in ]0, \lambda^*(\sigma,L)]$, we have  $$1\leq-\frac{B+C}{2A}.$$
         Thus the only possible solution when $2A\kappa_0^2+B+C\kappa_1^2=0$ is $\kappa_1^2 = 1$ and $\kappa_0^2 = -\frac{B+C}{2A} = 1$
         Putting everything together,} the solution set is precisely $$\{(\pm1,\pm1), (0,\pm 1), (\pm 1,0), (0,0)\}.$$
    
\end{proof}

\begin{proposition}
     The fixed points of the dynamic can be classified as follows:
    \begin{enumerate}
        \item The points $(\kappa_0,\kappa_1)=(\pm 1, \pm 1)$ are global minima of $\mathcal{R}^<$ on $[-1,1]^2$.  
        \item The points $(\kappa_0,\kappa_1)=(0, \pm 1)$ and $(\pm 1, 0)$ are strict saddle points of $\mathcal{R}^<$ on $[-1,1]^2$.
        \item The point $(\kappa_0,\kappa_1)=(0,0)$ is a global maxima of $\mathcal{R}^<$ on $[-1,1]^2$.
    \end{enumerate}
\end{proposition}
\begin{proof}
    Since $\mathcal{R}^<$ is smooth, its extrema on the square $[-1,1]^2$ occur either at critical points, where the gradient vanishes, or on the boundary. For the chosen range $\lambda \in ]0,\lambda^\star(\sigma,L)[$, the only interior critical point is $(0,0)$, which corresponds to a local maximum. Indeed, the Hessian of $\mathcal{R}^<$ at $(0,0)$ is \[
\nabla^2 \mathcal{R}^<(0,0) =
\begin{pmatrix}
2B & 0 \\
0 & 2B
\end{pmatrix}.
\]
We easily check that both eigenvalues are $2B$, which is negative since $B<0$, therefore concluding that $(0,0)$ is a maximum. The rest of the extrema of $\mathcal{R}^<$ must lie on the boundary of the square, we observe that $$\mathcal{R}^<(\pm 1,\pm 1)< \mathcal{R}^<(x,y), \quad (x,y)\in \partial ([-1,1]^2) \setminus \{(\pm 1,\pm 1)\}.$$ Thus concluding that $(\pm 1,\pm 1)$ are minimizers. Finally, we observe that $$\mathcal{R}^<(\kappa,\pm 1)<\mathcal{R}^<(0,\pm 1)<\mathcal{R}^<(0,\eta),\quad \kappa,\eta\in ]-1,1[.$$
Analogously, we see that
$$\mathcal{R}^<(\pm 1,\kappa)<\mathcal{R}^<(\pm 1,0)<\mathcal{R}^<(\eta,0),\quad \kappa,\eta\in ]-1,1[.$$
This shows that $(0,\pm 1)$ and $(\pm 1,0)$ are strict saddle points, concluding with the proof. We note that when $\lambda = \lambda^\star(\sigma,L)$, Proposition~\ref{optimallambda} implies that the critical points coincide with the minimizers already identified.

\end{proof}
\begin{proposition}
    Consider the context of Proposition \ref{condition_accumulation_noise}, then there exists $\bar{\gamma}>0$ such that for any stepsize $0<\gamma<\bar{\gamma}$, the iterates $(\kappa_0^k,\kappa_1^k)_{k\in\mathbb{N}}$ generated by \eqref{autonomousnoise} converge to an element of $\mathscr{C}$. 
\end{proposition}
\begin{proof}
    By Proposition \ref{condition_accumulation_noise}, the distance between successive iterates $(\kappa_0^k,\kappa_1^k)_{k\in\mathbb{N}}$, then the set of accumulation points of the sequence is connected \cite[Proposition 12.4.1]{lange}. Since we have a finite number of accumulation points by Proposition \ref{finitecritical}, the sequence has a unique accumulation point. Besides, the sequence belongs to the compact set $[-1,1]^2$, then it converges and its limit is one of the nine fixed points. 
\end{proof}
\begin{proposition}
    Consider the context of Proposition \ref{condition_accumulation_noise}, then there exists $\bar{\gamma}>0$ such that for any stepsize $0<\gamma<\bar{\gamma}$, the set of initializations such that the iterates $(\kappa_0^k,\kappa_1^k)_{k\in\mathbb{N}}$ generated by \eqref{autonomousnoise} converge to $(0,\pm 1), (\pm 1,0)$ or $(0,0)$ has Lebesgue measure zero (with respect to the Lebesgue measure on the manifold $\mathcal{M}$).
\end{proposition}
\begin{proof}
    The point $(0,0)$ is a maxima of the risk $\mathcal{R}^<$ on $[-1,1]^2$ and the value of the risk decreases along the iterates of \eqref{eq:PGD_iterates} by Proposition \ref{condition_accumulation_noise}. We follow the ideas presented in the proof of \citet[Proposition 12]{marion2024attention}, we can conclude that $\varphi$ is differentiable on $[-1,1]^2,$ and that its Jacobian is not degenerate, besides $\varphi$ is a local diffeomorphism around $(0,\pm 1)$ and $(\pm 1,0)$, whose Jacobian matrix in each point has one eigenvalue in $[0,1[$ and one eigenvalue in $]1,\infty[$. The result follows from the Center-Stable Manifold Theorem \cite[Theorem III.7]{shub}, we refer to \citet[Proposition 13]{marion2024attention} for a detailed and analogous proof.
\end{proof}

\section{Proofs of Section \ref{sec:statistical_properties}}

\subsection{Proof of Proposition \ref{prop:expectation_oracle_attention}}

We provide hereafter the proof of Proposition \ref{prop:expectation_oracle_attention}, which proof follows. 
\begin{proposition}\label{prop:encoding_stat_properties}
    Consider $(X_\ell)_{1\leq \ell\leq L}$ i.i.d. drawn from \eqref{def:gaussian_mixture_model}.
    Consider also $Z_\ell\in \{0,1\}$ the latent variable of $X_\ell$, i.e. $X_\ell|Z_\ell\sim\mathcal{N}(\mu_{Z_\ell}^\star,\sigma^2I_d)$, and
    $$
    T^{\mathrm{lin},\mu_0^\star,\mu_1^\star}(\mathbb{X})_1=\frac{2}{L}\sum_{k=1}^L (e_k(\mu_0^\star)+e_k(\mu_1^\star))X_k,$$ where $e_k(\mu)\eqdef \lambda\langle X_1,\mu\rangle\langle X_k,\mu\rangle$. Then \begin{align*}
        \mathbb{E}[T^{\mathrm{lin},\mu_0^\star,\mu_1^\star}(\mathbb{X})_1|Z_1=c]&=\mu_{c}^\star\frac{\lambda}{L}[(L+1)+2(L+3)\sigma^2],\quad c=\{0,1\}.\\
        \mathbb{E}[T^{\mathrm{lin},\mu_0^\star,\mu_1^\star}(\mathbb{X})_1]&=\frac{\mu_0^\star+\mu_1^\star}{2}\frac{\lambda}{L}[(L+1)+2(L+3)\sigma^2].
    \end{align*}
    Moreover, when $\lambda=\frac{L}{(L+1)+2(L+3)\sigma^2}$, then the encoding is unbiased, this is $$\mathbb{E}[T^{\mathrm{lin},\mu_0^\star,\mu_1^\star}(\mathbb{X})_1|Z_1=c]=\mu_{c}^\star,\quad c=\{0,1\}.$$

\end{proposition}
\bigskip
\begin{proof} We decompose the following term as follows,
    $$\sum_{k=1}^L \langle X_1,\mu_0^\star\rangle\langle X_k,\mu_0^\star\rangle X_k= \langle X_1,\mu_0^\star\rangle^2 X_1+\sum_{k=2}^L \langle X_1,\mu_0^\star\rangle\langle X_k,\mu_0^\star\rangle X_k$$
    Due to the independence of the variables, \begin{align*}
        \mathbb{E}[T^{\mathrm{lin},\mu_0^\star,\mu_1^\star}(\mathbb{X})_1]&=\frac{2\lambda}{L}(\mathbb{E}[(\langle X_1,\mu_0^\star\rangle^2+\langle X_1,\mu_1^\star\rangle^2) X_1]\\
    &+(L-1)\mathbb{E}[(\langle X_1,\mu_0^\star\rangle\langle X_2,\mu_0^\star\rangle+\langle X_1,\mu_1^\star\rangle\langle X_2,\mu_1^\star) X_2]).
    \end{align*}
    On the one hand we have 
\begin{align*}
X_1 |Z_1&= \mu_{Z_1}^\star + \varepsilon, \quad \varepsilon \sim \mathcal{N}(0, \sigma^2 I_d),\\
\langle X_1, \mu_0^\star \rangle^2|Z_1 &= \langle \mu_{Z_1}^\star, \mu_0^\star \rangle^2 
+ 2 \langle \mu_{Z_1}^\star, \mu_0^\star \rangle \langle \varepsilon, \mu_0^\star \rangle 
+ \langle \varepsilon, \mu_0^\star \rangle^2,\\
\langle X_1, \mu_0^\star \rangle^2 X_1 |Z_1
&= \langle \mu_{Z_1}^\star, \mu_0^\star \rangle^2 \mu_{Z_1}^\star 
+ \langle \mu_{Z_1}^\star, \mu_0^\star \rangle^2 \varepsilon 
+ 2 \langle \mu_{Z_1}^\star, \mu_0^\star \rangle \langle \varepsilon, \mu_0^\star \rangle \mu_{Z_1}^\star \\
&+ 2 \langle \mu_{Z_1}^\star, \mu_0^\star \rangle \langle \varepsilon, \mu_0^\star \rangle \varepsilon
+ \langle \varepsilon, \mu_0^\star \rangle^2 \mu_{Z_1}^\star 
+ \langle \varepsilon, \mu_0^\star \rangle^2 \varepsilon,\\
\mathbb{E}\big[\langle X_1, \mu_0^\star \rangle^2 X_1 \mid Z_1 \big] 
&= \langle \mu_{Z_1}^\star, \mu_0^\star \rangle^2 \mu_{Z_1}^\star
+ \underbrace{\mathbb{E}[\langle \mu_{Z_1}^\star, \mu_0^\star \rangle^2 \varepsilon]}_{0} 
+ \underbrace{\mathbb{E}[2 \langle \mu_{Z_1}^\star, \mu_0^\star \rangle \langle \varepsilon, \mu_0^\star \rangle \mu_{Z_1}^\star]}_{0} \\
&\quad + \underbrace{\mathbb{E}[2 \langle \mu_{Z_1}^\star, \mu_0^\star \rangle \langle \varepsilon, \mu_0^\star \rangle \varepsilon]}_{2 \sigma^2 \langle \mu_{Z_1}^\star, \mu_0^\star \rangle \mu_0^\star} 
+ \underbrace{\mathbb{E}[\langle \varepsilon, \mu_0^\star \rangle^2 \mu_{Z_1}^\star]}_{\sigma^2 \mu_{Z_1}^\star} 
+ \underbrace{\mathbb{E}[\langle \varepsilon, \mu_0^\star \rangle^2 \varepsilon]}_{0},\\
&= \langle \mu_{Z_1}^\star, \mu_0^\star \rangle^2 \mu_{Z_1}^\star 
+ \sigma^2 \big( \mu_{Z_1}^\star + 2 \langle \mu_{Z_1}^\star, \mu_0^\star \rangle \mu_0^\star \big).
\end{align*}

    On the other hand, 
    \begin{align*}
X_i|Z_i &= \mu_{Z_i}^\star + \varepsilon_i, \quad \varepsilon_i \sim \mathcal{N}(0,\sigma^2 I_d), \ \varepsilon_1 \perp\!\!\!\perp\varepsilon_2,\\
\mathbb{E}[\langle X_1,\mu_0^\star\rangle \langle X_2,\mu_0^\star\rangle X_2 \mid Z_1,Z_2]
&= \mathbb{E}[\langle X_1,\mu_0^\star\rangle \mid Z_1]\,
   \mathbb{E}[\langle X_2,\mu_0^\star\rangle X_2 \mid Z_2]\\
&= \langle \mu_{Z_1}^\star,\mu_0^\star\rangle \,
   \mathbb{E}[(\langle \mu_{Z_2}^\star,\mu_0^\star\rangle+\langle \varepsilon_2,\mu_0^\star\rangle)(\mu_{Z_2}^\star+\varepsilon_2)]\\
&= \langle \mu_{Z_1}^\star,\mu_0^\star\rangle \Big(
   \langle \mu_{Z_2}^\star,\mu_0^\star\rangle \mu_{Z_2}^\star
   + \underbrace{\langle \mu_{Z_2}^\star,\mu_0^\star\rangle \mathbb{E}[\varepsilon_2]}_{0}
   + \underbrace{\mathbb{E}[\langle \varepsilon_2,\mu_0^\star\rangle \mu_{Z_2}^\star]}_{0}
   + \sigma^2 \mu_0^\star \Big)\\
&= \langle \mu_{Z_1}^\star,\mu_0^\star\rangle\Big(
   \langle \mu_{Z_2}^\star,\mu_0^\star\rangle \mu_{Z_2}^\star + \sigma^2 \mu_0^\star\Big).
\end{align*}

    Therefore, \begin{align*}
        \mathbb{E}[T^{\mathrm{lin},\mu_0^\star,\mu_1^\star}(\mathbb{X})_1|Z_1,Z_2]=\frac{2\lambda}{L}[&(\langle \mu_{Z_1}^\star,\mu_0^\star\rangle^2+\langle \mu_{Z_1}^\star,\mu_1^\star\rangle^2)\mu_{Z_1}^\star+2\sigma^2\mu_{Z_1}^\star\\
        &+2\sigma^2(\langle \mu_{Z_1}^\star,\mu_0^\star\rangle\mu_0^\star+\langle \mu_{Z_1}^\star,\mu_1^\star\rangle\mu_1^\star)\\
        &+(L-1)(\langle \mu_{Z_1}^\star,\mu_0^\star\rangle\langle \mu_{Z_2}^\star,\mu_0^\star\rangle+\langle \mu_{Z_1}^\star,\mu_1^\star\rangle\langle \mu_{Z_2}^\star,\mu_1^\star\rangle)\mu_{Z_2}^\star\\
        &+(L-1)\sigma^2(\langle \mu_{Z_1}^\star,\mu_0^\star\rangle\mu_0^\star+\langle \mu_{Z_1}^\star,\mu_1^\star\rangle\mu_1^\star)].
    \end{align*}
    And then, 
    \begin{align*}
    \mathbb{E}[&T^{\mathrm{lin},\mu_0^\star,\mu_1^\star}(\mathbb{X})_1|Z_1] \\&=\frac{2\lambda}{L}[(\langle \mu_{Z_1}^\star,\mu_0^\star\rangle^2+\langle \mu_{Z_1}^\star,\mu_1^\star\rangle^2)\mu_{Z_1}^\star+2\sigma^2\mu_{Z_1}^\star+2\sigma^2(\langle \mu_{Z_1}^\star,\mu_0^\star\rangle\mu_0^\star+\langle \mu_{Z_1}^\star,\mu_1^\star\rangle\mu_1^\star)\\
    &+(L-1)\left(\frac{1}{2}+\sigma^2\right)(\langle \mu_{Z_1}^\star,\mu_0^\star\rangle\mu_0^\star+\langle \mu_{Z_1}^\star,\mu_1^\star\rangle\mu_1^\star)].
    \end{align*}
    Which let us conclude that for $c\in\{0,1\}$, \begin{align*}\mathbb{E}[T^{\mathrm{lin},\mu_0^\star,\mu_1^\star}(\mathbb{X})_1|Z_1=c]&=\frac{\mu_c^\star\lambda}{L}((L+1)+2(L+3)\sigma^2),\\
    \mathbb{E}[T^{\mathrm{lin},\mu_0^\star,\mu_1^\star}(\mathbb{X})_1]&=\frac{(\mu_0^\star+\mu_1^\star)\lambda}{2L}((L+1)+2(L+3)\sigma^2).\\
    \end{align*}
    
\end{proof}

\begin{proposition}
     Consider $(X_\ell)_{1\leq \ell\leq L}$ i.i.d. drawn from \eqref{def:gaussian_mixture_model}.
    Consider also $Z_\ell\in \{0,1\}$ the latent variable of $X_\ell$, i.e. $X_\ell|Z_\ell\sim\mathcal{N}(\mu_{Z_\ell}^\star,\sigma^2I_d)$, and
    $$
    T^{\mathrm{lin},\mu_0^\star,\mu_1^\star}(\mathbb{X})_1=\frac{2}{L}\sum_{k=1}^L (e_k(\mu_0^\star)+e_k(\mu_1^\star))X_k,$$ where $e_k(\mu)\eqdef \lambda\langle X_1,\mu\rangle\langle X_k,\mu\rangle$. Then for $c\in\{0,1\}$, \begin{align*}
        &\mathbb{E}[\Vert T^{\mathrm{lin},\mu_0^\star,\mu_1^\star}(\mathbb{X})_1\Vert^2|Z_1=c]\\
        &=\frac{4\lambda^2}{L^2}[1+\sigma^2(d+16)+8\sigma^4(d+7)+8\sigma^6(d+4)]\\
        &+2\lambda^2 \frac{(L-1)}{L^2}[3+\sigma^2(d+28)+4\sigma^4(d+16)+4\sigma^6(d+10)]\\
        &+\lambda^2\frac{(L-1)(L-2)}{L^2}[1+6\sigma^2+12\sigma^4+8\sigma^6].
    \end{align*}
    And \begin{align*}
        &\mathrm{Var}[ T^{\mathrm{lin},\mu_0^\star,\mu_1^\star}(\mathbb{X})_1|Z_1=c]\\
        &=\frac{4\lambda^2}{L^2}[1+\sigma^2(d+16)+8\sigma^4(d+7)+8\sigma^6(d+4)]\\
        &+2\lambda^2 \frac{(L-1)}{L^2}[3+\sigma^2(d+28)+4\sigma^4(d+16)+4\sigma^6(d+10)]\\
        &+\lambda^2\frac{(L-1)(L-2)}{L^2}[1+6\sigma^2+12\sigma^4+8\sigma^6]\\
        &-\frac{\lambda^2}{L^2}[(L+1)+2(L+3)\sigma^2]^2.
    \end{align*}
    When $\lambda=\frac{L}{(L+1)+2(L+3)\sigma^2}$, the encoding is unbiased, with variance \begin{align*}
        &\mathrm{Var}[ T^{\mathrm{lin},\mu_0^\star,\mu_1^\star}(\mathbb{X})_1|Z_1=c]\\
        &=\frac{4}{[(L+1)+2(L+3)\sigma^2]^2}[1+\sigma^2(d+16)+8\sigma^4(d+7)+8\sigma^6(d+4)]\\
        &+ \frac{2(L-1)}{[(L+1)+2(L+3)\sigma^2]^2}[3+\sigma^2(d+28)+4\sigma^4(d+16)+4\sigma^6(d+10)]\\
        &+\frac{(L-1)(L-2)}{[(L+1)+2(L+3)\sigma^2]^2}[1+6\sigma^2+12\sigma^4+8\sigma^6]-1.
    \end{align*}
    
    Besides, when $L\rightarrow\infty$ we get that, 
    \begin{align*}
        \mathrm{Var}[ T^{\mathrm{lin},\mu_0^\star,\mu_1^\star}(\mathbb{X})_1|Z_1=c]\sim2\sigma^2, \quad \lambda\sim\frac{1}{1+2\sigma^2}.
    \end{align*}
    In general, if $\lambda$ is not fixed and $L\rightarrow\infty$, we get 
    \begin{align*}
        \mathrm{Var}[ T^{\mathrm{lin},\mu_0^\star,\mu_1^\star}(\mathbb{X})_1|Z_1=c]\sim2\lambda^2\sigma^2(1+2\sigma^2)^2.
    \end{align*}
\end{proposition}
\begin{remark} 
   We recall that 
\[
\mathrm{Var}[X_1 \mid Z_1 = c] = \sigma^2 d.
\] 
Notably, by selecting $\lambda$ to ensure an unbiased encoding, the variance becomes independent of the dimension $d$ and equals $2\sigma^2$. This shows a variance reduction effect whenever the dimension $d$ is bigger than 2. More generally, $\lambda$ can be chosen independently of $d$ such that
\[
2 \lambda^2 (1 + 2\sigma^2)^2 \ll d.
\]
In this regime, the encoding also asymptotically reduces the variance of $X_1$, conditioned on its cluster assignment, as the number of components $L \to \infty$.

\end{remark}
\begin{proof}
    We note that the needed computations were already stated in the proof of Proposition \ref{prop:structure_of_R1}, we follow as in the proof of Lemma \ref{lem:riskmanifold_gmm}, without loss of generality, we assume $Z_1=0$, then we get that for $\mu_0,\mu_1\in\mathcal{M}$: \begin{align*} 
        &\mathbb{E}[\Vert T^{\mathrm{lin},\mu_0,\mu_1}(\mathbb{X})_1\Vert^2|Z_1=0]\\
        &=\frac{4\lambda^2}{L^2}[\kappa_0^4(1+\sigma^2(d+8))+8\sigma^2\kappa_0^2(1+\sigma^2(d+6))+8\sigma^4(1+\sigma^2(d+4))]\\
        &+4\lambda^2 \frac{L-1}{L^2}[\kappa_0^4(1+5\sigma^2)+4\sigma^2\kappa_0^2(1+6\sigma^2)+\sigma^2\kappa_0^2\kappa_1^2+16\sigma^6]\\
        &+2\lambda^2\frac{(L-1)}{L^2}[\kappa_0^4(1+\sigma^2(d+4))+4\sigma^2\kappa_0^2(1+\sigma^2(d+3))+4\sigma^4(1+\sigma^2(d+2))]\\
        &+\lambda^2\frac{(L-1)(L-2)}{L^2}[\kappa_0^4(1+4\sigma^2)+2\sigma^2\kappa_0^2(1+6\sigma^2)+8\sigma^6].
    \end{align*}
    Recalling that $\kappa_0=\langle \mu_0,\mu_0^\star\rangle$, $\kappa_1=\langle \mu_1,\mu_1^\star\rangle$, in order to compute $\mathbb{E}[\Vert T^{\mathrm{lin},\mu_0^*,\mu_1^*}(\mathbb{X})_1\Vert^2|Z_1=0]$, we just need to replace $\kappa_0$ and $\kappa_1$ by 1 in the previous expression, as follows
    \begin{align*}
        &\mathbb{E}[\Vert T^{\mathrm{lin},\mu_0^\star,\mu_1^\star}(\mathbb{X})_1\Vert^2|Z_1=0]\\
        &=\frac{4\lambda^2}{L^2}[1+\sigma^2(d+8)+8\sigma^2(1+\sigma^2(d+6))+8\sigma^4(1+\sigma^2(d+4))]\\
        &+4\lambda^2 \frac{(L-1)}{L^2}[1+5\sigma^2+4\sigma^2(1+6\sigma^2)+\sigma^2+16\sigma^6]\\
        &+2\lambda^2\frac{(L-1)}{L^2}[1+\sigma^2(d+4)+4\sigma^2(1+\sigma^2(d+3))+4\sigma^4(1+\sigma^2(d+2))]\\
        &+\lambda^2\frac{(L-1)(L-2)}{L^2}[1+4\sigma^2+2\sigma^2(1+6\sigma^2)+8\sigma^6]\\
        &=\frac{4\lambda^2}{L^2}[1+\sigma^2(d+16)+8\sigma^4(d+7)+8\sigma^6(d+4)]\\
        &+2\lambda^2 \frac{(L-1)}{L^2}[3+\sigma^2(d+28)+4\sigma^4(d+16)+4\sigma^6(d+10)]\\
        &+\lambda^2\frac{(L-1)(L-2)}{L^2}[1+6\sigma^2+12\sigma^4+8\sigma^6].
    \end{align*}

\pierremodif{The expression of the variance comes from subtracting to this term the square of the conditional expectation given in Proposition \ref{prop:expectation_oracle_attention}. The asymptotic expressions are then straightforward to derive.}

    
\end{proof}

\subsection{Proof of Proposition \ref{prop:risk_oracle_attention}}
\begin{proof} 
    Assume that the tokens are i.i.d., such that for any $\ell$, $X_\ell \sim \frac{1}{2}\mathcal{N}(\mu_0^\star, \sigma^2 I_d ) + \frac{1}{2}\mathcal{N}(\mu_1^\star, \sigma^2 I_d )$. 
    The risk of the oracle predictor $T^{{\rm lin}, \mu_0^\star, \mu_1^\star}$ can be decomposed as follows
\begin{align}
    \mathcal{L}(T^{{\rm lin}, \mu_0^\star, \mu_1^\star}) = (1+d\sigma^2)-(I_0)+(II_0)+(III_0)- (I)+(II)+(III),
\end{align}
where, \pierremodif{from the proof of Lemma \ref{lem:riskmanifold_gmm}, by taking $\kappa_0=\kappa_1=1$,}
\begin{itemize}
    \item $(I_0)= \frac{4\lambda}{L}\left[(1+\sigma^2(d+4))+2\sigma^2(1+\sigma^2(d+2))\right]$.
    \item $(II_0)=\frac{4\lambda^2}{L^2}[1+\sigma^2(d+16)+8\sigma^4(d+7)+8\sigma^6(d+4))]$.
    \item $(III_0)=4\lambda^2 \frac{L-1}{L^2}[1+10\sigma^2+24\sigma^4 +16\sigma^6]$.
    \item $(I)=2\lambda\frac{L-1}{L}[1+4\sigma^2+4\sigma^4]$.
    \item $(II)=2\lambda^2\frac{(L-1)}{L^2}[1+\sigma^2(d+8)+4\sigma^4(d+7)+4\sigma^6(d+2)]$.
    \item $(III)=\lambda^2\frac{(L-1)(L-2)}{L^2}[1+6\sigma^2+12\sigma^4 +8\sigma^6]$.
\end{itemize} 
When $L$ tends to $\infty$, only the first term together with $(I)$ and $(III)$ contribute. Therefore, we obtain that 
$$
\mathcal{L}(T^{{\rm lin}, \mu_0^\star, \mu_1^\star}) \underset{L\to \infty} \sim (1+d\sigma^2)
-\lambda[2+8\sigma^2+8\sigma^4]
+ \lambda^2[1+6\sigma^2+12\sigma^4 +8\sigma^6].
$$
Choosing $\lambda = \frac{1+4\sigma^2+4\sigma^4}{1+6\sigma^2+12\sigma^4 +8\sigma^6}$ (its value being independent of $L$) minimizes the equivalent bound obtained above. With such a choice, the equivalent becomes
\begin{align*}
\mathcal{L}(T^{{\rm lin}, \mu_0^\star, \mu_1^\star}) &\underset{L\to \infty}{ \sim}  (1+d\sigma^2)-\frac{(1+4\sigma^2+4\sigma^4)^2}{1+6\sigma^2+12\sigma^4 +8\sigma^6}\\
&\underset{L\to \infty}{ \sim}  (1+d\sigma^2)-\frac{(1+2\sigma^2)^4}{(1+2\sigma^2)^3}\\
&\underset{L\to \infty}{ \sim}  (1+d\sigma^2)-1-2\sigma^2 \\
&\underset{L\to \infty}{ \sim}  \sigma^2(d-2). 
\end{align*}
\end{proof}





\begin{remark} 
    In the case of the degenerate case ($\sigma=0$), similar computations lead to 
    \begin{align*}
    \mathcal{L}(T^{{\rm lin}, \mu_0^\star, \mu_1^\star}) 
    &= 1 
    - \frac{4\lambda}{L}  
    + 4\frac{\lambda^2}{L^2} 
    + 4\lambda^2 \frac{L-1}{L^2} 
    - 2\lambda \frac{L-1}{L}
    + 2\lambda^2 \frac{L-1}{L^2} 
    + \lambda^2 \frac{(L-1)(L-2)}{L^2}  \\
    &= 1 
    -  2\lambda    \frac{L+1}{L} +\lambda^2 
    \frac{(L+3)}{L}
    \end{align*}
    Optimizing this quantity w.r.t.\ $\lambda$ leads to choose $\lambda = \frac{L+1}{L+3}$, plugging this value for $\lambda$ gives 
    \begin{align*}
    \mathcal{L}(T^{{\rm lin}, \mu_0^\star, \mu_1^\star}) 
    &=   1- \frac{(L+1)^2}{L(L+3)}.
    \end{align*}
   In the degenerate case, we observe that as the sequence length tends to infinity, the risk of the attention-based predictor with oracle parameters converges to zero, matching that of the optimal quantizer.
\end{remark}



\section{Technical results} \label{sec:computation}


This section gathers a series of technical results about Gaussian random variables, used to derive expression of the risk in the rest of the document. 

\begin{lemma}\label{isserlis} \citep{isserlis}.
    Consider $G\sim \mathcal{N}(0,\sigma^2 I_d)$ and $\mu_a,\mu_b, \mu_c\in\R^d$, then \begin{enumerate}
        \item $\mathbb{E}[\Vert G\Vert^2]=\sigma^2 d$.
        \item $\mathbb{E}[\langle \mu_a,G\rangle]=0$.
        \item $\mathbb{E}(\langle \mu_a,G\rangle G)=\sigma^2\mu_a$.
        \item $\mathbb{E}[\langle \mu_a,G\rangle\langle \mu_b,G\rangle]=\sigma^2\langle \mu_a,\mu_b\rangle$.
        \item $\mathbb{E}[\langle \mu_a,G\rangle^2\langle \mu_b,G\rangle^2]=\sigma^4(\Vert\mu_a\Vert^2\Vert\mu_b\Vert^2+2\langle \mu_a,\mu_b\rangle^2)$.
        \item $\mathbb{E}[\langle \mu_a,G\rangle\langle \mu_b,G\rangle^2\langle\mu_c,G\rangle]=\sigma^4(\Vert\mu_b\Vert^2\langle\mu_a,\mu_c\rangle+2\langle\mu_a,\mu_b\rangle\langle\mu_b,\mu_c\rangle).$
        \item $\mathbb{E}[\langle\mu_a,G\rangle\langle\mu_b,G\rangle\Vert G\Vert^2]=\sigma^4(d+2)\langle\mu_a,\mu_b\rangle$.
        \item $\mathbb{E}[\langle\mu_a,G\rangle^2\langle\mu_b,G\rangle^2\Vert G\Vert^2]=\sigma^6(d+4)(\Vert\mu_a\Vert^2\Vert\mu_b\Vert^2+2\langle\mu_a,\mu_b\rangle^2$).
    \end{enumerate}
\end{lemma}

\medskip 

    

\begin{lemma}\label{I0}
    Consider $X\sim \mathcal{N}(\mu^\star,\sigma^2 I_d)$ where $\Vert\mu^\star\Vert=1$ and $\mu_a\in\R^d$, then
    
        $$
        \mathbb{E}[\langle X,\mu_a\rangle^2\Vert X\Vert^2]=\langle\mu^\star,\mu_a\rangle^2(1+\sigma^2(d+4))+\sigma^2\Vert \mu_a\Vert^2(1+\sigma^2(d+2)).
        $$
\end{lemma}
\begin{proof}
We decompose $X$ as follows,
    \begin{align*}
X &= \mu^\star + \varepsilon,\qquad \varepsilon\sim\mathcal N(0,\sigma^2 I_d),\\
\langle X,\mu_a\rangle^2\|X\|^2
&= \big(\langle\mu^\star,\mu_a\rangle + \langle\varepsilon,\mu_a\rangle\big)^2
   \big(1 + 2\langle\varepsilon,\mu^\star\rangle + \|\varepsilon\|^2\big)\\[4pt]
\mathbb{E}[\langle X,\mu_a\rangle^2\|X\|^2]
&= \langle\mu^\star,\mu_a\rangle^2
   + \langle\mu^\star,\mu_a\rangle^2\mathbb{E}\|\varepsilon\|^2
   + 4\langle\mu^\star,\mu_a\rangle\,\mathbb{E}[\langle\varepsilon,\mu_a\rangle\langle\varepsilon,\mu^\star\rangle]\\
&\qquad + \mathbb{E}[\langle\varepsilon,\mu_a\rangle^2]
   + \mathbb{E}\big[\langle\varepsilon,\mu_a\rangle^2\|\varepsilon\|^2\big],
\end{align*}
Then, by Lemma \ref{isserlis} one obtains \[
\mathbb{E}[\langle X,\mu_a\rangle^2\|X\|^2]
= \langle\mu^\star,\mu_a\rangle^2(1+\sigma^2(d+4))+\sigma^2\|\mu_a\|^2(1+\sigma^2(d+2)).
\]
\end{proof}
\bigskip
\begin{lemma}\label{II0}
    Let $X\sim \mathcal{N}(\mu^\star,\sigma^2I_d)$, where $\Vert\mu^\star\Vert=1$ and $\mu_a,\mu_b\in\R^d$, then $$p_0(\mu_a,\mu_b,\mu^\star)\eqdef \mathbb{E}(\langle X,\mu_a\rangle^2\langle X,\mu_b\rangle^2\Vert X\Vert^2)$$ can be expressed as
    \begin{align*}
p_0(\mu_a,\mu_b,\mu^\star)&=\langle\mu^\star,\mu_a\rangle^2\langle\mu^\star,\mu_b\rangle^2\\
&+\sigma^2\left(\langle \mu^\star,\mu_b\rangle^2\Vert \mu_a\Vert^2+4\langle\mu^\star,\mu_a\rangle\langle\mu^\star,\mu_b\rangle\langle\mu_a,\mu_b\rangle+\langle\mu^\star,\mu_a\rangle^2\Vert \mu_b\Vert^2\right)\\
        &+\sigma^2(d+8)\langle\mu^\star,\mu_a\rangle^2\langle\mu^\star,\mu_b\rangle^2\\
        &+\sigma^4(\Vert \mu_a\Vert^2\Vert \mu_b\Vert^2+2\langle \mu_a,\mu_b\rangle^2+(d+6)(\Vert\mu_a\Vert^2\langle \mu^\star,\mu_b\rangle^2+\Vert\mu_b\Vert^2\langle \mu^\star,\mu_a\rangle^2))\\
        &+4\sigma^4(d+6)\langle \mu^\star,\mu_a\rangle\langle \mu^\star,\mu_b\rangle\langle \mu_a,\mu_b\rangle+\sigma^6(d+4)(\Vert \mu_a\Vert^2\Vert \mu_b\Vert^2+2\langle \mu_a,\mu_b\rangle^2).
    \end{align*} 
\end{lemma}
\begin{proof}
Write \(X=\mu^\star+\varepsilon\) with \(\varepsilon\sim\mathcal N(0,\sigma^2 I_d)\).
Then
\[
\langle X,\mu_a\rangle^2\langle X,\mu_b\rangle^2\|X\|^2
= \big(\langle\mu^\star,\mu_a\rangle+\langle\varepsilon,\mu_a\rangle\big)^2
  \big(\langle\mu^\star,\mu_b\rangle+\langle\varepsilon,\mu_b\rangle\big)^2
  \big(1+2\langle\varepsilon,\mu^\star\rangle+\|\varepsilon\|^2\big).
\]
Expand the product and drop all odd-moment terms of \(\varepsilon\) (their expectation is \(0\)). The surviving types of expectations are of the form of Lemma \ref{isserlis}, collecting all nonzero contributions after expansion and simplifying gives the stated formula for \(p_0(\mu_a,\mu_b,\mu^\star)\).
\end{proof}

\bigskip

The proofs of Lemmas \ref{III0}--\ref{III} follow the same approach as the proof of Lemma \ref{II0}: we rewrite $X_i \mid Z_i$ as $\mu_{Z_i}^\star + \varepsilon_i$, where $(\varepsilon_i)_i$ are i.i.d.\ random variables with $\varepsilon_i \sim \mathcal{N}(0,\sigma^2 I_d)$. We then expand the product, discard odd-moment terms, and apply Lemma \ref{isserlis} to obtain the desired results.

\begin{lemma}\label{III0}
Let $Z_1$ and $Z_2\in \{0,1\}$ be fixed. Consider two independent $\R^d-$valued random variables $X_1$ and $X_2$, such that 
$$
X_i|Z_i\sim \mathcal{N}(\mu_{Z_i}^\star,\sigma^2 I_d),\quad \text{for each $i=\{1,2\}$},
$$ 
where the unit vectors $\mu_a^\star, \mu_b^\star$   (i.e., $\Vert\mu_a^\star\Vert=\Vert\mu_b^\star\Vert=1$) are orthogonal.  For $\mu_a,\mu_b,\mu_c\in\R^d$, define
\begin{align*}
    p_{1,0}(\mu_a,\mu_b,\mu_{Z_1}^\star,\mu_c)\eqdef\mathbb{E}[\langle X_1,\mu_a\rangle^2\langle X_1,\mu_b\rangle\langle X_1,\mu_c\rangle|Z_1].
\end{align*}
This quantity satisfies
\begin{align*}
    p_{1,0}(\mu_a,\mu_b,\mu_{Z_1}^\star,\mu_c)&=\langle \mu_{Z_1}^\star,\mu_a\rangle^2\langle \mu_{Z_1}^\star,\mu_b\rangle\langle \mu_{Z_1}^\star,\mu_c\rangle\\
    &+\sigma^2\left[\Vert\mu_a\Vert^2\langle \mu_{Z_1}^\star,\mu_b\rangle\langle \mu_{Z_1}^\star,\mu_c\rangle+2\langle \mu_{Z_1}^\star,\mu_a\rangle(\langle \mu_{Z_1}^\star,\mu_c\rangle\langle \mu_a,\mu_b\rangle+\langle \mu_{Z_1}^\star,\mu_b\rangle\langle \mu_a,\mu_c\rangle)\right]\\
&+\sigma^2\left[\langle\mu_{Z_1}^\star,\mu_a\rangle^2\langle\mu_b,\mu_c\rangle\right]\\
    &+\sigma^4(\Vert\mu_a\Vert^2\langle \mu_b,\mu_c\rangle+2\langle\mu_a,\mu_b\rangle\langle\mu_a,\mu_c\rangle).
\end{align*}
Moreover, we define \begin{align*}
p_1(\mu_a,\mu_b,\mu_{Z_1}^\star,\mu_{Z_2}^\star)\eqdef\mathbb{E}[\langle X_1,\mu_a\rangle^2\langle X_1,\mu_b\rangle\langle X_2,\mu_b\rangle\langle X_1,X_2\rangle|Z_1,Z_2],
\end{align*}
which satisfies
\begin{align*}
p_1(\mu_a,\mu_b,\mu_{Z_1}^\star,\mu_{Z_2}^\star)=\langle \mu_{Z_2}^\star,\mu_b\rangle p_{1,0}(\mu_a,\mu_b,\mu_{Z_1}^\star,\mu_{Z_2}^\star)+\sigma^2p_{1,0}(\mu_a,\mu_b,\mu_{Z_1}^\star,\mu_b).
\end{align*}
\end{lemma}
\bigskip
\begin{lemma}\label{I}
    Let $Z_1,Z_2 \in \{0,1\}$ be fixed. Consider two independent $\R^d-$valued random variables $X_1$ and $X_2$, such that $$
    X_i|Z_i\sim \mathcal{N}(\mu_{Z_i}^\star,\sigma^2 I_d),\quad \text{for each $i=\{1,2\}$},
    $$ 
where the unit vectors $\mu_a^\star, \mu_b^\star$ (i.e., $\Vert\mu_a^\star\Vert=\Vert\mu_b^\star\Vert=1$) are orthogonal.  For $\mu_a,\mu_b\in\R^d$, we get that 

\begin{align*}
        &\mathbb{E}[(\langle X_1,\mu_a\rangle\langle X_2,\mu_a\rangle+\langle X_1,\mu_b\rangle\langle X_2,\mu_b\rangle)\langle X_1, X_2\rangle|Z_1,Z_2]\\
        &=\langle \mu_{Z_1}^\star,\mu_{Z_2}^\star\rangle(\langle \mu_{Z_1}^\star,\mu_a\rangle\langle \mu_{Z_2}^\star,\mu_a\rangle+\langle \mu_{Z_1}^\star,\mu_b\rangle\langle \mu_{Z_2}^\star,\mu_b\rangle)\\
        &+\sigma^2(\langle \mu_a,\mu_{Z_1}^\star\rangle^2+\langle \mu_a,\mu_{Z_2}^\star\rangle^2+\langle \mu_b,\mu_{Z_1}^\star\rangle^2+\langle \mu_b,\mu_{Z_2}^\star\rangle^2)\\
        &+\sigma^4 (\Vert\mu_a\Vert^2+\Vert\mu_b\Vert^2)
    \end{align*}
\end{lemma}
\bigskip
\begin{lemma}\label{II}
    Let $Z_1,Z_2 \in \{0,1\}$ be fixed. Consider two independent $\R^d-$valued random variables $X_1$ and $X_2$, such that 
    $$
    X_i|Z_i\sim \mathcal{N}(\mu_{Z_i}^\star,\sigma^2 I_d),\quad \text{for each $i=\{1,2\}$},
    $$ 
where the unit vectors $\mu_a^\star, \mu_b^\star$  (i.e., $\Vert\mu_a^\star\Vert=\Vert\mu_b^\star\Vert=1$) are orthogonal. For $\mu_a,\mu_b\in\R^d$, define also: 
   \begin{align*}
       p_{2,0}(\mu_a,\mu_b,\mu_{Z_1}^\star)&\eqdef \mathbb{E}[\langle X_1,\mu_a\rangle\langle X_1,\mu_b\rangle|Z_1]\\
        p_{2,1}(\mu_a,\mu_b,\mu_{Z_2}^\star)&\eqdef\mathbb{E}[\langle X_2,\mu_a\rangle\langle X_2,\mu_b\rangle\Vert X_2\Vert^2|Z_2].
   \end{align*}
These quantities satisfy
\begin{align*}
    p_{2,0}(\mu_a,\mu_b,\mu_{Z_1}^\star)&=\langle \mu_{Z_1}^\star,\mu_a\rangle\langle \mu_{Z_1}^\star,\mu_b\rangle+\sigma^2\langle \mu_a,\mu_b\rangle,\\
    p_{2,1}(\mu_a,\mu_b,\mu_{Z_2}^\star)&=\langle\mu_{Z_2}^\star,\mu_a\rangle\langle\mu_{Z_2}^\star,\mu_b\rangle+\sigma^2((d+4)\langle \mu_{Z_2}^\star,\mu_a\rangle\langle\mu_{Z_2}^\star,\mu_b\rangle+\langle\mu_a,\mu_b\rangle)\\
        &+\sigma^4(d+2)\langle \mu_a,\mu_b\rangle.
\end{align*}
   
    Moreover, we define
   \begin{align*}
       p_2(\mu_a,\mu_b,\mu_{Z_1}^\star,\mu_{Z_2}^\star)&\eqdef \mathbb{E}[\langle X_1,\mu_a\rangle\langle X_2,\mu_a\rangle\langle X_1,\mu_b\rangle\langle X_2,\mu_b\rangle\lVert X_2\Vert^2|Z_1,Z_2],\\
       \end{align*}
       which satisfies
       \begin{align*}
       p_2(\mu_a,\mu_b,\mu_{Z_1}^\star,\mu_{Z_2}^\star)&=p_{2,0}(\mu_a,\mu_b,\mu_{Z_1}^\star)p_{2,1}(\mu_a,\mu_b,\mu_{Z_2}^\star).
       \end{align*}
\end{lemma}
\bigskip
\begin{lemma}\label{III}
    Let $Z_1,Z_2,Z_3 \in \{0,1\}$ be fixed. Consider three independent $\R^d-$valued random variables $X_1, X_2, X_3$, where $$X_i|Z_i\sim \mathcal{N}(\mu_{Z_i}^\star,\sigma^2 I_d),\quad \text{for each $i=\{1,2,3\}$},$$ 
such that $\mu_a^\star, \mu_b^\star$ unit vectors  (i.e., $\Vert\mu_a^\star\Vert=\Vert\mu_b^\star\Vert=1$) are orthogonal. For $\mu_a,\mu_b\in\R^d$, define also: 
   \begin{align*}
       p_{3,0}(\mu_a,\mu_b,\mu_{Z_1}^\star)&\eqdef \mathbb{E}[\langle X_1,\mu_a\rangle\langle X_1,\mu_b\rangle|Z_1],\\
        p_{3,1}(\mu_a,\mu_b,\mu_{Z_2}^\star,\mu_{Z_3}^\star)&\eqdef \mathbb{E}[\langle X_2,\mu_a\rangle\langle X_3,\mu_b\rangle\langle X_2,X_3\rangle|Z_2,Z_3].
   \end{align*}
These quantities satisfy
\begin{align*}
    p_{3,0}(\mu_a,\mu_b,\mu_{Z_1}^\star)&=\langle \mu_{Z_1}^\star,\mu_a\rangle\langle \mu_{Z_1}^\star,\mu_b\rangle+\sigma^2\langle \mu_a,\mu_b\rangle,\\
    p_{3,1}(\mu_a,\mu_b,\mu_{Z_2}^\star,\mu_{Z_3}^\star)&=\langle\mu_{Z_2}^\star,\mu_a\rangle\langle\mu_{Z_3}^\star,\mu_b\rangle\langle\mu_{Z_2}^\star,\mu_{Z_3}^\star\rangle\\
    &\qquad +\sigma^2(\langle\mu_{Z_2}^\star,\mu_a\rangle\langle\mu_{Z_2}^\star,\mu_b\rangle+\langle\mu_{Z_3}^\star,\mu_a\rangle\langle\mu_{Z_3}^\star,\mu_b\rangle)\\
        &\qquad +\sigma^4\langle \mu_a,\mu_b\rangle.
\end{align*}
   
   Moreover, we define
   \begin{align*}
       p_3(\mu_a,\mu_b,\mu_{Z_1}^\star,\mu_{Z_2}^\star,\mu_{Z_3}^\star)&\eqdef \mathbb{E}[\langle X_1,\mu_a\rangle\langle X_2,\mu_a\rangle\langle X_1,\mu_b\rangle\langle X_3,\mu_b\rangle\langle X_2,X_3\rangle|Z_1,Z_2,Z_3],
   \end{align*}
   which satisfies
   \begin{align*}
p_3(\mu_a,\mu_b,\mu_{Z_1}^\star,\mu_{Z_2}^\star,\mu_{Z_3}^\star)&=p_{3,0}(\mu_a,\mu_b,\mu_{Z_1}^\star)p_{3,1}(\mu_a,\mu_b,\mu_{Z_2}^\star,\mu_{Z_3}^\star).
   \end{align*}
\end{lemma}



\section{Experimental details}\label{sec:xp_details}

This section provides algorithmic details, choices of parameters, and settings used for the plots displayed in Sections \ref{sec:dirac} and \ref{sec:gmm}.

\subsection{Projected Stochastic Gradient Descent}\label{sec:psgd}
 We formally define the method Projected Stochastic Gradient Descent (PSGD), which we run for our numerical experiments.

\paragraph{PSGD iterates for linear attention heads.}
\label{def:psgdlinear}
Given the objective function $\mathcal{R}^{\rho}:(\mathbb{S}^{d-1})^2\rightarrow\R$ defined in \eqref{penalizationreal}, we define $h:(\mathbb{S}^{d-1})^2\times \R^{L\times d}$  as \begin{equation}\label{hdef}
    h(\mu_0,\mu_1,\mathcal{X})=\Big\Vert X_1-\frac{2}{L}\sum_{k=1}^L \lambda[ X_1^\top(\mu_0\mu_0^\top+\mu_1\mu_1^\top) X_k]X_k\Big\Vert_2^2+\rho\langle \mu_0,X_1\rangle^2\langle\mu_1,X_1\rangle^2,
\end{equation} where $X_i$ is the $i-$th row of the matrix $\mathcal{X}$. Consequently we can write $$\mathcal{R}^{\rho}(\mu_0,\mu_1)=\mathbb{E}_{\mathbb{X}\sim \mathcal{D}}[h(\mu_0,\mu_1,\mathbb{X})],$$ where $\mathcal{D}$ is the distribution over $\R^{L\times d}$ where each row is i.i.d. according to $$\frac{1}{2}\mathcal{N}(\mu_0^\star,\sigma^2I_d)+\frac{1}{2}\mathcal{N}(\mu_1^\star,\sigma^2I_d).$$ Then, given and an initialization $(\mu_0^0,\mu_1^0)\in(\mathbb{S}^{d-1})^2$, a stepsize $\gamma$, we define $(\mu_0^k,\mu_1^k)\in(\mathbb{S}^{d-1})^2$ recursively by: \begin{equation}\label{PGDrho}\tag{$\mathrm{PSGD}$}
    \begin{split}
        g_0^k&=\frac{1}{M}\sum_{i=1}^M\nabla_{\mu_0} h(\mu_0^k,\mu_1^k,\xi_i^k),\\
        g_1^k&=\frac{1}{M}\sum_{i=1}^M\nabla_{\mu_1} h(\mu_0^k,\mu_1^k,\xi_i^k),\\
        \mu_0^{k+1}&=\frac{\mu_0^k-\gamma(I_d-\mu_0^k(\mu_0^k)^\top)g_0^k}{\Vert \mu_0^k-\gamma(I_d-\mu_0^k(\mu_0^k)^\top)g_0^k\Vert_2},\\
        \mu_1^{k+1}&=\frac{\mu_1^k-\gamma(I_d-\mu_1^k(\mu_1^k)^\top)g_1^k}{\Vert \mu_1^k-\gamma(I_d-\mu_1^k(\mu_1^k)^\top)g_1^k\Vert_2},
        \end{split}
    \end{equation}
    where $M$ is called the batch size, and for each $k\in\mathbb{N}$, $(\xi_i^k)_{i=\{1,\ldots,M\}}$ are $M$ independents samples of $\mathcal{D}$.

\paragraph{PSGD iterates for softmax attention heads.}
\label{psgdsoft}
    Given the objective function $\mathcal{R}^{\mathrm{soft},\rho_0}:(\mathbb{S}^{d-1})^2\times\R^2\rightarrow\R$ defined in \eqref{risk_softmax_regularized}, for simplicity, we note that for an appropriate $h_0$, we can write $$\mathcal{R}^{\mathrm{soft},\rho_0}(\mu_0,\mu_1,\psi,\lambda)=\mathbb{E}_{\mathbb{X}\sim \mathcal{D}}[h_0(\mu_0,\mu_1,\psi,\lambda,\mathbb{X})],$$ where $\mathcal{D}$ is the distribution over $\R^{L\times d}$ where each row is i.i.d. according to $$\frac{1}{2}\mathcal{N}(\mu_0^\star,\sigma^2I_d)+\frac{1}{2}\mathcal{N}(\mu_1^\star,\sigma^2I_d).$$ Then, given and an initialization $(\mu_0^0,\mu_1^0)\in(\mathbb{S}^{d-1})^2$, $(\psi^0,\lambda^0)=(2,3)$, a stepsize $\gamma$, we define $(\mu_0^k,\mu_1^k)\in(\mathbb{S}^{d-1})^2$ and $(\psi^k,\lambda^k)\in\R^2$ recursively by: \begin{equation}\label{PGDtheta}\tag{$\mathrm{PSGD_{soft}}$}
    \begin{split}
        g_0^k&=\frac{1}{M}\sum_{i=1}^M\nabla_{\mu_0} h_0(\mu_0^k,\mu_1^k,\psi^k,\lambda^k,\xi_i^k),\\
        g_1^k&=\frac{1}{M}\sum_{i=1}^M\nabla_{\mu_1} h_0(\mu_0^k,\mu_1^k,\psi^k,\lambda^k,\xi_i^k),\\
        \mu_0^{k+1}&=\frac{\mu_0^k-\gamma(I_d-\mu_0^k(\mu_0^k)^\top)g_0^k}{\Vert \mu_0^k-\gamma(I_d-\mu_0^k(\mu_0^k)^\top)g_0^k\Vert_2},\\
        \mu_1^{k+1}&=\frac{\mu_1^k-\gamma(I_d-\mu_1^k(\mu_1^k)^\top)g_1^k}{\Vert \mu_1^k-\gamma(I_d-\mu_1^k(\mu_1^k)^\top)g_1^k\Vert_2},\\
        \psi^{k+1}&=\psi^k-\gamma \frac{1}{M}\sum_{i=1}^M\nabla_{\psi}h_0(\mu_0^k,\mu_1^k,\psi^k,\lambda^k,\xi_i^k),\\
        \lambda^{k+1}&=\lambda^k-\gamma \frac{1}{M}\sum_{i=1}^M\nabla_{\lambda}h_0(\mu_0^k,\mu_1^k,\psi^k,\lambda^k,\xi_i^k),
        \end{split}
    \end{equation}
    where $M$ is called the batch size, and for each $k\in\mathbb{N}$, $(\xi_i^k)_{i=\{1,\ldots,M\}}$ are $M$ independents samples of $\mathcal{D}$.
\begin{remark}
Gradient computations in the numerical experiments were carried out using JAX \citep{jax}.
\end{remark}

\subsection{Experimental details}
In the following, we provide the experimental setup corresponding to Sections~\ref{sec:dirac} and~\ref{sec:gmm}.

We use input sequences of length $L=30$ of 5-dimensional tokens ($d=5$), and define the true centroids as $\mu_0^\star = (0, 0, 0, 0, 1)$ and $\mu_1^\star = (-1, 0, 0, 0, 0)$. We recall that the metric used to quantify the distance to the centroids (up to a sign) is defined in~\eqref{def:xp_metric2}.

\paragraph{Experimental details of Section \ref{sec:dirac}.}
\label{xp_details_dirac}

Regarding the experiment on the manifold, i.e., Figure \ref{fig:ig:risk_linear_noiseless_manifold}, we perform $10^4$ \eqref{PGDrho} iterations without regularization ($\rho=0$) with a learning rate of $\gamma=0.01$, $\lambda = 0.6$, batch size $M=256$. The experiment is repeated across 10 independent runs, each initialized randomly on the manifold $\tilde{\mathcal{M}}$.

For the rest of the experiments of this section, we adopt the same setup as before, with the exception that each run is randomly initialized on the unit sphere. In Figure \ref{fig:risk_linear_noiseless_sphere}, we perform $10^4$ iterations to observe that without adding a regularization term, we only get partial alignment of the Transformer parameters towards the true centroids. 

Then, in Figure \ref{linear_noiseless_reg} we perform $5 \times 10^3$ iterations of \eqref{PGDrho} to minimize the regularized risk $\mathcal{R}^\rho$ for $15$ values of the regularization strength $\rho$, linearly spaced in $[0, 0.3]$. Finally, in Figure \ref{linear_noiseless_iters} we choose $\rho=0.1$ and perform $10^4$ \eqref{PGDrho} iterations.

\paragraph{Experimental details of Section \ref{sec:gmm}.}\label{xp_details_gmm}

Regarding the experiment on the manifold, i.e., Figure \ref{fig:linear_manifold_iters}, we run the algorithm for $10^4$ iterations without regularization ($\rho=0$), with a learning rate of $\gamma=0.01$, batch size $M=256$, and choosing $\lambda=0.6$ for $\sigma=0.3$, and $\lambda=0.2$ for $\sigma=1$. The experiment is repeated across 10 independent runs, each initialized randomly on the manifold $\mathcal{M}$. 

For the rest of the experiments of this section, we adopt the same setup as before, with the exception that each run is randomly initialized on the unit sphere. In Figure \ref{linear_noise_reg} we perform $5 \times 10^3$ iterations of \eqref{PGDrho} to minimize the regularized risk $\mathcal{R}^\rho$ for $30$ values of the regularization strength $\rho$, linearly spaced in $[0, 3]$. Finally, in Figure \ref{linear_noise_iters} we choose $\rho=0.2$ and perform $10^4$ \eqref{PGDrho} iterations.

\begin{remark}
    All experiments in Section \ref{sec:dirac} and \ref{sec:gmm} can be run on a standard laptop. Most complete within a few minutes, with the exception of those in Figures \ref{linear_noiseless_reg} and \ref{linear_noise_reg}, which require approximately 20 minutes and up to an hour, respectively, due to repeated problem-solving across a grid of regularization strengths.
\end{remark}

\section{Additional numerical experiments}
\label{app:xp_higher_dim}

In what follows, we first present numerical experiments in dimension 100. We then vary the dimension from 4 to 200. Results are shown only for the linear approach, as the softmax variant exhibits numerical instability in higher dimensions.

\subsection{Influence of the dimension}
\label{app:xp_dim}

\textbf{Experiments in \texorpdfstring{$\R^{100}$}{}.}
We use input sequences of length $L=30$ in  $\R^{100}$, where we define two centroids: $\mu_0^\star=(\underbrace{0,\ldots,0}_{\text{99 times}},1)$ and $\mu_1^\star=(-1,\underbrace{0,\ldots,0}_{\text{99 times}})$. The model is trained using \eqref{PGDrho} with an online batch sampling strategy, with a batch size of 256, and a learning rate of 0.01. Due to the big dimensionality of the problem, we modify the concept introduced as distance to the centroid up to a sign by the concept of minimal root mean squared error, which is nothing but the distance to the centroid (up to a sign) divided by the square root of the dimension, i.e.,

\[
\text{Minimal $\mathrm{RMSE}$}=\frac{1}{\sqrt{d}}\min_{\pi\in S_2}\min_{s\in \{-1,1\}^2}\sqrt{\sum_{i=0}^1 \Vert \hat{\mu}_{\pi(i)}-s_i\mu_i^\star\Vert^2},
\]
where $S_2$ is the symmetric group of order 2, \(\mu_0^\star, \mu_1^\star\) are the true centroids, and \(\hat{\mu}_0, \hat{\mu}_1\) are the returned parameters from \eqref{PGDrho}. In Figures \ref{linear_manifold_big}, \ref{linear_iters_big} we can observe the behavior of the RMSE over the iterations for different levels of noise $\sigma$. We remark that in Figure \ref{linear_manifold_big} we initialize on the manifold $\mathcal{M}$, and there is no regularization term (i.e. $\rho=0$), in Figure \ref{linear_iters_big} we initialize randomly over the unit sphere and we set $\rho=0.2$. In both experiments we set $\lambda=0.6$ for the case $\sigma=0$ and $\sigma=0.3$, and $\lambda=0.2$ for the case $\sigma=1$.

\begin{figure}
    \centering
    \begin{subfigure}[b]{0.325\textwidth}
         \centering
         \includegraphics[width=\textwidth]{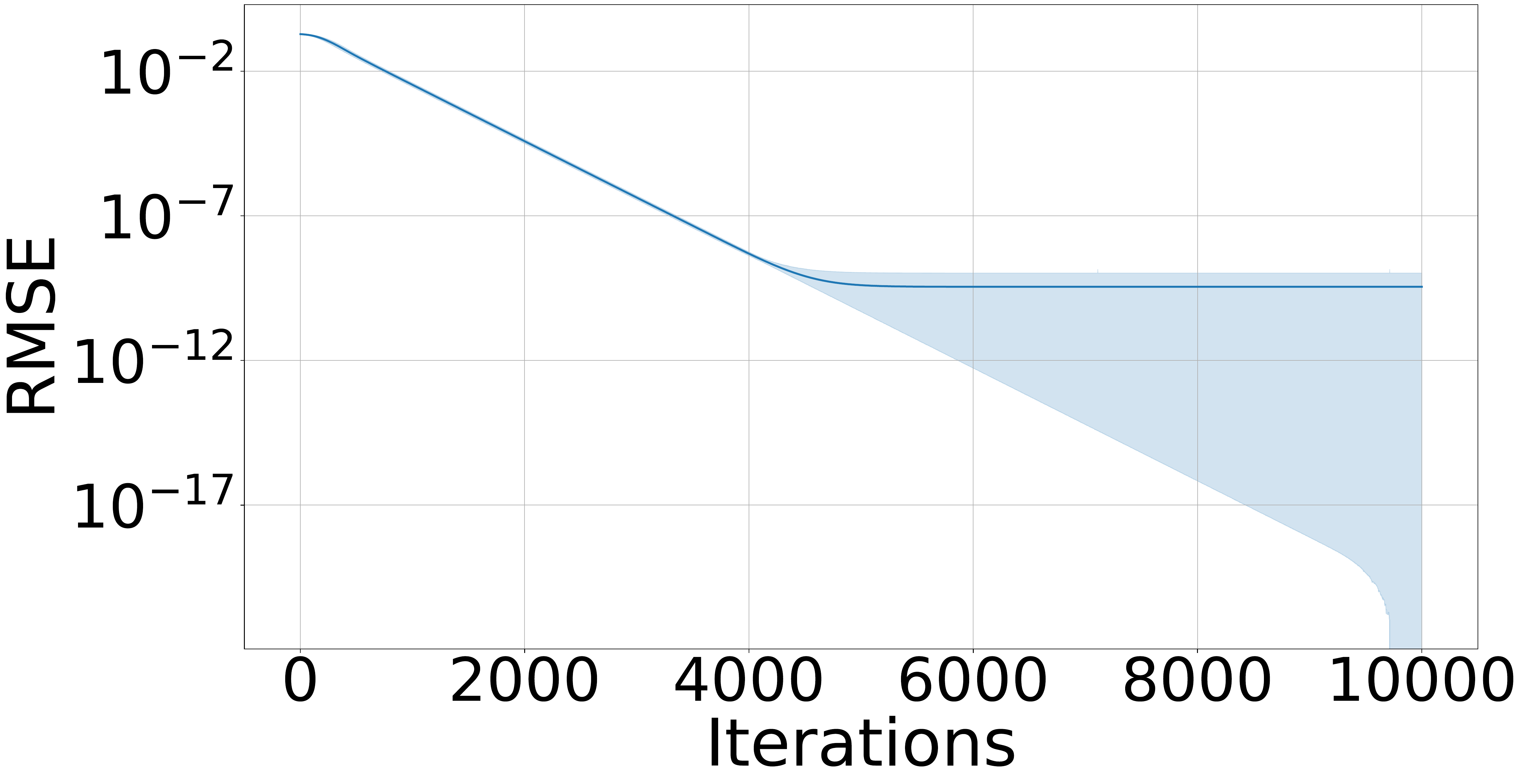}
         \caption{$\sigma=0$.}
         \label{linear_manifold_iter_0_b}
     \end{subfigure}
    \begin{subfigure}[b]{0.325\textwidth}
         \centering
         \includegraphics[width=\textwidth]{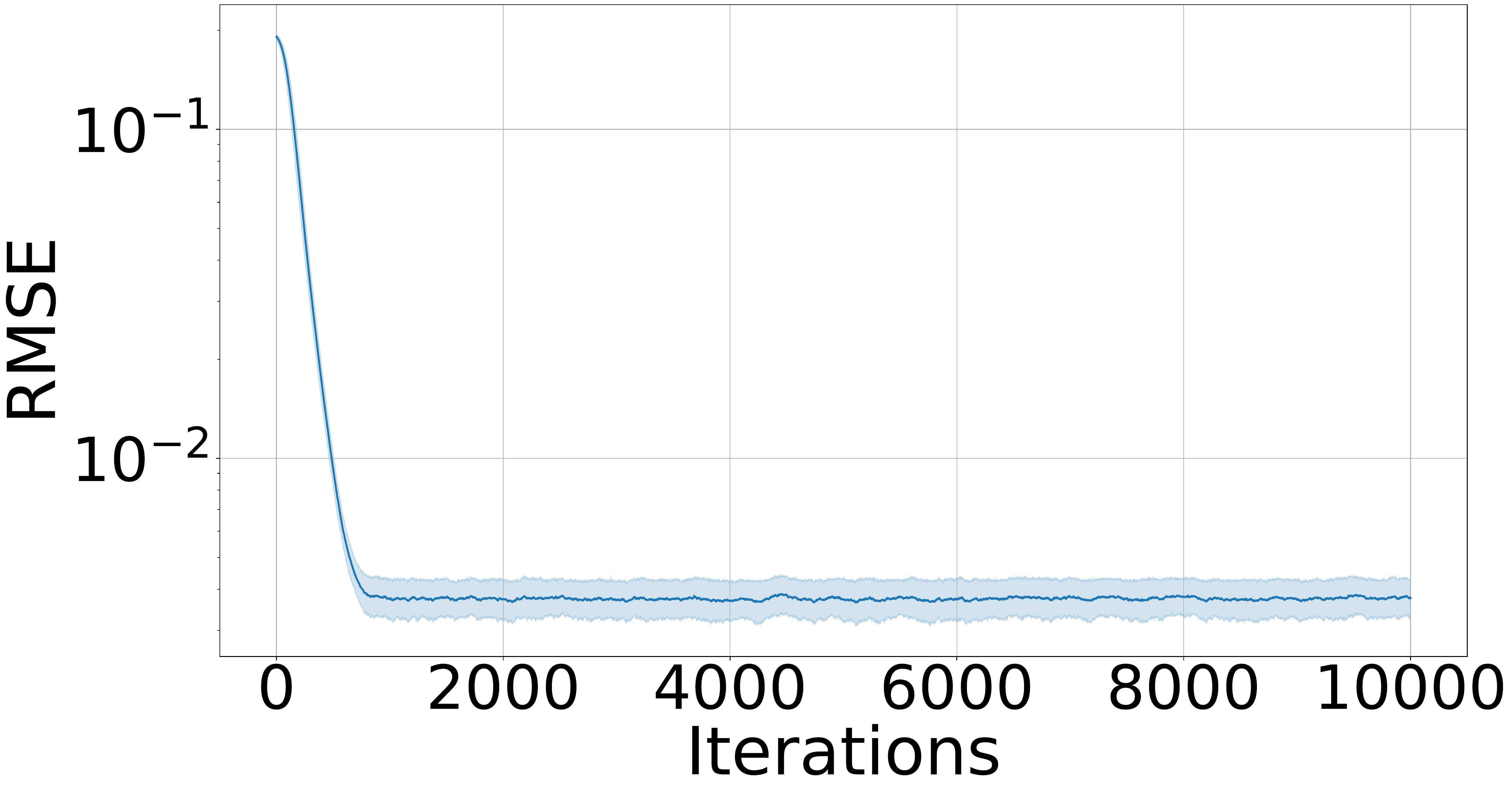}
         \caption{$\sigma=0.3$.}
         \label{linear_manifold_iter_0.3_b}
     \end{subfigure}
     \begin{subfigure}[b]{0.325\textwidth}
         \centering
         \includegraphics[width=\textwidth]{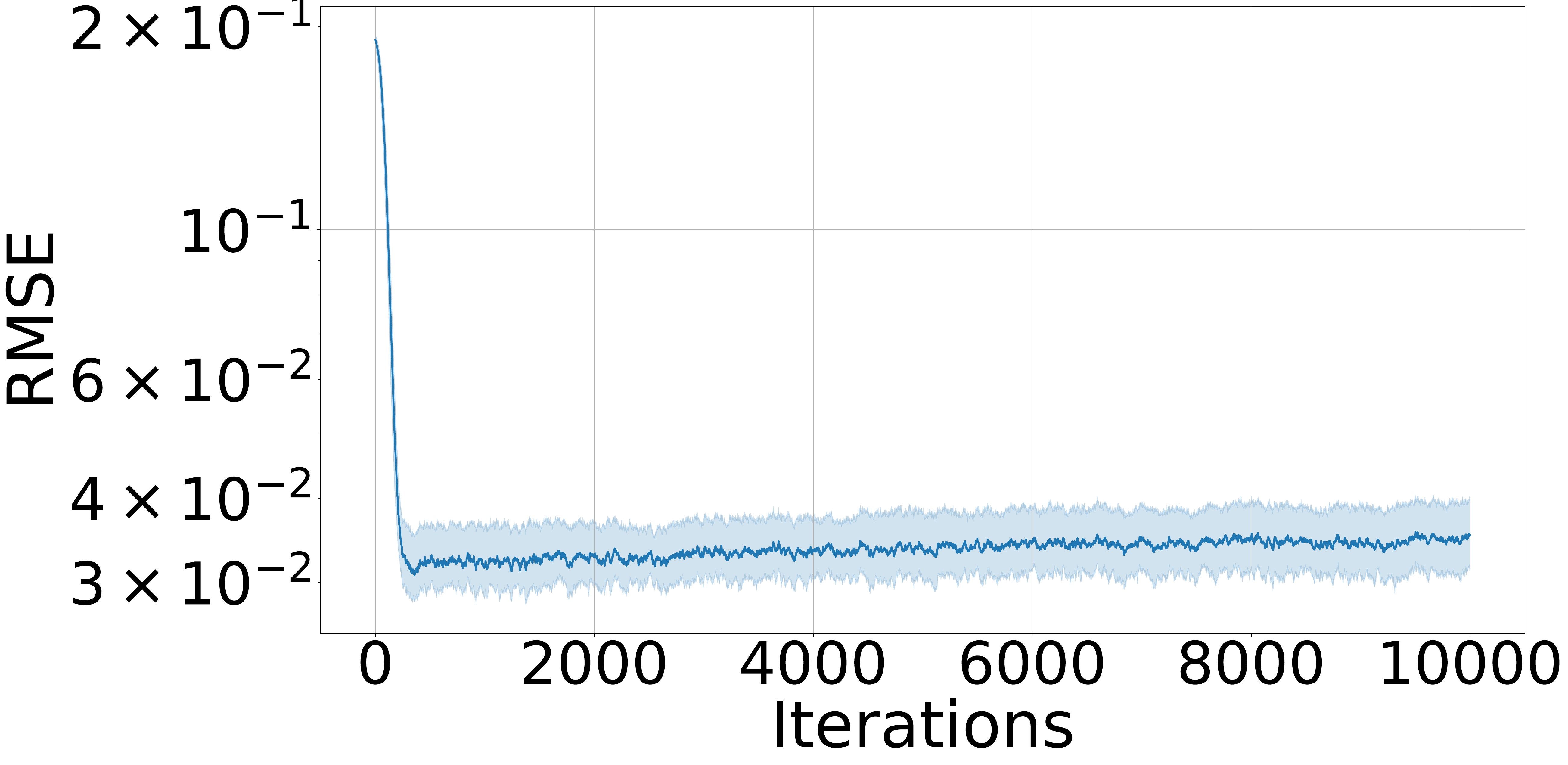}
         \caption{$\sigma=1$.}
         \label{linear_manifold_iter_1.0_b}
     \end{subfigure}
    \caption{Minimal RMSE vs Iterations, Initialization on the manifold in dimension 100.}
    \label{linear_manifold_big}
\end{figure}

\begin{figure}
    \centering
    \begin{subfigure}[b]{0.325\textwidth}
         \centering
         \includegraphics[width=\textwidth]{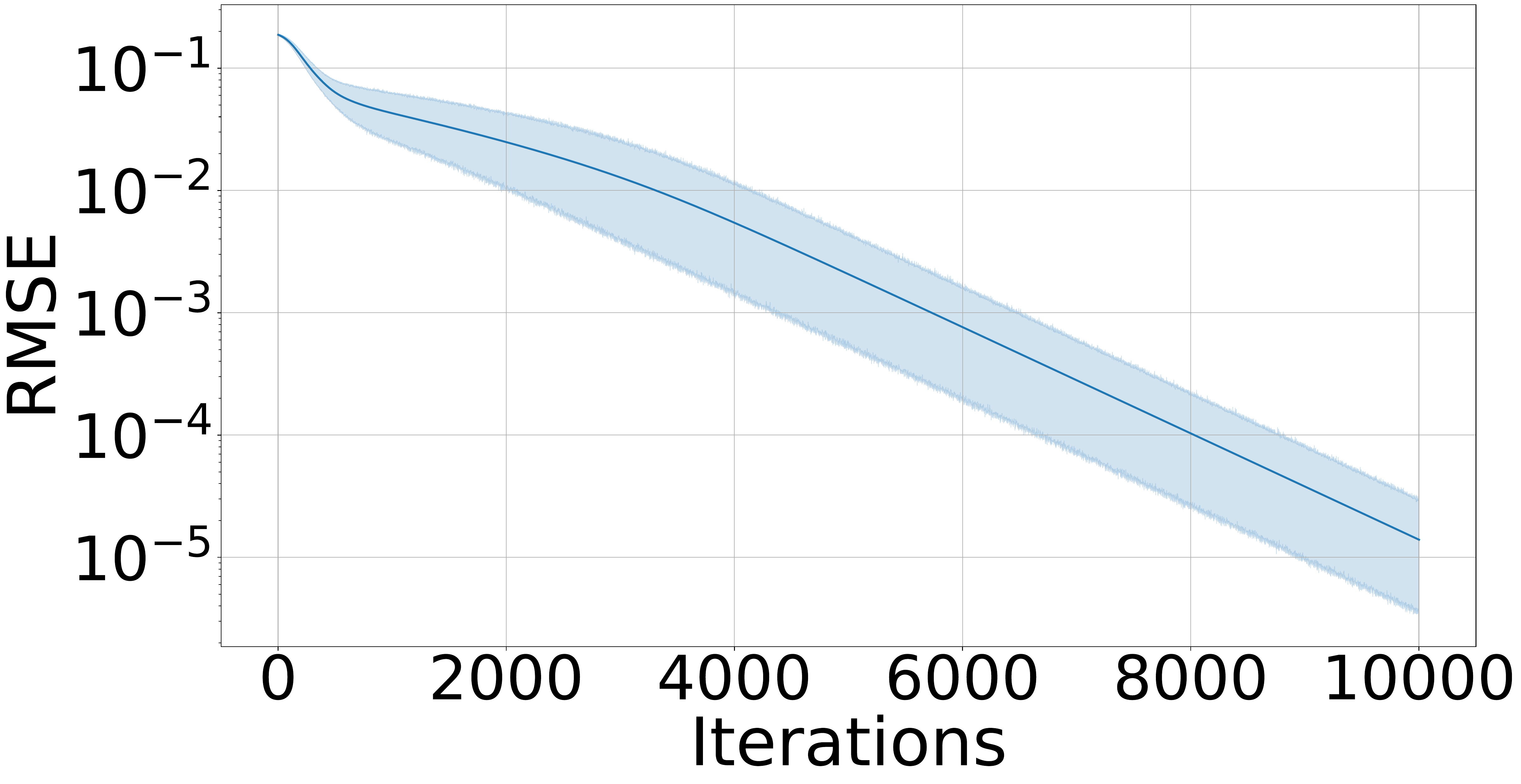}
         \caption{$\sigma=0$.}
         \label{linear_iter_0_b}
     \end{subfigure}
    \begin{subfigure}[b]{0.325\textwidth}
         \centering
         \includegraphics[width=\textwidth]{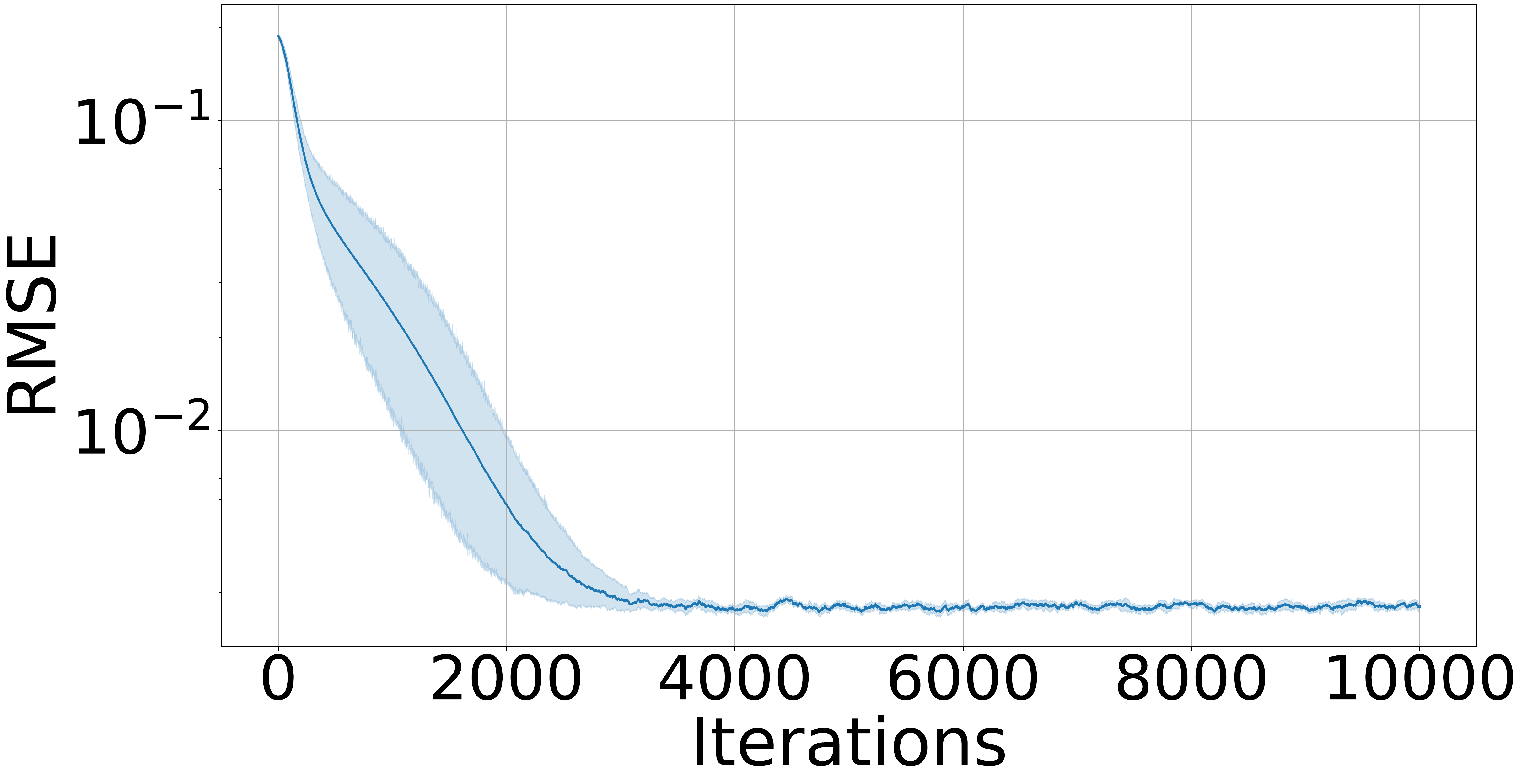}
         \caption{$\sigma=0.3$.}
         \label{linear_iter_0.3_b}
     \end{subfigure}
     \begin{subfigure}[b]{0.325\textwidth}
         \centering
         \includegraphics[width=\textwidth]{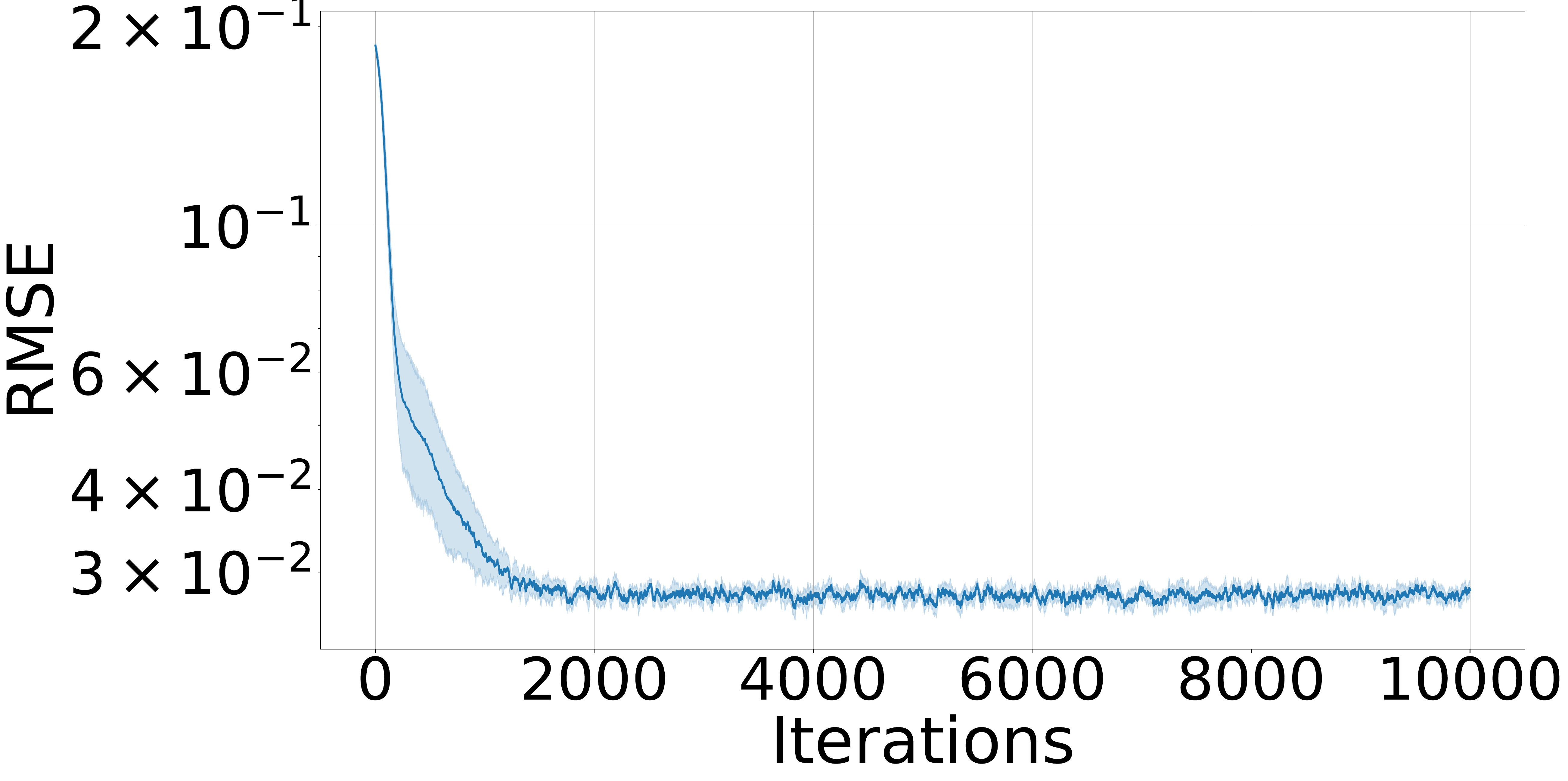}
         \caption{$\sigma=1$}
         \label{linear_iter_1.0_b}
     \end{subfigure}
    
    \caption{Minimal RMSE vs Iterations in dimension 100, Regularization $\rho=0.1$ for $\sigma=0$,  $\rho=0.2$ for $\sigma>0$, Initialization on the unit sphere.}
    \label{linear_iters_big}
\end{figure}

In each experiment, the RMSE is of the order $10^{-2}$, which can be interpreted as, on average per coordinate, the estimators $\hat{\mu}_0,\hat{\mu}_1$ are missing the true parameters $\mu_0^\star,\mu_1^\star$ by $10^{-2}$, suggesting a high level of accuracy in the estimation procedure.

\textbf{Making $d$ vary.}
We repeat the same experiment as before, just varying the dimension of the problem, the two centroids in $\R^d$ are defined by $\mu_0^\star=(\underbrace{0,\ldots,0}_{\text{d-1 times}},1)$ and $\mu_1^\star=(-1,\underbrace{0,\ldots,0}_{\text{d-1 times}})$. For $d$ ranging between $4$ and $200$, we show in Figures \ref{linear_manifold_change} and \ref{linear_iters_change}, on the x-axis the dimension of the problem and on the y-axis the minimal RMSE after $5000$ iterations. We remark that in Figures \ref{linear_manifold_iter_0.3_c} and \ref{linear_manifold_iter_1.0_c}, the y-axis displays only a single line due to the logarithmic scale; this may appear misleading, but it simply indicates that the next power of 10 is much larger and is not captured on the plot.

\begin{figure}
    \centering
    \begin{subfigure}[b]{0.325\textwidth}
         \centering
         \includegraphics[width=1.03\textwidth]{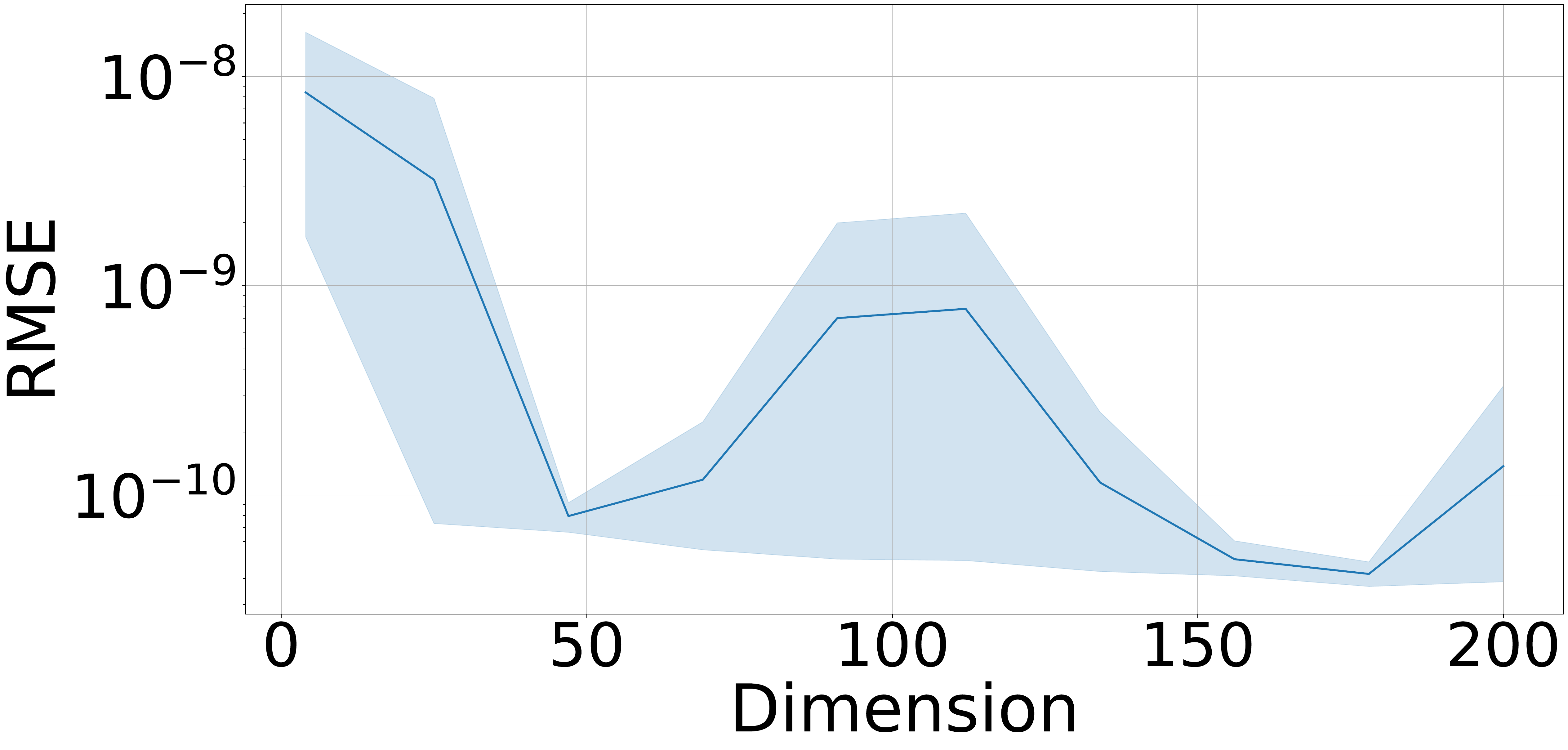}
         \caption{$\sigma=0$.}
         \label{linear_manifold_iter_0_c}
     \end{subfigure}
    \begin{subfigure}[b]{0.32\textwidth}
         \centering
         \includegraphics[width=0.96\textwidth]{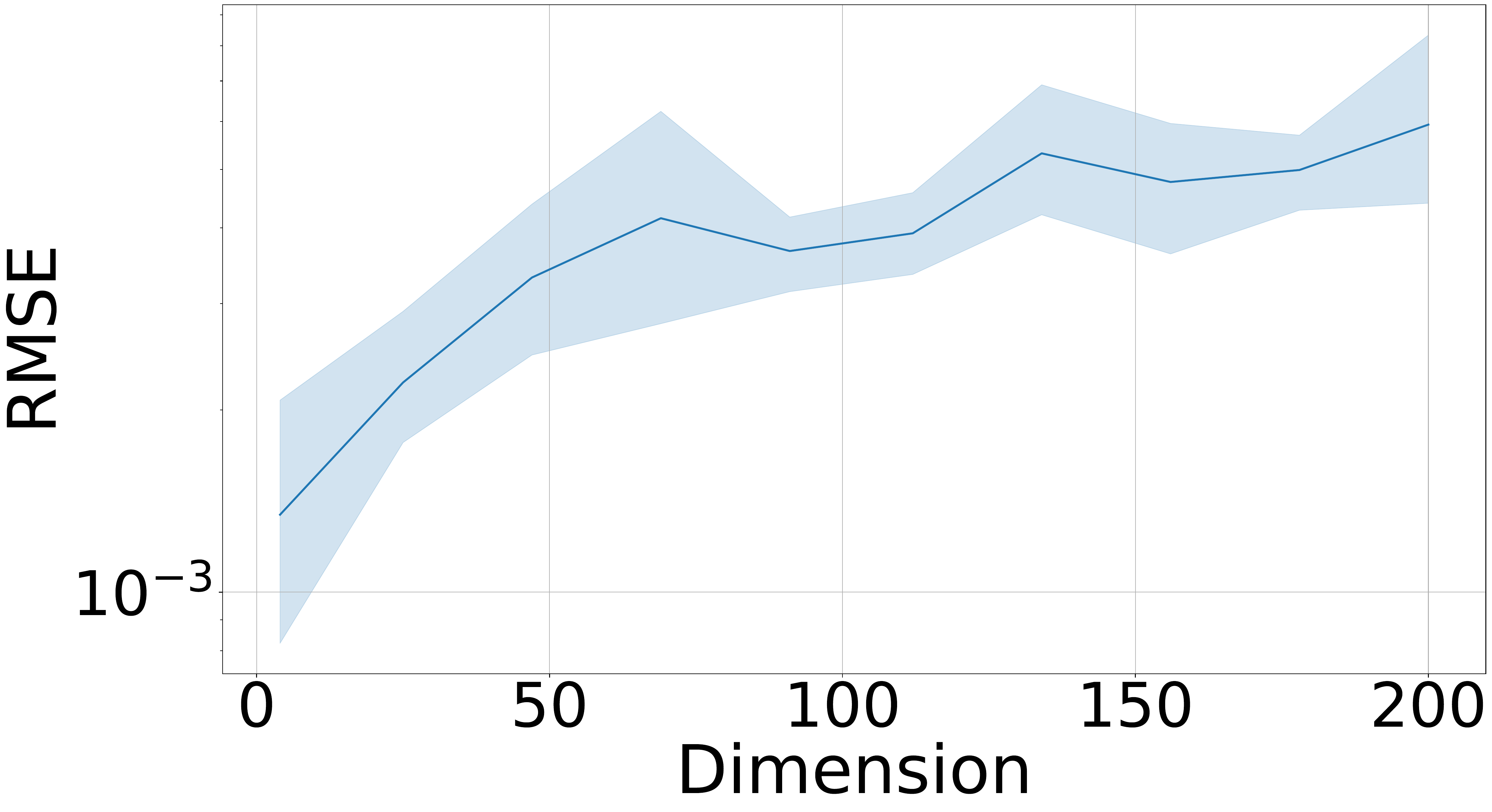}
         \caption{$\sigma=0.3$.}
         \label{linear_manifold_iter_0.3_c}
     \end{subfigure}
     \begin{subfigure}[b]{0.32\textwidth}
         \centering
         \includegraphics[width=0.96\textwidth]{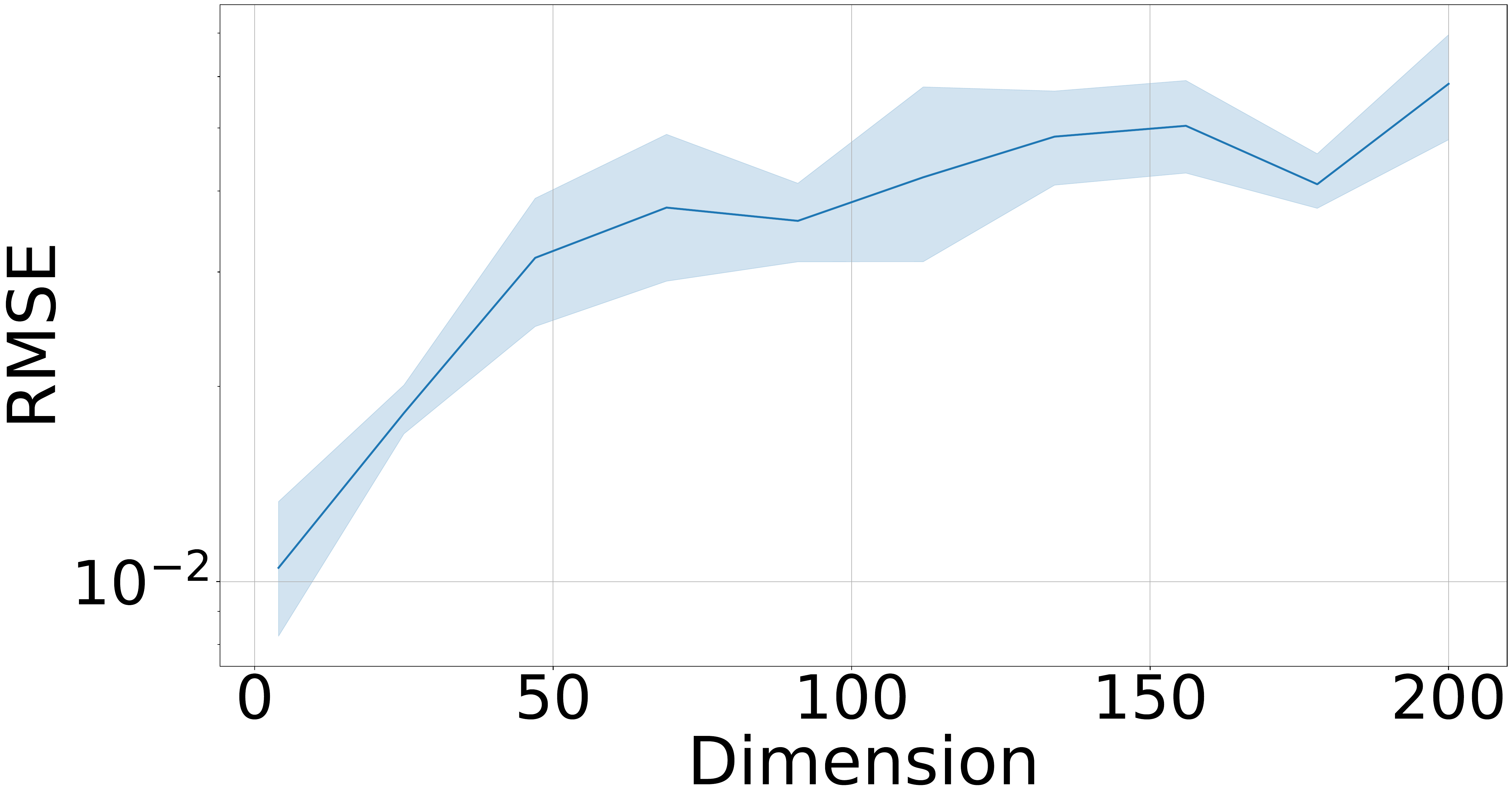}
         \caption{$\sigma=1$.}
         \label{linear_manifold_iter_1.0_c}
     \end{subfigure}
    \caption{Minimal RMSE vs Dimensionality, Initialization on the manifold $\mathcal{M}$.
    }
    \label{linear_manifold_change}
\end{figure}

\begin{figure}[!ht]
    \centering
    \begin{subfigure}[b]{0.325\textwidth}
         \centering
         \includegraphics[width=\textwidth]{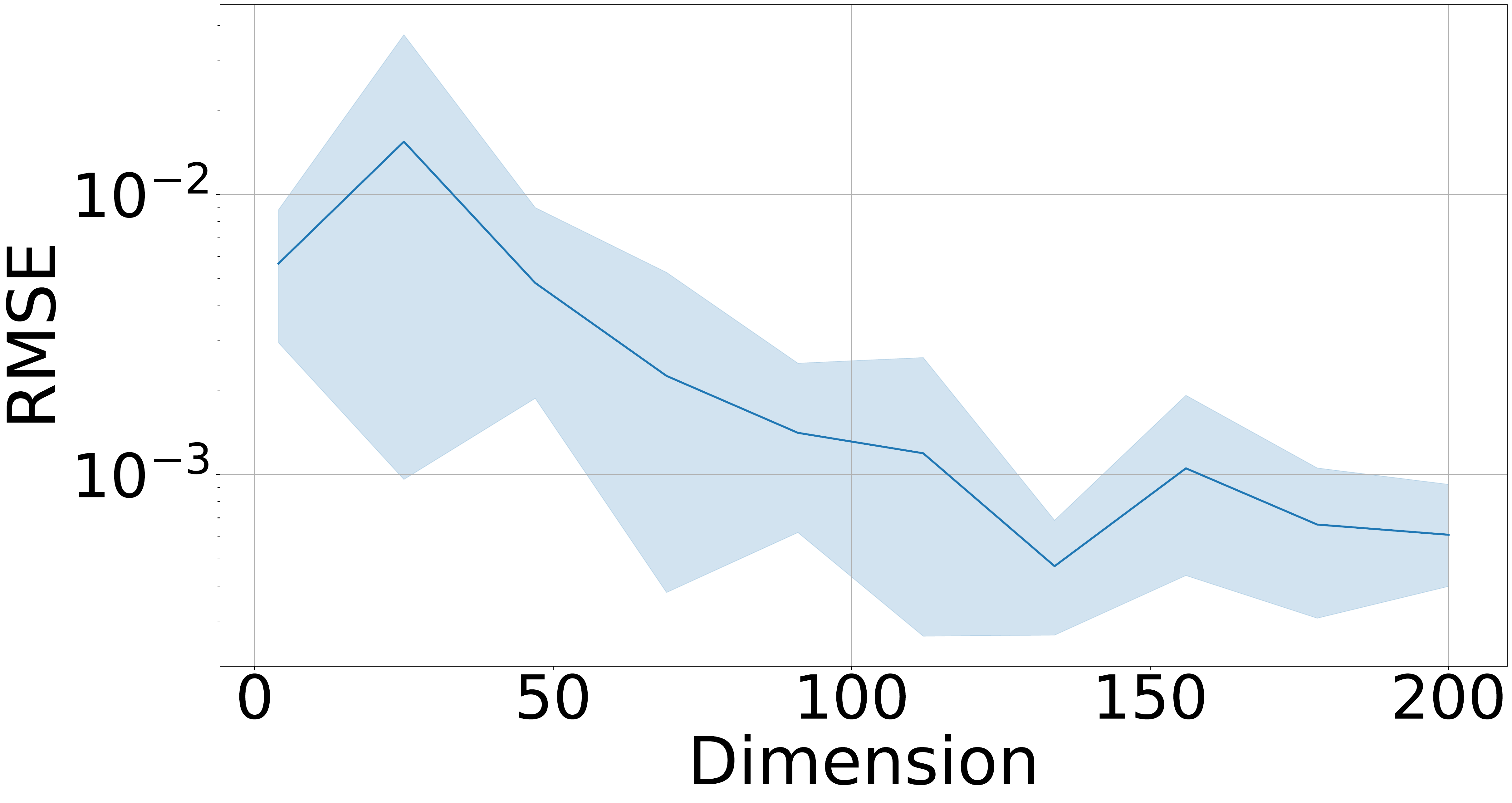}
         \caption{$\sigma=0$.}
         \label{linear_iter_0_c}
     \end{subfigure}
    \begin{subfigure}[b]{0.325\textwidth}
         \centering
         \includegraphics[width=\textwidth]{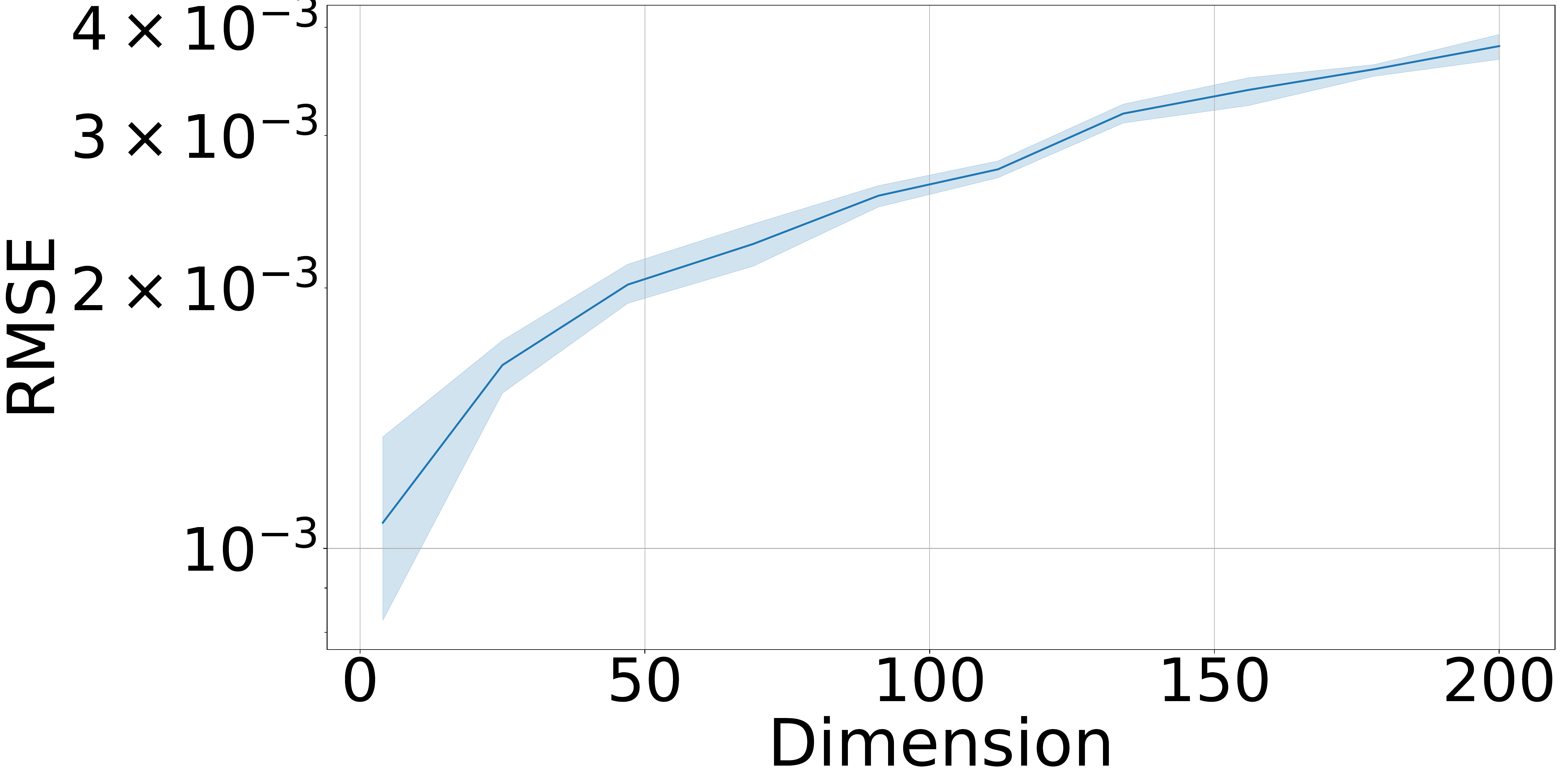}
         \caption{$\sigma=0.3$.}
         \label{linear_iter_0.3_c}
     \end{subfigure}
     \begin{subfigure}[b]{0.325\textwidth}
         \centering
         \includegraphics[width=\textwidth]{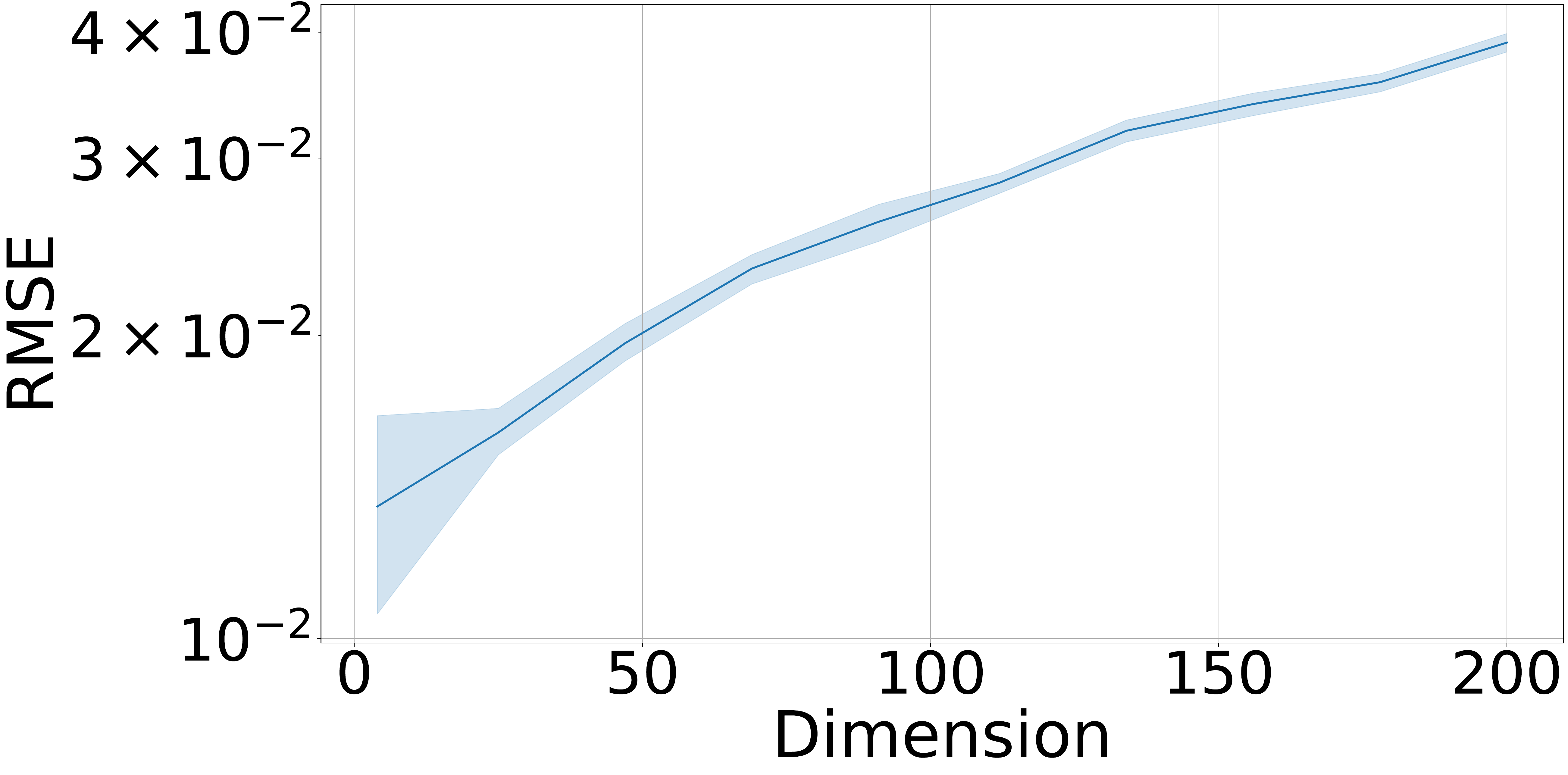}
         \caption{$\sigma=1$.}
         \label{linear_iter_1.0_c}
     \end{subfigure}
    
    \caption{Minimal RMSE vs Dimensionality, Regularization $\rho=0.1$ for $\sigma=0$,  $\rho=0.2$ for $\sigma>0$, Initialization on the unit sphere.}
    \label{linear_iters_change}
\end{figure}


Regardless of the initialization regime, in the noiseless case ($\sigma=0$) the minimal RMSE decreases as the problem dimension $d$ grows. By contrast, for any strictly positive noise level, the minimal RMSE increases slowly with $d$— for $\sigma=0.3$ it remains of order $10^{-3}$, and for $\sigma=1$ of order $10^{-2}$. This reflects the growing difficulty of the problem as both dimensionality and noise increase. We recall that the minimal RMSE can be interpreted as the average discrepancy per coordinate between the estimated parameters and the true centroids, suggesting a high level of accuracy in each experiment.

Regarding running times, the experiments run in Figures \ref{linear_manifold_big} and \ref{linear_iters_big} take approximately two hours on a standard laptop, while that of Figures \ref{linear_manifold_change} and \ref{linear_iters_change} may require up to 12 hours, to cover the dimension grid.

\subsection{Relaxing the orthogonality assumption}
\label{app:relaxing_orhtogonality_xp}
We replicate the experiments and parameter‐selection procedure from Section \ref{app:xp_dim}, but this time initializing the centroids and the initial points uniformly at random on the sphere $\mathbb{S}^{d-1}$ in each run. Figure \ref{fig:norm1centroiddim50} illustrates the algorithm’s convergence behavior over $10000$ iterations in the case $d=50$. In contrast, Figure \ref{fig:norm1centroidchangedim} shows the minimal RMSE of \eqref{PGDrho} after $5000$ iterations for dimensions $d$ ranging from 4 to 100.
\begin{figure}
    \centering
    \includegraphics[width=0.4\linewidth]{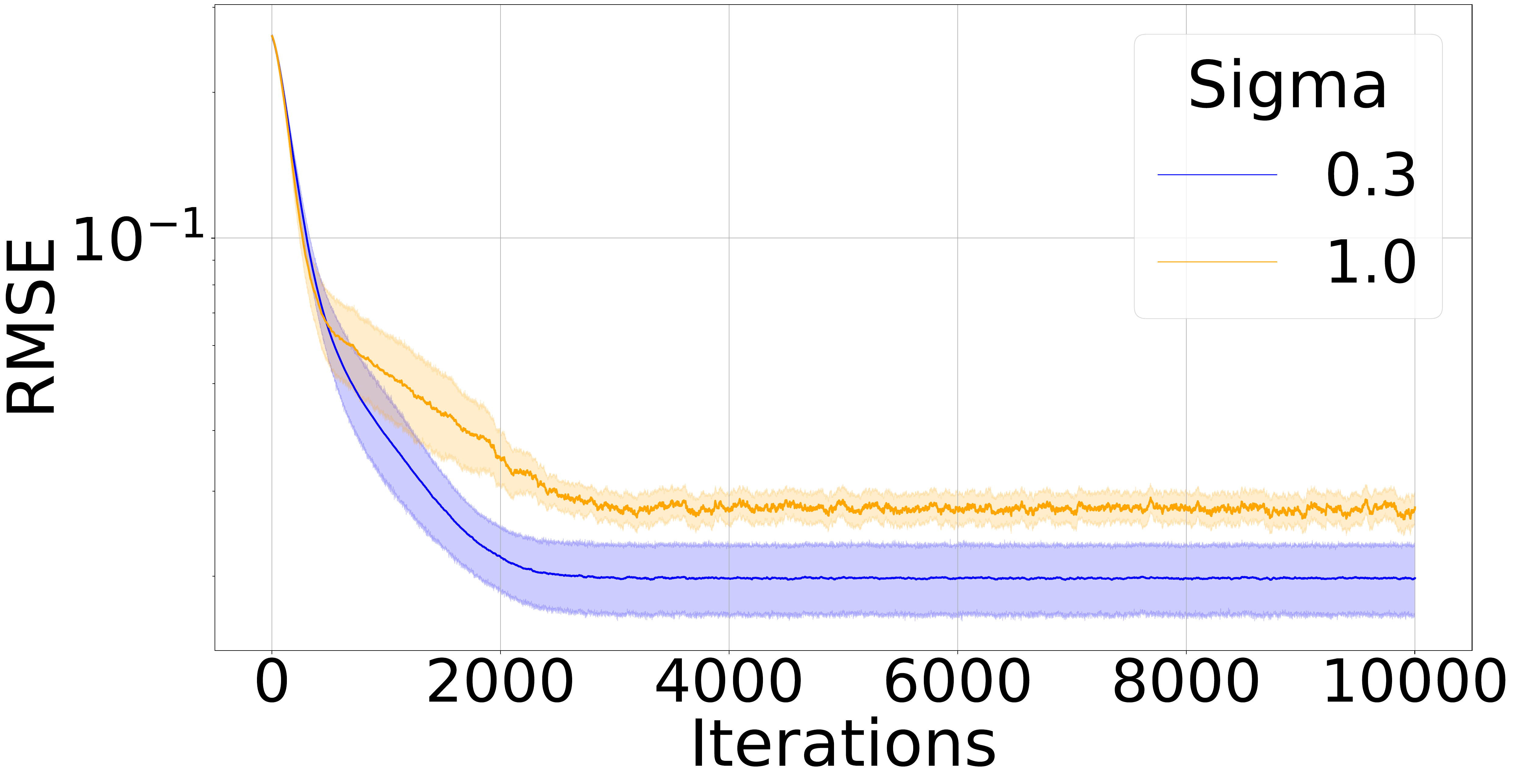}
    \caption{Minimal RMSE vs Iterations in dimension 50, Random initialization on the unit sphere of the centroids and of initial guesses, Regularization $\rho=0.2$, 10 runs, 95\% percentile intervals are plotted.}
    \label{fig:norm1centroiddim50}
\end{figure}
\begin{figure}
    \centering
    \includegraphics[width=0.4\linewidth]{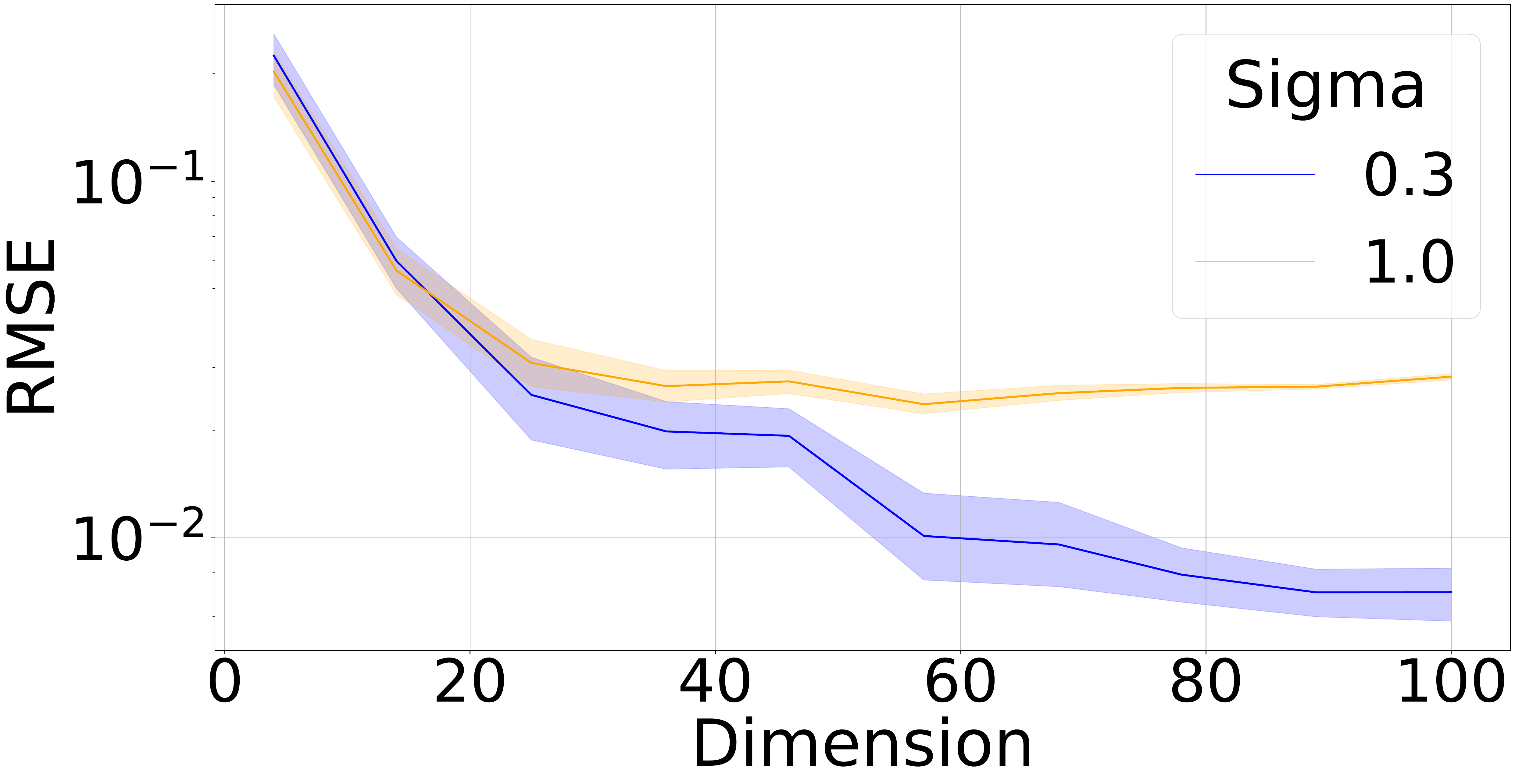}
    \caption{Minimal RMSE vs Dimensionality, Random initialization on the unit sphere of the centroids and of initial guesses, Regularization $\rho=0.2$, 10 runs, 95\% percentile intervals are plotted.}
    \label{fig:norm1centroidchangedim}
\end{figure}

We observed the expected behavior: as the dimension increases, randomly initializing centroids on the sphere makes them more likely to be orthogonal, and thus \clrmodif{training via the regularized theoretical risk} yields better results at higher dimensions. This effect is stronger at lower noise levels and becomes noticeably clearer beyond 40 dimensions.

Regarding running times, the experiments run in Figure \ref{fig:norm1centroiddim50} take approximately one hour on a standard laptop, while those of Figure \ref{fig:norm1centroidchangedim} may require up to 7 hours, to cover the dimension grid.

\subsection{Euclidean Projected SGD} \label{projsgd}
Based on Section \ref{def:psgdlinear} and the algorithm \eqref{PGDrho} presented there, we define the projected Euclidean SGD as the following update rule, with initialization $(\mu_0^0,\mu_1^0)\in(\mathbb{S}^{d-1})^2$, and stepsize $\gamma$: \begin{equation}\label{eq:psgd-euclidean}\tag{$\mathrm{PSGD-Euclidean}$} \begin{split} g_0^k&=\frac{1}{M}\sum_{i=1}^M\nabla_{\mu_0} h(\mu_0^k,\mu_1^k,\xi_i^k),\\ g_1^k&=\frac{1}{M}\sum_{i=1}^M\nabla_{\mu_1} h(\mu_0^k,\mu_1^k,\xi_i^k),\\ \mu_0^{k+1}&=\frac{\mu_0^k-\gamma g_0^k}{\Vert \mu_0^k-\gamma g_0^k\Vert_2},\\ \mu_1^{k+1}&=\frac{\mu_1^k-\gamma g_1^k}{\Vert \mu_1^k-\gamma g_1^k\Vert_2}, \end{split} \end{equation} 
where $h$ is defined in \eqref{hdef}.

Our empirical analysis reveals that it is a viable and simpler alternative to the theoretically grounded Riemannian gradient. This section provides the numerical evidence supporting this choice.  To compare the performance of Euclidean SGD against the Riemannian SGD (both projected on the unit sphere), we repeated the experiments outlined in Section \ref{sec:gmm}. The results are presented in Figure \ref{projsgdplots}. This empirical equivalence confirms that the simple projection of the Euclidean gradient serves as an effective and computationally simpler proxy for the true Riemannian gradient. 
\begin{figure}[ht]
    \centering
    \begin{subfigure}[b]{0.48\textwidth}
         \centering
         \includegraphics[width=\textwidth]{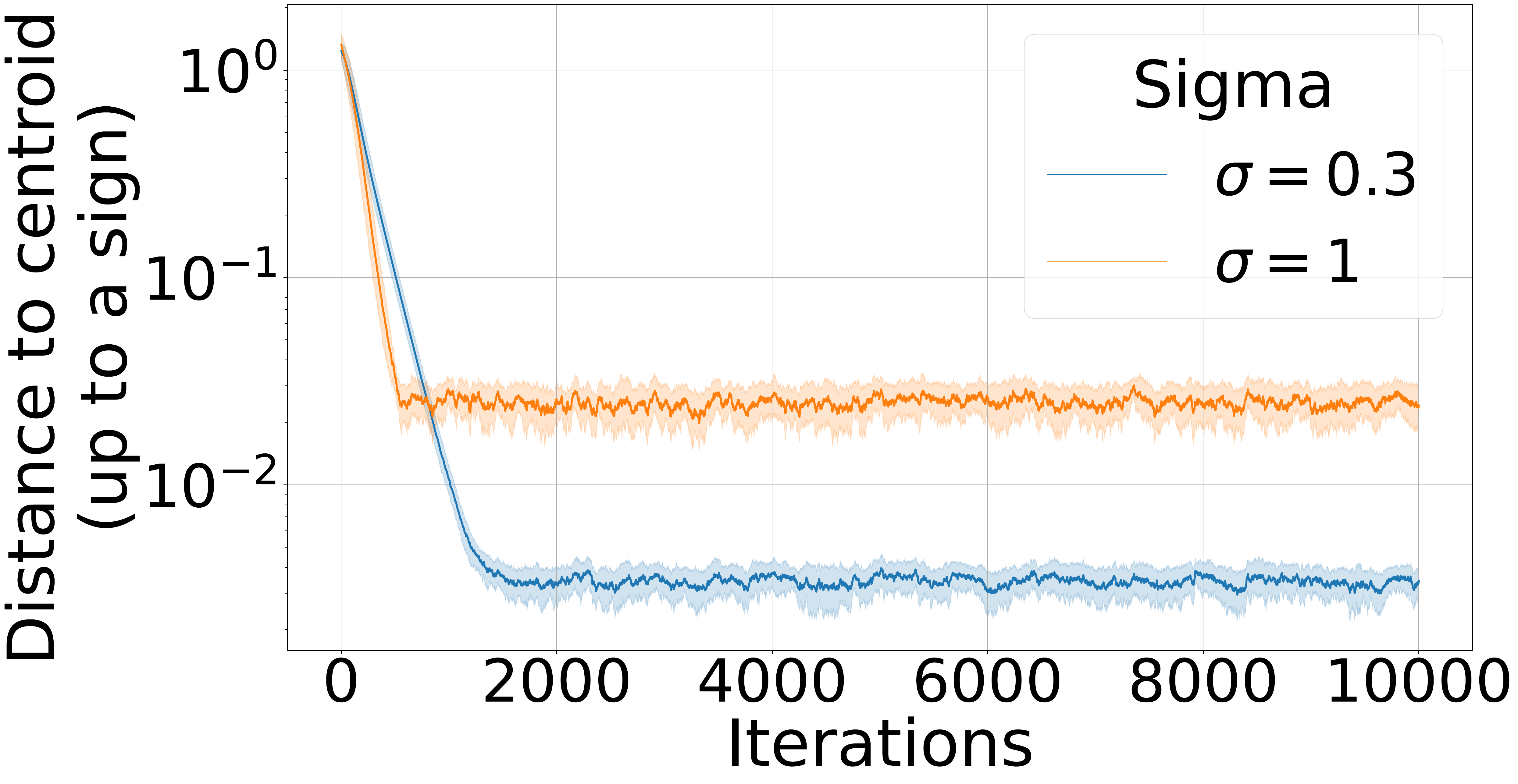}
         \caption{Distance to centroids vs Projected SGD iterations for the minimization of $\mathcal{R}$, initialization on the manifold.}
     \end{subfigure}
     \begin{subfigure}[b]{0.48\textwidth}
         \centering
         \includegraphics[width=\textwidth]{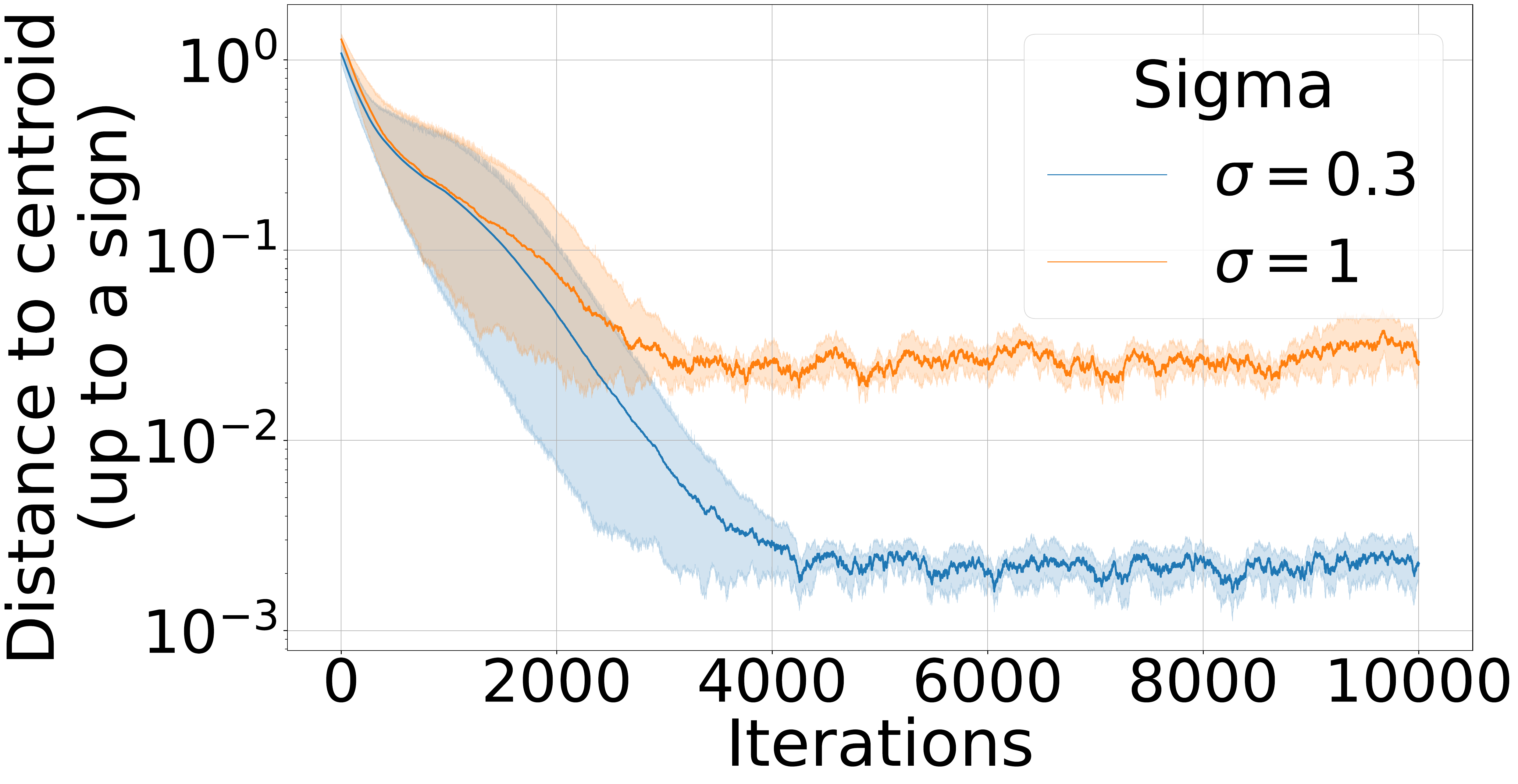}
         \caption{Distance to centroids vs Projected SGD iterations for the minimization of $\mathcal{R}^{\rho}$, initialization on the unit sphere, with regularization $\rho=0.2$. }
     \end{subfigure}
     \caption{Performance of Projected SGD. 10 runs, 95\% percentile intervals are plotted.}
     \label{projsgdplots}
\end{figure}

\subsection{Extension to Gaussian mixture model with 3 components}
\label{app:K=3}
We propose an extension of our work to the case of three orthonormal centroids. We believe that the approach described below would further generalize to the case of $K$ orthonormal centroids with $K<d$. Specifically, we assume that the tokens are i.i.d.\ drawn from the mixture model
\begin{align}
\label{def:3-gaussian_mixture_model}
\tag{$\mathrm{P}_{\sigma}$}
X_\ell &\sim \frac{1}{3}\mathcal{N}(\mu_0^\star,\sigma^2I_d)+\frac{1}{3}\mathcal{N}(\mu_1^\star,\sigma^2I_d)+\frac{1}{3}\mathcal{N}(\mu_2^\star,\sigma^2I_d), 
\end{align}
where $\mu_0^\star,\mu_1^\star, \mu_2^\star$  are orthonormal vectors. It is natural to consider an attention-based predictor composed of three attention heads, parameterized by $\mu_0,\mu_1,\mu_2\in\mathbb{R}^d$,
    \begin{equation}
    T^{{\rm lin}, \mu_0, \mu_1,\mu_2}(\mathbb{X}) = H^{\mathrm{lin},\mu_0}(\mathbb{X}) + H^{\mathrm{lin},\mu_1}(\mathbb{X})+H^{\mathrm{lin},\mu_2}(\mathbb{X}).     
    \end{equation}
The associated risk is $\mathcal{R}(\mu_0,\mu_1,\mu_2)=\mathbb{E}[\Vert X_1-T^{{\rm lin}, \mu_0, \mu_1,\mu_2}(\mathbb{X})_1\Vert_2^2]$. 
There are two natural generalizations of the regularization term to this case: 
\begin{align*}
r^{(1)}(\mu_0,\mu_1,\mu_2)&= \sum_{0\leq i<j\leq 2}\langle\mu_i,X_1\rangle^2\langle\mu_j,X_1\rangle^2,\\
r^{(2)}(\mu_0,\mu_1,\mu_2)&= \prod_{i=0}^2\langle\mu_i,X_1\rangle^2.
\end{align*}
The first one promotes pairwise orthogonality while the second one promotes mutual orthogonality.

We use input sequences of length $L=30$ in $\R^6$, where we define the three centroids $\mu_0^\star=(1,0,0,0,0,0), \mu_1^\star=(0,0,0,1,0,0), \mu_2^\star=(0,0,0,0,0,1)$. The model is trained with an online batch sampling strategy (similar to \ref{PGDrho}, changing the data distribution and the regularization term), with a batch size of 256, and a learning rate of 0.01. We take $\lambda=0.6$ for $\sigma=0.3$, and $\lambda=0.2$ for $\sigma=1$. Since any parameter could learn any centroid up to a sign, we introduce the following distance to the centroid (up to a sign): $$\min_{\pi\in S_3}\min_{s\in \{-1,1\}^3}\sqrt{\sum_{i=0}^2 \Vert \hat{\mu}_{\pi(i)}-s_i\mu_i^\star\Vert^2},$$
where $S_3$ is the symmetric group of order 3 and $\hat{\mu}_0,\hat{\mu}_1,\hat{\mu}_2$ are the parameters returned by the algorithm. We present the results in Figure \ref{mixture3}. We observe that the regularization $r^{(1)}$ outperforms $r^{(2)}$, since it explicitly includes all pairwise terms to enforce orthogonality. However, we note that the number of regularization terms grows quadratically with the number of centroids.

\begin{figure}
    \centering
    \begin{subfigure}[b]{0.49\textwidth}
         \centering
         \includegraphics[width=\textwidth]{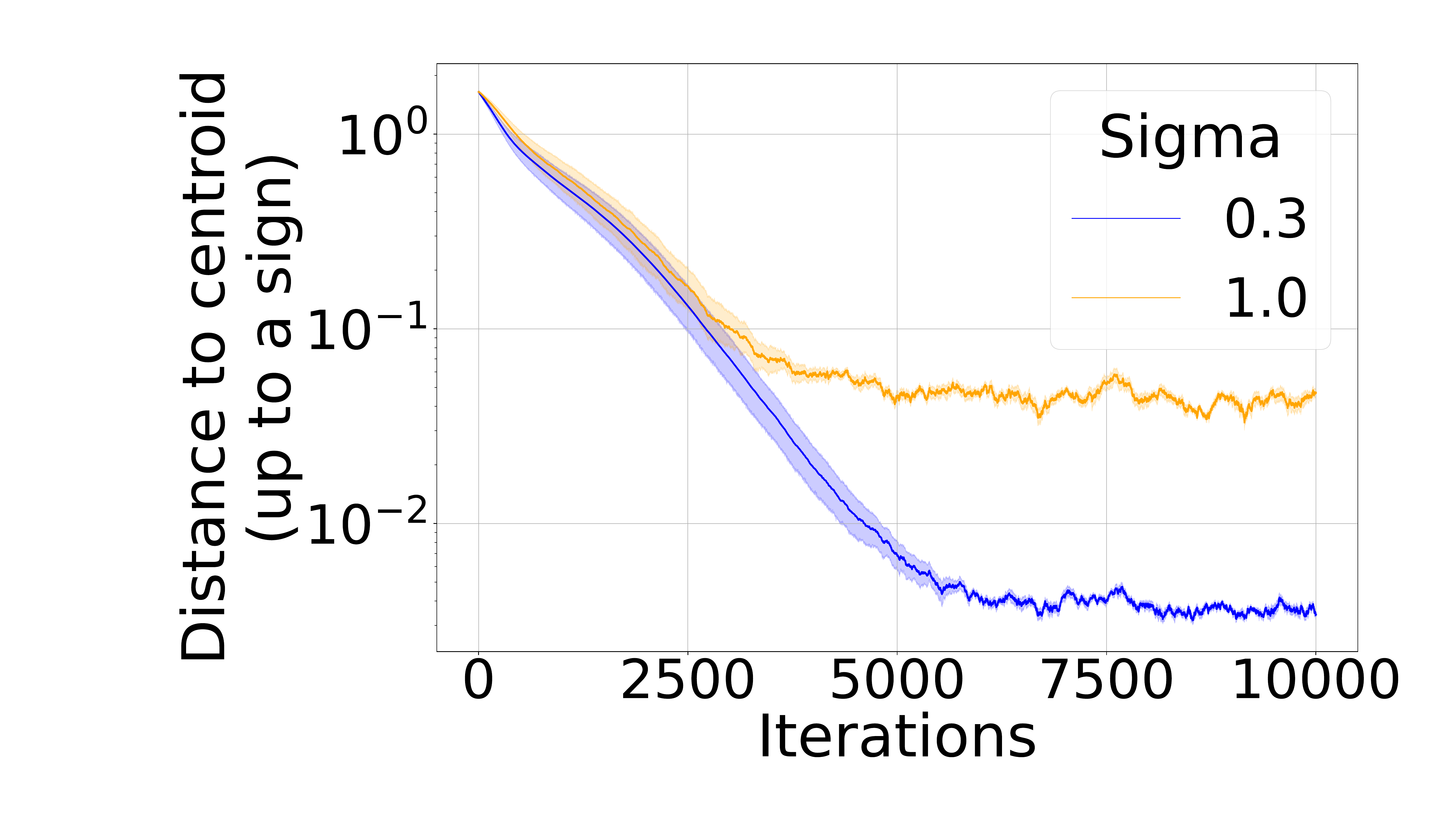}
         \caption{With regularization $r^{(1)}$.}
         \label{mixture3_reg}
     \end{subfigure}
    \begin{subfigure}[b]{0.49\textwidth}
         \centering
         \includegraphics[width=\textwidth]{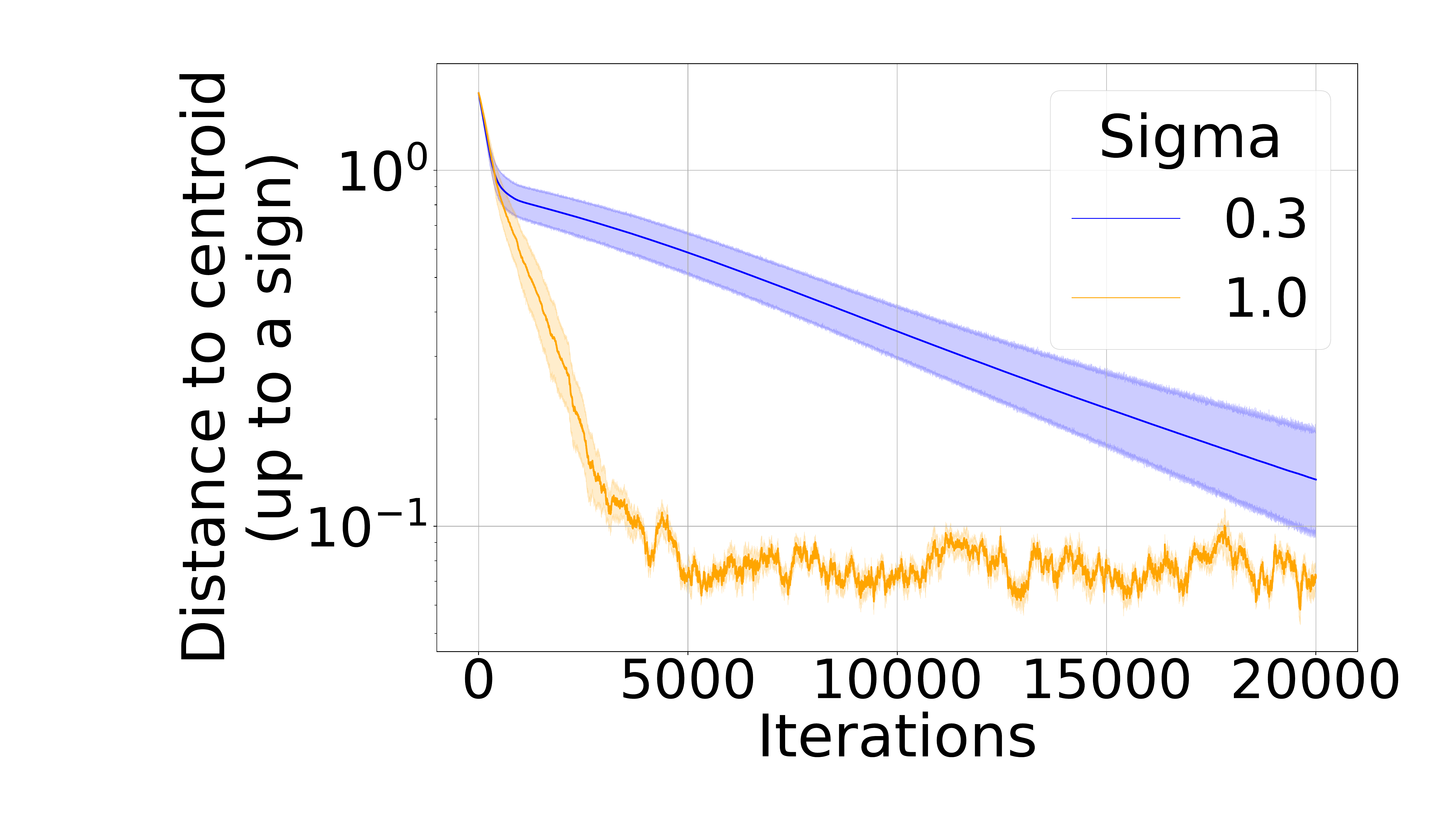}
         \caption{With regularization $r^{(2)}$.}
         \label{mixture3_20000}
     \end{subfigure}
    
    \caption{Distance to centroids vs number of iterations, with regularization strength $\rho=0.2$, initialization on the unit sphere. 10 runs, 95\% percentile intervals are plotted.}
    \label{mixture3}
\end{figure}

Regarding running times, the experiments run in Figure \ref{mixture3} take approximately 15 minutes on a standard laptop.

\color{black}
\section{Softmax attention layers and clustering}\label{sec:softmax}

In this section, we assess the abilities of attention-based predictors involving a softmax activation in a clustering context.

\subsection{Problem setting}

\paragraph{An attention-based learner with softmax activation.} We recall that an attention head made of a self-attention layer can be written as follows: 
 $$
 H^{\mathrm{soft_\lambda}}(\mathbb{X}) = \mathrm{softmax}_\lambda \left(  \mathbb{X}QK^\top \mathbb{X}^\top \right) \mathbb{X} V
 $$
where the softmax of temperature $\lambda >0$ is applied row-wise, and the matrices $K,Q,V \in \mathbb{R}^{d\times p}$ are usually referred to as keys, queries and values. 
As in Section \ref{sec:starter}, we assume that the values are taken as identity, meaning that the attention head simply outputs combinations of the initial tokens weighted by attention scores. Furthermore, we assume that the key and query matrices are equal to the same row matrix $\mu^\top\in \mathbb{R}^{1\times d}$, we obtain 
 \begin{align}
 H^{\mathrm{soft_{\lambda}},\mu}(\mathbb{X}) = \mathrm{softmax}_\lambda \left(  \mathbb{X}\mu \mu^\top \mathbb{X}^\top \right) \mathbb{X} .
 \end{align}
With such an architecture, the $\ell$-th output vector is therefore given by
\begin{equation}
H^{\mathrm{soft_{\lambda}},\mu}(\mathbb{X})_\ell = \sum_{k=1}^L \mathrm{softmax}_\lambda \left(  X_\ell^\top  \mu \mu^\top \mathbb{X}^\top \right)_k X_k ,
\end{equation}
which corresponds to aggregating the $X_k$'s when $X_k$ and $X_\ell$ are simultaneously aligned with $\mu$.
This head should be a good candidate to estimate a centroid of a mixture model. 
In the case where the mixture involves two components, one could train two attention heads:
\begin{align}
    (\hat{\mu}_0,\hat{\mu}_1) \in \mathrm{argmin}_{\mu_0 , \mu_1 \in \mathbb{S}^{d-1}} \mathcal{R}^{\mathrm{soft}}(\mu_0,\mu_1),
\end{align}
where 
\begin{align}
\begin{split}
    \mathcal{R}^{\mathrm{soft}}(\mu_0,\mu_1) &= \frac{1}{L}\mathbb{E}\left[ \left\| \mathbb{X} - (H^{\mathrm{soft_{\lambda}},\mu_0} + H^{\mathrm{soft_{\lambda}},\mu_1})(\mathbb{X}) \right\|_F^2 \right] \\
    &=\frac{1}{L} \mathbb{E}\left[  \sum_{\ell=1}^L \left\| X_\ell  - (H^{\mathrm{soft_{\lambda}},\mu_0} + H^{\mathrm{soft_{\lambda}},\mu_1})(\mathbb{X})_\ell \right\|_2^2 \right].
    \end{split}
\end{align}

\begin{remark}[The attention heads are biased]
\label{rem:attention_bias}
As the tokens $(X_\ell)_\ell$ are independent, we have that $$\mathcal{R}^{\mathrm{soft}}(\mu_0,\mu_1)=\mathbb{E}[\Vert X_1-(H^{\mathrm{soft_\lambda},\mu_0}+H^{\mathrm{soft}_\lambda,\mu_1})(\mathbb{X})_1\Vert_2^2].$$
We note that $H^{\mathrm{soft,\mu}}(\mathbb{X})_1=\mathrm{softmax}_{\lambda}(\langle X_1,\mu\rangle v)\mathbb{X}$, where the vector $v$ is $L$-dimensional with components $v_\ell=\langle X_\ell,\mu\rangle$. In the idealized case where $\sigma^2=0$, then each token $X_\ell$ is sampled according to a mixture of Dirac masses given by $\frac{1}{2}\delta_{\mu_0^\star}+\frac{1}{2}\delta_{\mu_1^\star}$. Therefore, if we evaluate  $H^{\mathrm{soft_{\lambda}},\mu}(\mathbb{X})_1$ on $\mu=\mu_0^\star$, we observe the following:
\begin{itemize}
    \item Conditionally to $X_1=\mu_1^\star$, then $$H^{\mathrm{soft_\lambda},\mu_0^\star}(\mathbb{X})_1=\left(\frac{1}{L},\ldots,\frac{1}{L}\right)\mathbb{X}.$$
    This implies that even in a completely misaligned set-up (i.e., $\mu=\mu_0^\star, X_1=\mu_1^\star$), the proposed attention head will return, as the transformation of the first token $X_1$, the empirical mean of the sequence tokens. This highlight the bias introduced by such an attention head, which should be handled through the use of centering techniques.
    \item Conditionally to $X_1=\mu_0^\star$, then $$
    H^{\mathrm{soft_\lambda},\mu_0^\star}(\mathbb{X})_1=\mathrm{softmax}_{\lambda}(v)\mathbb{X}\approx \frac{\exp(\lambda)\mu_0^\star+\mu_1^\star}{\exp(\lambda)+1}.
    $$ 
    This suggests that in the perfectly aligned case (i.e., $\mu=\mu_0^\star, X_1=\mu_0^\star$), selecting a sufficiently large softmax temperature $\lambda$ will cause the model to assign negligible weight to the misaligned components --a desirable property.
\end{itemize}
\end{remark}

\paragraph{Debiasing and disentangling heads.} To handle the bias introduced by the attention heads, discussed in Remark \ref{rem:attention_bias}, we propose to consider centered heads instead, leading to the following modified version of the risk
\begin{align}
    \mathcal{R}^{\mathrm{soft}}(\mu_0,\mu_1,\lambda, \psi) 
    &= \frac{1}{L}\mathbb{E}\left[  \sum_{\ell=1}^L \left\| X_\ell  -   (  H^{\mathrm{soft_{\lambda}},\mu_0} + H^{\mathrm{soft_{\lambda}},\mu_1})(\mathbb{X})_\ell + \frac{\psi}{L} \sum_{k=1}^L X_k \right\|_2^2 \right].
\end{align}
Considering such a risk is equivalent to using heads where a term proportional to $\frac{1}{L} \sum_{k=1}^L X_k$ is substracted.  This type of head is known as shaped attention \citep{noci2023shaped,simplifying}. 
For instance, initializing $\psi=2$ debiases both attention heads independently, without considering their interaction.
Using heads with oracle parameters, one would expect that a single head provides all the necessary information, making it sufficient to debias only that head (i.e. $\psi = 1$). In that case, one should obtain: 
\begin{align}
    \mathcal{R}^{\mathrm{soft}}(\mu_0^\star,\mu_1^\star,\lambda^{\star}, 1) 
    \approx \min\mathcal{R}^{\mathrm{soft}}.
\end{align}

However, when using non-oracle parameters $\mu_0$ and $\mu_1$ within the debiased heads, the risk function may admit global minima where the heads align with zero, one, or both centroids, which is undesirable for the clustering purpose. Therefore, we must enforce a constraint ensuring that each head aligns with exactly one centroid. To achieve this, we introduce the regularization term:
$$
r_0(\mu_0,\mu_1)\eqdef\mathbb{E}[(\langle \mu_0,X_1\rangle-1)^2(\langle \mu_1,X_1\rangle-1)^2]+\langle \mu_0,\mu_1\rangle,
$$
leading to the following regularized optimization problem
\begin{equation}\label{risk_softmax_regularized}\tag{$\mathcal{P}_{{\rho_0}}$}
\begin{aligned}
    \min_{\mu_0,\mu_1\in\mathbb{S}^{d-1}}\mathcal{R}^{\mathrm{soft},\rho_0}(\mu_0,\mu_1,\lambda, \psi) 
    \eqdef\mathcal{R}^{\mathrm{soft}}(\mu_0,\mu_1,\lambda,\psi)+\rho_0 r_0(\mu_0,\mu_1),
    \end{aligned}
\end{equation}
where $\rho_0>0$. 


\subsection{Numerical experiments}\label{app:xp_softmax}

We run Projected Stochastic Gradient Descent (see Appendix \ref{psgdsoft}) to learn the centroids $\mu_0^\star$ and $\mu_1^\star$ as well as the weights $\psi$ and $\lambda$.

In this experiment, we use input sequences of length $L=30$ of 5-dimensional tokens ($d=5$), drawn from a 2-component Gaussian mixture of centroids $\mu_0^\star=(0,0,0,0,1)$ and $\mu_1^\star=(-1,0,0,0,0)$. The variance of each component is either set to $\sigma=0.3$ (low interference) or to $\sigma=1$ (high interference).

 The model based on two softmax attention heads parameterized by $\mu_0$ and $\mu_1$ is trained using \eqref{PGDtheta} with an online batch sampling strategy, with a batch size of 256, a learning rate of $\gamma=0.01$, and running for a total of 3000 iterations. Additionally, we initialize with $\lambda$ set to 3 and a centering value $\psi$ of 2. Here we use the metric \emph{distance to the centroids}, given by
\begin{align}
&\sqrt{\min \left\{ \mathrm{dist}_1, \mathrm{dist}_2 \right\}},
\label{def:xp_metric3}
\end{align}
where $$\mathrm{dist}_1=\Vert\hat{\mu}_0 - \mu_0^\star\Vert^2 + \Vert\hat{\mu}_1 - \mu_1^\star\Vert^2,$$
$$\mathrm{dist}_2=\Vert\hat{\mu}_0 - \mu_1^\star\Vert^2 + \Vert\hat{\mu}_1 - \mu_0^\star\Vert^2,$$
and $\mu_0^\star, \mu_1^\star$ denote the true centroids, respectively, while $\hat{\mu}_0, \hat{\mu}_1$ are the parameters returned by \eqref{PGDtheta}. We remark that this distance is finer than the one defined in~\ref{def:xp_metric2}, as it does not disregard sign flips. The results are visualized in Figure \ref{softmax_reg_noise}, we observe that a regularization term substantially improves the accuracy of the recovered solutions. However, as the strength of the regularization increases, it gradually overrides the original objective and impairs the alignment of the head parameters with the true centroids —an 
effect that becomes more pronounced at higher noise level, an effect also noticed in Section \ref{sec:gmm}.
\begin{figure}[ht]
    \centering
    \begin{subfigure}[b]{0.48\textwidth}
        \centering
         \includegraphics[width=\textwidth]{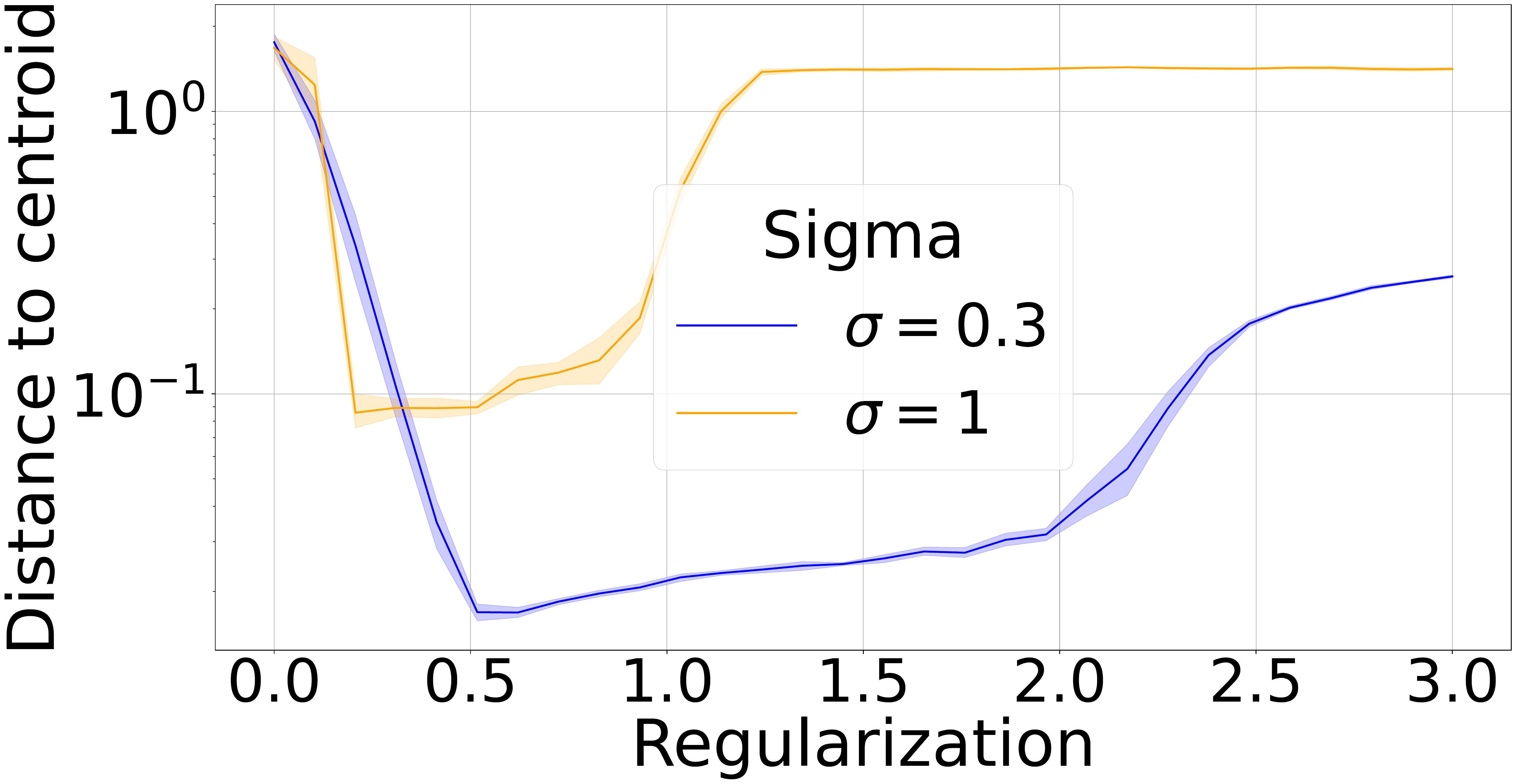}
         \caption{Distance to the centroids after 5000 PSGD iterations vs regularization strength $\rho$ for the minimization of $\mathcal{R}^{\mathrm{soft},\rho_0}$, with an initialization on the unit sphere, with data drawn from the non-degenerate case $\eqref{def:gaussian_mixture_model}$. 10 runs, 95\% percentile intervals are plotted}
         \label{softmax_reg_noise}
    \end{subfigure}
    \begin{subfigure}[b]{0.48\textwidth}
        \centering
        \includegraphics[width=\textwidth]{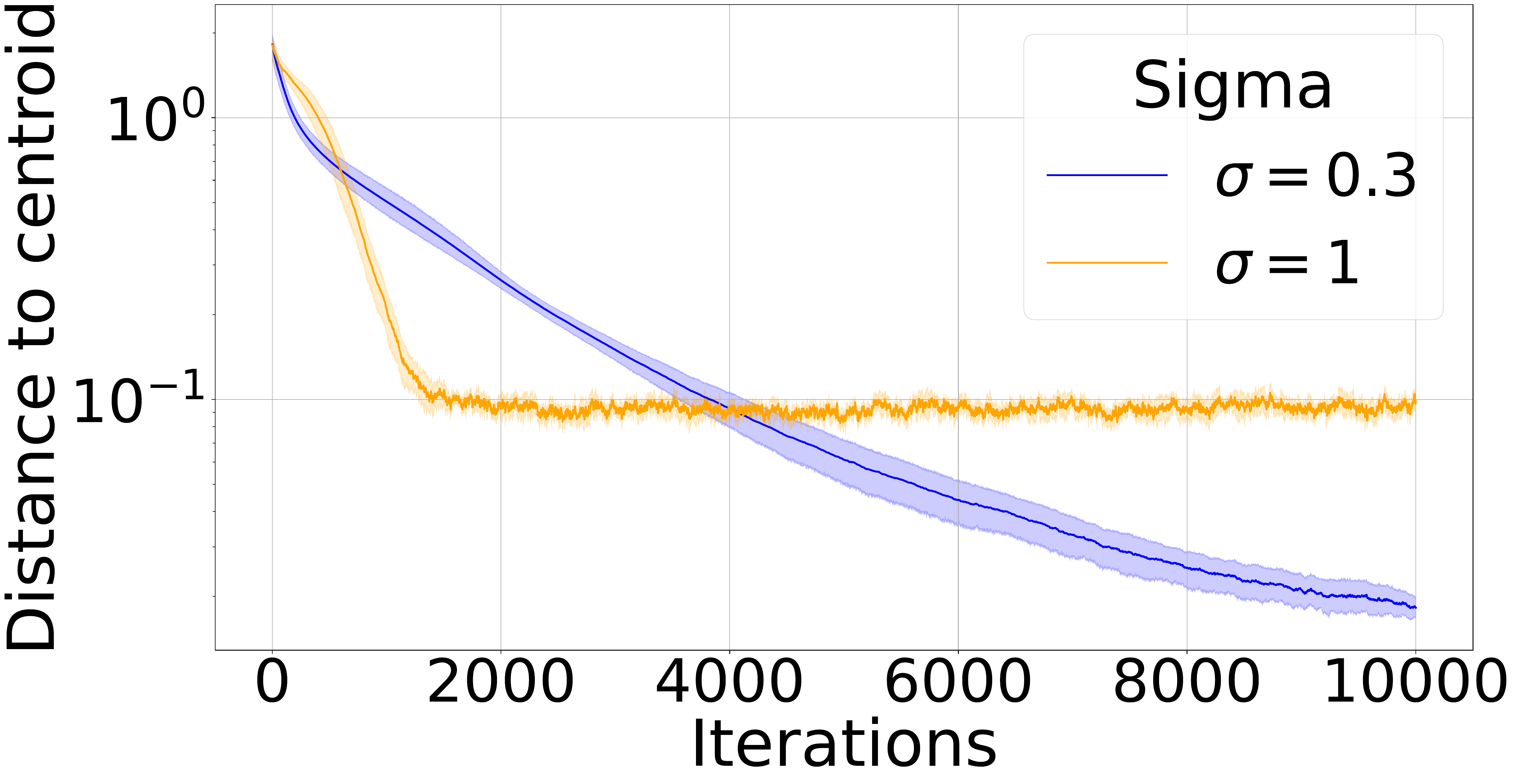}
        \caption{Distance to centroids vs PSGD iterations for the minimization of $\mathcal{R}^{\mathrm{soft},\rho_0}$, with an initialization on the unit sphere and regularization $\rho_0=0.5$, with data drawn from the non-degenerate case $\eqref{def:gaussian_mixture_model}$. 10 runs, 95\% percentile intervals are plotted.}
        \label{softmax_iter_noise}
    \end{subfigure}
    \caption{Performance of \eqref{PGDtheta}, with data drawn from the the non-degenerate case \eqref{def:gaussian_mixture_model}. 10 runs, 95\% percentile intervals are plotted.}
    \label{plots_softmax}
\end{figure}


In Figure \ref{softmax_iter_noise}, we set the regularization parameter $\rho_0$ to 0.5, and run \ref{PGDtheta} for $10^4$ iterations. We observe that the model yields accurate solutions under low interference ($\sigma = 0.3$); however, as the interference increases ($\sigma = 1$), the ability of the softmax attention heads to align with the underlying centroids is progressively impaired. A similar loss in alignment accuracy is observed as the dimensionality increases.


The experiments in Figure \ref{softmax_iter_noise} run in a few minutes on a standard laptop, whereas those in Figure \ref{softmax_reg_noise} may take up to two hours to cover the grid in the regularization hyperparameter.

\section{Proofs of Section \ref{sec:incontext}}

\begin{proof}[Proof of Proposition \ref{def:attention_layer_in_contextsetting}]
For $c\in\{0,1\}$, one has 

\begin{align*}
 &\mathbb{E}\left[\frac{2\lambda}{L}\sum_{\ell=1}^L\langle X_1,X_\ell\rangle X_\ell\Bigg| \mu_1^\star, \mu_0^\star, Z_1=c\right]\\
&=\frac{2\lambda}{L}\,\mathbb{E}[\|X_1\|^2 X_1 \mid \mu_1^\star, \mu_0^\star, Z_1=c]
+\frac{2\lambda (L-1)}{L}\,\mathbb{E}[\langle X_1,X_2\rangle X_2 \mid \mu_1^\star,\mu_0^\star,Z_1=c] \\
&= \frac{2\lambda}{L} \Big( \underbrace{\mathbb{E}[\|X_1\|^2X_1 \mid \mu_1^\star,\mu_0^\star,Z_1=c]}_{(1+(d+2)\sigma^2)\mu_c^\star} 
+ (L-1) \underbrace{\mathbb{E}[\langle X_1,X_2\rangle X_2 \mid \mu_1^\star,\mu_0^\star,Z_1=c]}_{(\frac{1}{2}+\sigma^2)\mu_c^\star} \Big) \\
&= \frac{2\lambda}{L}\Big[(1+(d+2)\sigma^2)+(L-1)\left(\frac{1}{2}+\sigma^2\right)\Big]\mu_c^\star.
\end{align*}

Where we have computed the conditional expectations by expanding each term using $X_i = \mu_{Z_i}^\star + \varepsilon_i$ with independent Gaussian errors $\varepsilon_i$, discarding odd-moment terms, and applying Isserlis' theorem \citep{isserlis}.

We remark that choosing $\lambda=\frac{L}{2}\frac{1}{1+(d+2)\sigma^2+(L-1)\left(\frac{1}{2}+\sigma^2\right)}$ we get that the encoding is unbiased.

And \begin{align*}
    \Bigg\Vert\mathbb{E}\left[\frac{2\lambda}{L}\sum_{\ell=1}^L\langle X_1,X_\ell\rangle X_\ell\Bigg| \mu_1^\star, \mu_0^\star, Z_1=c\right]\Bigg\Vert^2&=\frac{4\lambda^2}{L^2}\left[1+(d+2)\sigma^2+(L-1)\left(\frac{1}{2}+\sigma^2\right)\right]^2.
\end{align*}

Besides, \begin{align*}
    \mathbb{E}\left[\frac{2\lambda}{L}\sum_{\ell=1}^L\langle X_1,X_\ell\rangle X_\ell \Bigg| \mu_1^\star, \mu_0^\star \right]=\frac{\lambda}{L}\left[1+(d+2)\sigma^2+(L-1)\left(\frac{1}{2}+\sigma^2\right)\right](\mu_0^\star+\mu_1^\star)
\end{align*}
Also, \begin{align*}
    \Bigg\Vert\mathbb{E}\left[\frac{2\lambda}{L}\sum_{\ell=1}^L\langle X_1,X_\ell\rangle X_\ell\Bigg|\mu_1^\star, \mu_0^\star\right]\Bigg\Vert^2=\frac{2\lambda^2}{L^2}\left[1+(d+2)\sigma^2+(L-1)\left(\frac{1}{2}+\sigma^2\right)\right]^2
\end{align*}

On the other hand,

\begin{align*}
    \mathbb{E}\left[\Bigg\Vert\frac{2\lambda}{L}\sum_{\ell=1}^L\langle X_1,X_\ell\rangle X_\ell\Bigg\Vert^2\Bigg| \mu_1^\star,\mu_0^\star,Z_1=c\right]&=\frac{4\lambda^2}{L^2}\mathbb{E}[\Vert X_1\Vert^6|\mu_1^\star,\mu_0^\star, Z_1=c]\\
    &+\frac{12\lambda^2}{L^2}(L-1)\mathbb{E}[\Vert \langle X_1,X_2\rangle X_2\Vert^2|\mu_1^\star,\mu_0^\star,Z_1=c]\\
    &+\frac{8\lambda^2}{L^2}\frac{(L-1)(L-2)}{2}\mathbb{E}[\langle X_1,X_2\rangle\langle X_1,X_3\rangle\langle X_2,X_3\rangle|\mu_1^\star,\mu_0^\star,Z_1=c].
\end{align*}
Recalling the expressions stated at the beginning of the proof of Proposition \ref{riskrandomsetting} for moments of Gaussian r.v.~, we conclude that 
\begin{align*}
    &\mathbb{E}\left[\Bigg\Vert\frac{2\lambda}{L}\sum_{\ell=1}^L\langle X_1,X_\ell\rangle X_\ell\Bigg\Vert^2\Bigg| \mu_1^\star,\mu_0^\star,Z_1=c\right]\\
    &=\frac{4\lambda^2}{L^2}(1+3(d+4)\sigma^2+3(d+2)(d+4)\sigma^4+d(d+2)(d+4)\sigma^6)\\
    &+\frac{12\lambda^2}{L^2}(L-1)\left(\frac{1}{2}+\frac{(d+8)}{2}\sigma^2+3(d+2)\sigma^4+d(d+2)\sigma^6\right)\\
    &+\frac{8\lambda^2}{L^2}\frac{(L-1)(L-2)}{2}\left(2\left(\sigma^2+\frac{1}{2}\right)^3+(d-2)\sigma^6\right).
\end{align*}
And \begin{align*}
    &\mathrm{Var}\left[\frac{2\lambda}{L}\sum_{\ell=1}^L\langle X_1,X_\ell\rangle X_\ell\Bigg| \mu_1^\star,\mu_0^\star,Z_1=c\right]\\
    &=\mathbb{E}\left[\Bigg\Vert\frac{2\lambda}{L}\sum_{\ell=1}^L\langle X_1,X_\ell\rangle X_\ell\Bigg\Vert^2\Bigg| \mu_1^\star,\mu_0^\star,Z_1=c\right]-\Bigg\Vert\mathbb{E}\left[\frac{2\lambda}{L}\sum_{\ell=1}^L\langle X_1,X_\ell\rangle X_\ell\Bigg| \mu_1^\star, \mu_0^\star, Z_1=c\right]\Bigg\Vert^2\\
    &=\frac{4\lambda^2}{L^2}(1+3(d+4)\sigma^2+3(d+2)(d+4)\sigma^4+d(d+2)(d+4)\sigma^6)\\
    &+\frac{12\lambda^2}{L^2}(L-1)\left(\frac{1}{2}+\frac{(d+8)}{2}\sigma^2+3(d+2)\sigma^4+d(d+2)\sigma^6\right)\\
    &+\frac{8\lambda^2}{L^2}\frac{(L-1)(L-2)}{2}\left(2\left(\sigma^2+\frac{1}{2}\right)^3+(d-2)\sigma^6\right)\\
    &-\frac{4\lambda^2}{L^2}\left[1+(d+2)\sigma^2+(L-1)\left(\frac{1}{2}+\sigma^2\right)\right]^2.
\end{align*}
When $L\rightarrow\infty$, \begin{align*}
    &\mathrm{Var}\left[\frac{2\lambda}{L}\sum_{\ell=2}^L\langle X_1,X_\ell\rangle X_\ell\Bigg| \mu_1^\star,\mu_0^\star,Z_1=c\right]\\
    &\sim4\lambda^2\left(2\left(\sigma^2+\frac{1}{2}\right)^3+(d-2)\sigma^6\right)-4\lambda^2\left(\sigma^2+\frac{1}{2}\right)^2\\
    &=2\lambda^2\sigma^2(1+4\sigma^2+2d\sigma^4).
\end{align*}
Choosing the $\lambda=\frac{1}{1+2\sigma^2}$, we have an unbiased encoding with variance $$2\sigma^2\frac{1+4\sigma^2+2d\sigma^4}{(1+2\sigma^2)^2}$$
\end{proof}
\begin{proof}[Proof of Proposition \ref{riskrandomsetting}]
    We have
    \begin{align*}
    &\mathbb{E}\left[\Big\Vert X_1-\frac{2\lambda}{L}\sum_{\ell=1}^L\langle X_1,X_\ell\rangle X_\ell\Big\Vert^2 \Bigg| \mu_1^\star,\mu_0^\star \right]\\
    &=\mathbb{E}\left[\Vert X_1\Vert^2\right|\mu_1^\star,\mu_0^\star]-\frac{4\lambda}{L}\mathbb{E}[\Vert X_1\Vert^4|\mu_1^\star,\mu_0^\star]-\frac{4\lambda}{L}\sum_{\ell=2}^L\mathbb{E}\left[\langle X_1,X_\ell\rangle^2|\mu_1^\star,\mu_0^\star\right]\\
    &+\frac{4\lambda^2}{L^2}\mathbb{E}[\Vert X_1\Vert^6|\mu_1^\star,\mu_0^\star]+\frac{4\lambda^2}{L^2}\sum_{\ell=2}^L \mathbb{E}[\Vert \langle X_1,X_\ell\rangle X_\ell\Vert^2|\mu_1^\star,\mu_0^\star]\\
    &+\frac{8\lambda^2}{L^2}\sum_{\ell=2}^L \mathbb{E}[\Vert X_1\langle X_1,X_{\ell}\rangle\Vert^2||\mu_1^\star,\mu_0^\star]+\frac{8\lambda^2}{L^2}\sum_{2\leq \ell<k\leq L}\mathbb{E}[ \langle X_1,X_\ell\rangle\langle X_1,X_k\rangle\langle X_\ell,X_k\rangle|\mu_1^\star,\mu_0^\star ].
\end{align*}
Furthermore we have the following, 
\begin{align*}
    \mathbb{E}[\Vert X_1\Vert^2|\mu_1^\star,\mu_0^\star]&=1+\sigma^2d, \\
    \mathbb{E}\left[\langle X_1,X_2\rangle^2|\mu_1^\star,\mu_0^\star\right]&=\frac{1}{2}+2\sigma^2+d\sigma^4,\\ 
    \mathbb{E}[\Vert \langle X_1,X_2\rangle X_2\Vert^2|\mu_1^\star,\mu_0^\star]&=\frac{1}{2}+\frac{d+8}{2}\sigma^2+3(d+2)\sigma^4+d(d+2)\sigma^6,\\ 
    \mathbb{E}[\Vert X_1\Vert^4|\mu_1^\star,\mu_0^\star ]&=1+2(d+2)\sigma^2+d(d+2)\sigma^4,\\
    \mathbb{E}[\Vert X_1\Vert^6|\mu_1^\star,\mu_0^\star ]&=1+3(d+4)\sigma^2+3(d+2)(d+4)\sigma^4+d(d+2)(d+4)\sigma^6,\\
    \mathbb{E}[ \langle X_1,X_2\rangle\langle X_1,X_3\rangle\langle X_2,X_3\rangle|\mu_1^\star,\mu_0^\star ]&=2\left(\sigma^2+\frac{1}{2}\right)^3+(d-2)\sigma^6.
\end{align*}

These identities follow by writing
\[
X_i|\mu_i^\star = \mu_i^\star + \varepsilon_i, 
\qquad \varepsilon_i \sim \mathcal{N}(0,\sigma^2 I_d)\quad \text{i.i.d.},
\]
and expanding the expressions in terms of $\mu_i^\star$ and $\varepsilon_i$.  
The expectations then reduce to Gaussian moments, which can be evaluated systematically using Isserlis' theorem \citep{isserlis}.  
Since this involves only straightforward but lengthy computations, we omit the details here.

Since no expression depends on $\mu_1^\star,\mu_0^\star$, we have $$\mathbb{E}\left[\Big\Vert X_1-\frac{2\lambda}{L}\sum_{\ell=1}^L\langle X_1,X_\ell\rangle X_\ell\Big\Vert^2  \right]=\mathbb{E}\left[\Big\Vert X_1-\frac{2\lambda}{L}\sum_{\ell=1}^L\langle X_1,X_\ell\rangle X_\ell\Big\Vert^2\Bigg|\mu_1^\star,\mu_0^\star\right].$$
And
\begin{align*}
    &\mathbb{E}\left[\Big\Vert X_1-\frac{2\lambda}{L}\sum_{\ell=1}^L\langle X_1,X_\ell\rangle X_\ell\Big\Vert^2\right]\\
    &=1+\sigma^2d-\frac{4\lambda}{L}(1+2(d+2)\sigma^2+d(d+2)\sigma^4)-\frac{4\lambda}{L}(L-1)\left(\frac{1}{2}+2\sigma^2+d\sigma^4\right)\\
    &+\frac{4\lambda^2}{L^2}(1+3(d+4)\sigma^2+3(d+2)(d+4)\sigma^4+d(d+2)(d+4)\sigma^6)\\
    &+\frac{12\lambda^2}{L^2}(L-1)\left(\frac{1}{2}+\frac{d+8}{2}\sigma^2+(3d+6)\sigma^4+d(d+2)\sigma^6\right)\\
    &+\frac{8\lambda^2}{L^2}\frac{(L-1)(L-2)}{2}\left(2\left(\sigma^2+\frac{1}{2}\right)^3+(d-2)\sigma^6\right).
\end{align*}
When $L\rightarrow\infty$, we obtain 
\begin{align*}
    \mathbb{E}\left[\Big\Vert X_1-\frac{2\lambda}{L}\sum_{\ell=1}^L\langle X_1,X_\ell\rangle X_\ell\Big\Vert^2\right] =&   (1+\sigma^2d)-2\lambda(1+4\sigma^2+2d\sigma^4)\\
    &+4\lambda^2\left(2\left(\sigma^2+\frac{1}{2}\right)^3+(d-2)\sigma^6\right).
\end{align*}
And we can choose $\lambda=\frac{1+4\sigma^2+2d\sigma^4}{4\left(2\left(\sigma^2+\frac{1}{2}\right)^3+(d-2)\sigma^6\right)}$ to get \begin{align*}
    \mathbb{E}\left[\Big\Vert X_1-\frac{2\lambda}{L}\sum_{\ell=1}^L\langle X_1,X_\ell\rangle X_\ell\Big\Vert^2\right] &=   (1+\sigma^2d)-\frac{(1+4\sigma^2+2d\sigma^4)^2}{4\left(2\left(\sigma^2+\frac{1}{2}\right)^3+(d-2)\sigma^6\right)}\\
    &=\sigma^2(d-2)\frac{1+2\sigma^2}{1+6\sigma^2+12\sigma^4+4d\sigma^6}\\
    &\leq \sigma^2(d-2).
\end{align*}
\end{proof}

\newpage
\section*{NeurIPS Paper Checklist}

\begin{enumerate}

\item {\bf Claims}
    \item[] Question: Do the main claims made in the abstract and introduction accurately reflect the paper's contributions and scope?
    \item[] Answer: \answerYes{} 
    \item[] Justification: Our main claim is stated in the abstract, and fully studied in Section \ref{sec:gmm}, particularly Theorem \ref{thm:main_gmm}.
    \item[] Guidelines:
    \begin{itemize}
        \item The answer NA means that the abstract and introduction do not include the claims made in the paper.
        \item The abstract and/or introduction should clearly state the claims made, including the contributions made in the paper and important assumptions and limitations. A No or NA answer to this question will not be perceived well by the reviewers. 
        \item The claims made should match theoretical and experimental results, and reflect how much the results can be expected to generalize to other settings. 
        \item It is fine to include aspirational goals as motivation as long as it is clear that these goals are not attained by the paper. 
    \end{itemize}

\item {\bf Limitations}
    \item[] Question: Does the paper discuss the limitations of the work performed by the authors?
    \item[] Answer: \answerYes{} 
    \item[] Justification: The authors are transparent about the simplified nature of the architectures and mathematical models they study, and consistently highlight these choices as part of the scope and limitations of their work.
    \item[] Guidelines:
    \begin{itemize}
        \item The answer NA means that the paper has no limitation while the answer No means that the paper has limitations, but those are not discussed in the paper. 
        \item The authors are encouraged to create a separate "Limitations" section in their paper.
        \item The paper should point out any strong assumptions and how robust the results are to violations of these assumptions (e.g., independence assumptions, noiseless settings, model well-specification, asymptotic approximations only holding locally). The authors should reflect on how these assumptions might be violated in practice and what the implications would be.
        \item The authors should reflect on the scope of the claims made, e.g., if the approach was only tested on a few datasets or with a few runs. In general, empirical results often depend on implicit assumptions, which should be articulated.
        \item The authors should reflect on the factors that influence the performance of the approach. For example, a facial recognition algorithm may perform poorly when image resolution is low or images are taken in low lighting. Or a speech-to-text system might not be used reliably to provide closed captions for online lectures because it fails to handle technical jargon.
        \item The authors should discuss the computational efficiency of the proposed algorithms and how they scale with dataset size.
        \item If applicable, the authors should discuss possible limitations of their approach to address problems of privacy and fairness.
        \item While the authors might fear that complete honesty about limitations might be used by reviewers as grounds for rejection, a worse outcome might be that reviewers discover limitations that aren't acknowledged in the paper. The authors should use their best judgment and recognize that individual actions in favor of transparency play an important role in developing norms that preserve the integrity of the community. Reviewers will be specifically instructed to not penalize honesty concerning limitations.
    \end{itemize}

\item {\bf Theory assumptions and proofs}
    \item[] Question: For each theoretical result, does the paper provide the full set of assumptions and a complete (and correct) proof?
    \item[] Answer: \answerYes{} 
    \item[] Justification: All theoretical assumptions are explicitly stated. Complete proofs are provided in Appendix \ref{sec:app_degenerate} and \ref{sec:app_gmm}, with supporting technical results in Appendix \ref{sec:computation}.
    \item[] Guidelines:
    \begin{itemize}
        \item The answer NA means that the paper does not include theoretical results. 
        \item All the theorems, formulas, and proofs in the paper should be numbered and cross-referenced.
        \item All assumptions should be clearly stated or referenced in the statement of any theorems.
        \item The proofs can either appear in the main paper or the supplemental material, but if they appear in the supplemental material, the authors are encouraged to provide a short proof sketch to provide intuition. 
        \item Inversely, any informal proof provided in the core of the paper should be complemented by formal proofs provided in appendix or supplemental material.
        \item Theorems and Lemmas that the proof relies upon should be properly referenced. 
    \end{itemize}

    \item {\bf Experimental result reproducibility}
    \item[] Question: Does the paper fully disclose all the information needed to reproduce the main experimental results of the paper to the extent that it affects the main claims and/or conclusions of the paper (regardless of whether the code and data are provided or not)?
    \item[] Answer: \answerYes{} 
    \item[] Justification: The main algorithm used for our numerical experiments is defined in Section \ref{sec:psgd}. Detailed descriptions of the experimental setups are provided in Sections \ref{xp_dirac}, \ref{xp_gmm}, Appendix \ref{app:xp_higher_dim}, and \ref{app:xp_softmax}. This information enables others to replicate our experiments without the need for code.
    \item[] Guidelines:
    \begin{itemize}
        \item The answer NA means that the paper does not include experiments.
        \item If the paper includes experiments, a No answer to this question will not be perceived well by the reviewers: Making the paper reproducible is important, regardless of whether the code and data are provided or not.
        \item If the contribution is a dataset and/or model, the authors should describe the steps taken to make their results reproducible or verifiable. 
        \item Depending on the contribution, reproducibility can be accomplished in various ways. For example, if the contribution is a novel architecture, describing the architecture fully might suffice, or if the contribution is a specific model and empirical evaluation, it may be necessary to either make it possible for others to replicate the model with the same dataset, or provide access to the model. In general. releasing code and data is often one good way to accomplish this, but reproducibility can also be provided via detailed instructions for how to replicate the results, access to a hosted model (e.g., in the case of a large language model), releasing of a model checkpoint, or other means that are appropriate to the research performed.
        \item While NeurIPS does not require releasing code, the conference does require all submissions to provide some reasonable avenue for reproducibility, which may depend on the nature of the contribution. For example
        \begin{enumerate}
            \item If the contribution is primarily a new algorithm, the paper should make it clear how to reproduce that algorithm.
            \item If the contribution is primarily a new model architecture, the paper should describe the architecture clearly and fully.
            \item If the contribution is a new model (e.g., a large language model), then there should either be a way to access this model for reproducing the results or a way to reproduce the model (e.g., with an open-source dataset or instructions for how to construct the dataset).
            \item We recognize that reproducibility may be tricky in some cases, in which case authors are welcome to describe the particular way they provide for reproducibility. In the case of closed-source models, it may be that access to the model is limited in some way (e.g., to registered users), but it should be possible for other researchers to have some path to reproducing or verifying the results.
        \end{enumerate}
    \end{itemize}

\item {\bf Open access to data and code}
    \item[] Question: Does the paper provide open access to the data and code, with sufficient instructions to faithfully reproduce the main experimental results, as described in supplemental material?
    \item[] Answer: \answerYes{} 
    \item[] Justification: Our primary contribution is theoretical, focusing on the analysis of a projected gradient descent algorithm. The experimental results serve to illustrate these theoretical findings. The data used are synthetic mixtures of Gaussian, and the algorithm is a standard projected gradient descent on the unit sphere, both of which are straightforward to implement. The code can be found in the supplementary materials.
    \item[] Guidelines: 
    \begin{itemize}
        \item The answer NA means that paper does not include experiments requiring code.
        \item Please see the NeurIPS code and data submission guidelines (\url{https://nips.cc/public/guides/CodeSubmissionPolicy}) for more details.
        \item While we encourage the release of code and data, we understand that this might not be possible, so “No” is an acceptable answer. Papers cannot be rejected simply for not including code, unless this is central to the contribution (e.g., for a new open-source benchmark).
        \item The instructions should contain the exact command and environment needed to run to reproduce the results. See the NeurIPS code and data submission guidelines (\url{https://nips.cc/public/guides/CodeSubmissionPolicy}) for more details.
        \item The authors should provide instructions on data access and preparation, including how to access the raw data, preprocessed data, intermediate data, and generated data, etc.
        \item The authors should provide scripts to reproduce all experimental results for the new proposed method and baselines. If only a subset of experiments are reproducible, they should state which ones are omitted from the script and why.
        \item At submission time, to preserve anonymity, the authors should release anonymized versions (if applicable).
        \item Providing as much information as possible in supplemental material (appended to the paper) is recommended, but including URLs to data and code is permitted.
    \end{itemize}

\item {\bf Experimental setting/details}
    \item[] Question: Does the paper specify all the training and test details (e.g., data splits, hyperparameters, how they were chosen, type of optimizer, etc.) necessary to understand the results?
    \item[] Answer: \answerYes{} 
    \item[] Justification: Detailed accounts of the experimental configurations are presented in Sections \ref{xp_dirac}, \ref{xp_gmm}, \ref{app:xp_higher_dim}, and \ref{app:xp_softmax}.
    \item[] Guidelines:
    \begin{itemize}
        \item The answer NA means that the paper does not include experiments.
        \item The experimental setting should be presented in the core of the paper to a level of detail that is necessary to appreciate the results and make sense of them.
        \item The full details can be provided either with the code, in appendix, or as supplemental material.
    \end{itemize}

\item {\bf Experiment statistical significance}
    \item[] Question: Does the paper report error bars suitably and correctly defined or other appropriate information about the statistical significance of the experiments?
    \item[] Answer: \answerYes{} 
    \item[] Justification: To assess the statistical significance of our experimental results, we conducted 10 independent runs for each experiment. The outcomes are visualized using Seaborn's lineplot function, which displays the mean performance along with error bands representing variability across runs.
    \item[] Guidelines:
    \begin{itemize}
        \item The answer NA means that the paper does not include experiments.
        \item The authors should answer "Yes" if the results are accompanied by error bars, confidence intervals, or statistical significance tests, at least for the experiments that support the main claims of the paper.
        \item The factors of variability that the error bars are capturing should be clearly stated (for example, train/test split, initialization, random drawing of some parameter, or overall run with given experimental conditions).
        \item The method for calculating the error bars should be explained (closed form formula, call to a library function, bootstrap, etc.)
        \item The assumptions made should be given (e.g., Normally distributed errors).
        \item It should be clear whether the error bar is the standard deviation or the standard error of the mean.
        \item It is OK to report 1-sigma error bars, but one should state it. The authors should preferably report a 2-sigma error bar than state that they have a 96\% CI, if the hypothesis of Normality of errors is not verified.
        \item For asymmetric distributions, the authors should be careful not to show in tables or figures symmetric error bars that would yield results that are out of range (e.g. negative error rates).
        \item If error bars are reported in tables or plots, The authors should explain in the text how they were calculated and reference the corresponding figures or tables in the text.
    \end{itemize}

\item {\bf Experiments compute resources}
    \item[] Question: For each experiment, does the paper provide sufficient information on the computer resources (type of compute workers, memory, time of execution) needed to reproduce the experiments?
    \item[] Answer: \answerYes{} 
    \item[] Justification: All experiments were conducted on a standard laptop. Specific details regarding the computational resources, such as execution times, are provided at the end of each numerical experimental section.
    \item[] Guidelines:
    \begin{itemize}
        \item The answer NA means that the paper does not include experiments.
        \item The paper should indicate the type of compute workers CPU or GPU, internal cluster, or cloud provider, including relevant memory and storage.
        \item The paper should provide the amount of compute required for each of the individual experimental runs as well as estimate the total compute. 
        \item The paper should disclose whether the full research project required more compute than the experiments reported in the paper (e.g., preliminary or failed experiments that didn't make it into the paper). 
    \end{itemize}
    
\item {\bf Code of ethics}
    \item[] Question: Does the research conducted in the paper conform, in every respect, with the NeurIPS Code of Ethics \url{https://neurips.cc/public/EthicsGuidelines}?
    \item[] Answer: \answerYes{} 
    \item[] Justification: We have read and understood the code of ethics; and have committed to conform.
    \item[] Guidelines:
    \begin{itemize}
        \item The answer NA means that the authors have not reviewed the NeurIPS Code of Ethics.
        \item If the authors answer No, they should explain the special circumstances that require a deviation from the Code of Ethics.
        \item The authors should make sure to preserve anonymity (e.g., if there is a special consideration due to laws or regulations in their jurisdiction).
    \end{itemize}

\item {\bf Broader impacts}
    \item[] Question: Does the paper discuss both potential positive societal impacts and negative societal impacts of the work performed?
    \item[] Answer: \answerNA{} 
    \item[] Justification: This work is primarily theoretical and employs simplified, synthetic distributions. It does not have direct societal implications beyond contributing to a deeper understanding of certain aspects of attention-based mechanisms.
    \item[] Guidelines:
    \begin{itemize}
        \item The answer NA means that there is no societal impact of the work performed.
        \item If the authors answer NA or No, they should explain why their work has no societal impact or why the paper does not address societal impact.
        \item Examples of negative societal impacts include potential malicious or unintended uses (e.g., disinformation, generating fake profiles, surveillance), fairness considerations (e.g., deployment of technologies that could make decisions that unfairly impact specific groups), privacy considerations, and security considerations.
        \item The conference expects that many papers will be foundational research and not tied to particular applications, let alone deployments. However, if there is a direct path to any negative applications, the authors should point it out. For example, it is legitimate to point out that an improvement in the quality of generative models could be used to generate deepfakes for disinformation. On the other hand, it is not needed to point out that a generic algorithm for optimizing neural networks could enable people to train models that generate Deepfakes faster.
        \item The authors should consider possible harms that could arise when the technology is being used as intended and functioning correctly, harms that could arise when the technology is being used as intended but gives incorrect results, and harms following from (intentional or unintentional) misuse of the technology.
        \item If there are negative societal impacts, the authors could also discuss possible mitigation strategies (e.g., gated release of models, providing defenses in addition to attacks, mechanisms for monitoring misuse, mechanisms to monitor how a system learns from feedback over time, improving the efficiency and accessibility of ML).
    \end{itemize}
    
\item {\bf Safeguards}
    \item[] Question: Does the paper describe safeguards that have been put in place for responsible release of data or models that have a high risk for misuse (e.g., pretrained language models, image generators, or scraped datasets)?
    \item[] Answer: \answerNA{} 
    \item[] Justification: This paper is mainly theoretical and does not involve the release of any models or datasets that could pose a risk of misuse.
    \item[] Guidelines:
    \begin{itemize}
        \item The answer NA means that the paper poses no such risks.
        \item Released models that have a high risk for misuse or dual-use should be released with necessary safeguards to allow for controlled use of the model, for example by requiring that users adhere to usage guidelines or restrictions to access the model or implementing safety filters. 
        \item Datasets that have been scraped from the Internet could pose safety risks. The authors should describe how they avoided releasing unsafe images.
        \item We recognize that providing effective safeguards is challenging, and many papers do not require this, but we encourage authors to take this into account and make a best faith effort.
    \end{itemize}

\item {\bf Licenses for existing assets}
    \item[] Question: Are the creators or original owners of assets (e.g., code, data, models), used in the paper, properly credited and are the license and terms of use explicitly mentioned and properly respected?
    \item[] Answer: \answerYes{} 
    \item[] Justification: Our code is self-contained and does not include any source code or binary files from external libraries. Therefore, there are no concerns regarding permissions or licensing. We did cite open-sourced libraries used in our paper, such as JAX.
    \item[] Guidelines:
    \begin{itemize}
        \item The answer NA means that the paper does not use existing assets.
        \item The authors should cite the original paper that produced the code package or dataset.
        \item The authors should state which version of the asset is used and, if possible, include a URL.
        \item The name of the license (e.g., CC-BY 4.0) should be included for each asset.
        \item For scraped data from a particular source (e.g., website), the copyright and terms of service of that source should be provided.
        \item If assets are released, the license, copyright information, and terms of use in the package should be provided. For popular datasets, \url{paperswithcode.com/datasets} has curated licenses for some datasets. Their licensing guide can help determine the license of a dataset.
        \item For existing datasets that are re-packaged, both the original license and the license of the derived asset (if it has changed) should be provided.
        \item If this information is not available online, the authors are encouraged to reach out to the asset's creators.
    \end{itemize}

\item {\bf New assets}
    \item[] Question: Are new assets introduced in the paper well documented and is the documentation provided alongside the assets?
    \item[] Answer: \answerNA{} 
    \item[] Justification: This paper does not introduce or release any new assets. The numerical experiments conducted serve solely to support the theoretical results and are not central to the paper's contribution.
    \item[] Guidelines:
    \begin{itemize}
        \item The answer NA means that the paper does not release new assets.
        \item Researchers should communicate the details of the dataset/code/model as part of their submissions via structured templates. This includes details about training, license, limitations, etc. 
        \item The paper should discuss whether and how consent was obtained from people whose asset is used.
        \item At submission time, remember to anonymize your assets (if applicable). You can either create an anonymized URL or include an anonymized zip file.
    \end{itemize}

\item {\bf Crowdsourcing and research with human subjects}
    \item[] Question: For crowdsourcing experiments and research with human subjects, does the paper include the full text of instructions given to participants and screenshots, if applicable, as well as details about compensation (if any)? 
    \item[] Answer: \answerNA{} 
    \item[] Justification: This work does not involve crowdsourcing nor research with human subjects.
    \item[] Guidelines:
    \begin{itemize}
        \item The answer NA means that the paper does not involve crowdsourcing nor research with human subjects.
        \item Including this information in the supplemental material is fine, but if the main contribution of the paper involves human subjects, then as much detail as possible should be included in the main paper. 
        \item According to the NeurIPS Code of Ethics, workers involved in data collection, curation, or other labor should be paid at least the minimum wage in the country of the data collector. 
    \end{itemize}

\item {\bf Institutional review board (IRB) approvals or equivalent for research with human subjects}
    \item[] Question: Does the paper describe potential risks incurred by study participants, whether such risks were disclosed to the subjects, and whether Institutional Review Board (IRB) approvals (or an equivalent approval/review based on the requirements of your country or institution) were obtained?
    \item[] Answer: \answerNA{} 
    \item[] Justification: This work does not involve crowdsourcing nor research with human subjects.
    \item[] Guidelines:
    \begin{itemize}
        \item The answer NA means that the paper does not involve crowdsourcing nor research with human subjects.
        \item Depending on the country in which research is conducted, IRB approval (or equivalent) may be required for any human subjects research. If you obtained IRB approval, you should clearly state this in the paper. 
        \item We recognize that the procedures for this may vary significantly between institutions and locations, and we expect authors to adhere to the NeurIPS Code of Ethics and the guidelines for their institution. 
        \item For initial submissions, do not include any information that would break anonymity (if applicable), such as the institution conducting the review.
    \end{itemize}

\item {\bf Declaration of LLM usage}
    \item[] Question: Does the paper describe the usage of LLMs if it is an important, original, or non-standard component of the core methods in this research? Note that if the LLM is used only for writing, editing, or formatting purposes and does not impact the core methodology, scientific rigorousness, or originality of the research, declaration is not required.
    \item[] Answer: \answerNA{} 
    \item[] Justification: We did not employ any large language models (LLMs) during any stage of the research process.
    \item[] Guidelines:
    \begin{itemize}
        \item The answer NA means that the core method development in this research does not involve LLMs as any important, original, or non-standard components.
        \item Please refer to our LLM policy (\url{https://neurips.cc/Conferences/2025/LLM}) for what should or should not be described.
    \end{itemize}

\end{enumerate}

\end{document}